\documentclass{article}

\usepackage[main,final]{neurips_2026}

\usepackage{algorithm}
\usepackage{algorithmic}
\usepackage{wrapfig}
\usepackage{booktabs}
\usepackage[colorlinks=true, allcolors=blue]{hyperref}

\usepackage{subcaption}
\usepackage{enumitem}
\usepackage{graphicx}
\usepackage{caption}

\newcommand{\INPUT}{\item[\textbf{Input:}]}

\newcommand{\ppre}{p^{\mathrm{pre}}}

\usepackage{amsmath}
\usepackage{amssymb}
\usepackage{mathtools}
\usepackage{amsthm}
\usepackage{empheq}

\newcommand{\cX}{\mathcal{X}}

\newcommand{\E}{\mathbb{E}}
\newcommand{\R}{\mathbb{R}}

\newcommand{\Var}{\mathrm{Var}}
\DeclareMathOperator*{\argmax}{arg\,max}

\usepackage[utf8]{inputenc} 
\usepackage[T1]{fontenc}    
\usepackage{url}            
\usepackage{booktabs}       
\usepackage{amsfonts}       
\usepackage{nicefrac}       
\usepackage{microtype}      
\usepackage{xcolor}         

\theoremstyle{plain}
\newtheorem{theorem}{Theorem}[section]
\newtheorem{proposition}[theorem]{Proposition}
\newtheorem{lemma}[theorem]{Lemma}
\newtheorem{corollary}[theorem]{Corollary}
\theoremstyle{definition}
\newtheorem{definition}[theorem]{Definition}
\newtheorem{assumption}[theorem]{Assumption}
\theoremstyle{remark}
\newtheorem{remark}[theorem]{Remark}

\title{Learnable Chernoff Baselines for Provable Inference-Time Alignment}

\author{
  Sunil Madhow \\
  UC San Diego \\
  \texttt{smadhow@ucsd.edu}
  \And
  Yuchen Liang \\
  Indiana University\\
  \texttt{yucliang@iu.edu}
  \And
  Ness Shroff \\
  Ohio State University\\
  \texttt{shroff.11@osu.edu}
  \AND
  Yingbin Liang \\
  Ohio State University\\
  \texttt{liang.889@osu.edu}
  \And
  Yu-Xiang Wang \\
  UC San Diego\\
  \texttt{yuxiangw@ucsd.edu}
}

\begin{document}
\maketitle

\begin{abstract}
We study inference-time reward-guided alignment for generative models. Existing methods often rely on either architecture-specific adaptations or computationally costly inference procedures. We introduce \textbf{Learnable Chernoff Baselines (LCBs)} as a method for efficiently and approximately sampling from the exponentially tilted kernels that arise from KL-regularized reward alignment. Using only black-box sampling access to the pretrained model, LCBs implement a form of rejection sampling with adaptively selected acceptance probabilities, which allows fine-grained control over inference-compute scaling. We establish total-variation guarantees to the ideal aligned model, which reveal the dimension-independent quantities governing the tradeoff between accurate sampling and inference compute. We empirically demonstrate in both continuous and discrete diffusion settings that LCB sampling closely matches ideal rejection sampling, but uses substantially fewer queries to the pretrained model. We also include experiments on a diffusion large language model. 
\end{abstract}

\section{Introduction}
\label{intro}


As Diffusion Models and other generative systems grow in their ability to produce meaningful facsimiles of human data, we face two central questions:  (1) \textbf{How can we harness pretrained generative models} to tackle downstream tasks that are poorly represented in their training corpora? (2) \textbf{How can we guarantee safety} by suppressing harmful or undesirable content while amplifying desirable content? The AI alignment literature aims to address these questions by teaching models to optimize {reward functions} that encode ``good behavior.'' These can come from human feedback \citep{ouyang2022instructgpt}, formal verifiers of model output \citep{yang2023leandojo}, or domain expertise \citep{li2024derivativefreeguidancecontinuousdiscrete}.

We consider a common formulation \citep{beirami2025theoretical, li2024derivativefreeguidancecontinuousdiscrete, huang2025correcting} of the alignment problem in terms of a KL-regularized reward maximization problem:
\begin{equation}
\label{eq:tradeoff}
p^* = \argmax_{p \in \Delta(\mathcal{X}) }\; \mathbb{E}_{x_0\sim p}[r(x_0)] - \alpha\, \mathrm{KL}\!\left(p \,\|\, p^{\mathrm{pre}}\right),
\end{equation}
where $\cX$ is the space of generated content (e.g. images or text), $\Delta(\cX)$  denotes the set of probability distributions over $\cX$, $p^{\mathrm{pre}} \in \Delta(\mathcal{X})$ is the sampling distribution of the final output $x_0$ of a pretrained model, $r : \mathcal{X} \to \mathbb{R}$ is a (possibly non-differentiable) reward function, and $\alpha > 0$ controls the strength of Kullback-Leibler divergence regularization $\mathrm{KL}\!\left(p \,\|\, p^{\mathrm{pre}}\right) = \sum_{x\in \cX} p(x)\log(p(x)/p^{\mathrm{pre}}(x))$. This regularization is introduced to ensure that the aligned model $p^*$ continues to produce semantically meaningful samples by remaining in a neighborhood of the pre-trained model $p^{\mathrm{pre}}$.

One approach to accessing $p^*(x_0)$ is RL finetuning, which updates model weights for higher reward \citep{fan2023dpok, shao2024deepseekmathpushinglimitsmathematical}.  RL finetuning, however, is computationally costly and requires open-weight access to the pre-trained model. Another approach, known as \textbf{Inference-time alignment}, aims at directly sampling from $p^*$ with only API-access to $\ppre$'s sampling process. Inference-time alignment enables more flexible reward-guided generation with closed-weights models.

\begin{table}[t]
\centering
\caption{Summary of distributions of interest}
\label{tab:distributions_main}
\normalsize
\setlength{\tabcolsep}{8pt}
\renewcommand{\arraystretch}{1.25}
\begin{tabular}{@{}lp{0.78\columnwidth}@{}}
\toprule
\textbf{Process} & \textbf{Description} \\
\midrule
$\ppre(x_T \dots x_0)$ &
Pretrained (unguided) diffusion process. \\
\addlinespace[4pt]

$p^*(x_T \dots x_0)$ &
Aligned/tilted process satisfying $p^*(x_t | x_{t + 1})\propto e^{v_t(x_t)}\ppre_t(x_t | x_{t + 1})$ \\
\addlinespace[4pt]
$\hat p(x_T \dots x_0)$ &
Learned-value tilt of $\ppre$ satisfying $\hat p_t(x_t | x_{t + 1})\propto e^{\hat v_t(x_t)}\ppre_t(x_t | x_{t + 1})$ \\
\addlinespace[4pt]

$\hat q(x_T \dots x_0)$ &
LCB-induced approximate sampling law; targets $\hat q_t(x_t | x_{t + 1}) \approx \hat{p}_t(x_t | x_{t + 1})$. \\
\bottomrule
\end{tabular}
\end{table}

\textbf{The main contribution} of this paper is a novel inference-time alignment algorithm known as \emph{Learnable Chernoff Baselines} (LCBs). LCBs are among the first efficient model agnostic inference-time alignment methods with nonasymptotic sampling guarantees. LCBs also show promise in correctly and efficiently aligning models at both toy and medium scales. To describe the algorithm, we first need to understand the structure of $p^*$, the solution to \eqref{eq:tradeoff}.


\noindent\textbf{Structure of the solution.} 
It is well known that the optimal solution to \eqref{eq:tradeoff} is an exponential tilt, $p^*(x_0) \propto e^{r(x_0)/\alpha} \, p^{\mathrm{pre}}(x_0)$ \citep{levine2018rlasinference, yuan2023rewarddirected}. Less known is that, in the ubiquitous case where $\ppre(x_0)$ is induced by a Markov process over intermediate states $x_T \dots x_0$ (with backward indexing), $p^*(x_0)$ is realized by exponentially tilting each transition by the \emph{soft-value function}, $v_t (x_t):= \log \E_{\ppre(x_0|x_t)}[e^{r(x_0)/\alpha}]$\footnote{Here we abuse the notation and use $\ppre$ and $p^*$ to also denote the transition kernels $\ppre(x_{t}|x_{t+1})$ (and $p^*(x_t|x_{t+1})$) of the corresponding Markov process. The claim says that $p^*(x_0)$ can be sampled by starting from $p^*(x_T)$ and then following the Markov process defined by a sequence of tilted transition kernels $p^*(x_t|x_{t+1}) \propto e^{v_t(x_t)/\alpha}\ppre(x_{t}|x_{t+1})$.} \citep{uehara2025tutorial}. 


\noindent\textbf{The LCBs algorithm.}  Our proposed method is a rejection-based sampling paradigm for targeting $p^*(x_0)$ at the level of individual samples. In addition to estimating the soft-value functions $v_t$ with $\hat{v}_t$ as in \citep{li2024derivativefreeguidancecontinuousdiscrete}, we  introduce Learnable Chernoff Baselines (LCBs) as a technique for implementing adaptive rejection sampling envelopes for transition kernels. ``Baseline functions'' accelerate sampling by enabling adaptive choices of transition acceptance probabilities based on projected reward. A tradeoff between sample accuracy and compute efficiency is governed by a user parameter $\delta >0$. 

\noindent\textbf{A summary of results.} For the sampling distribution $\hat{q}(x_0)$ induced by LCBs with user parameter $\delta > 0$, our most important results are:

\begin{enumerate}[leftmargin = *]
    \item 
    Given estimates $\{\hat{v}_t\}$ of the soft-value functions $\{v_t\}$, we prove theoretically (Theorem \ref{thm:lcb tv bound}) that  $\hat{q}(x_0)$ satisfies the following TV bound with respect to an ideal rejection sampling scheme that outputs samples $\hat p (x_0)$ from the process with tilted kernels $\hat p(x_t | x_{t + 1}) \propto e^{\hat{v}_t(x_t)}\ppre_t(x_t | x_{t + 1})$:
    \[\textstyle d_{TV}(\hat{q}(x_0), \hat{p}(x_0)) \lesssim \delta \sum_t e^{J^*_t + 2\epsilon_0}, \]
    where $J^*_t$ is the optimal value of the ``LCB objective'' and $\epsilon_0$ is a statistical term from the number of samples. 
    
    \item We show as corollary (Corollary \ref{cor:subgaussian value corollary}) that, if $\hat{v}_t(x_t)$ is $\sigma_t$-sub-Gaussian with respect to $\ppre_t(x_t | x_{t + 1})$ for all $x_{t + 1}$, then 
    \[\textstyle d_{TV}(\hat{q}(x_0), \hat{p}(x_0)) \lesssim \delta \sum_t e^{2\sigma_t\sqrt{2\log\frac{1}{\delta}} + 2\epsilon_0}. \]
     We further show (Theorem \ref{thm:mog thm}) that the sub-Gaussian assumption holds for the true value function $v_t$ in a DDPM-style model with a Gaussian mixture target, with $\sigma_t = O\left(\beta_t + \sqrt{ \tilde \beta_t} \right)$, where $\beta_t$ is the forward noise schedule and $\tilde \beta_t$ is the reverse one.
    \item In order to connect $\hat{q}(x_0)$ to the true target, $p^*(x_0)$, we provide a non-parametric estimator $\hat{v}_t$ with asymptotically zero MSE (under suitable assumptions), and propagate its error to show that 
    \[\textstyle d_{TV}(\hat{q}(x_0), p^*(x_0)) \leq d_{TV}(\hat{q}(x_0), \hat{p}(x_0)) + o_n(1), \]
    where $n$ is the number of samples used to pretrain $\hat{v}_t$  (Theorem \ref{thm:tv error rejection sampling}, Appendix \ref{apdx:value estimation}).
    \item We conduct experiments in both continuous ($2$ dimensional mixture of Gaussians) and discrete (LLaDA \citep{nie2025large}) settings. For the former, we find that LCBs can obtain similar alignment strength to BoN and Rejection Sampling with $7\times$ and $12.5\times$ fewer queries to the pretrained model respectively. For the latter, we obtain comparable alignment to BoN with $20-40\%$ fewer queries in a fixed-prompt setting. We also experimentally demonstrate that LCBs generalize to the prompt-conditioned setting.  
\end{enumerate}


\noindent\textbf{Related work}  The objective in \eqref{eq:tradeoff} and its analysis via soft-value functions (Section \ref{sec:preliminaries}) has become a common starting point for research in inference, safety and alignment \citep{uehara2025tutorial}. Though alignment is often accomplished by finetuning \citep{fan2023dpok}, we focus on inference-time alignment methods.  
Many existing inference-time alignment methods for sampling from \eqref{eq:tradeoff} depend on specific assumptions on the architecture of $\ppre$ or reward function $r$ \citep{dhariwal2021diffusion,ho2022classifierfree, kawar2022denoising,wu2023practical, nisonoff2025unlockingguidancediscretestatespace,deng2023rad,singh2025code}, and are thus not directly comparable with LCBs. 
Our method belongs to a growing class of \emph{model-agnostic} inference-time alignment approaches that require only black-box sampling access to $\ppre(x_t \mid x_{t+1})$. Prominent examples include Sequential Monte Carlo/Particle Filters (SMC/PF) \citep{li2024derivativefreeguidancecontinuousdiscrete, wu2023practical, dou2024diffusion, cardoso2024monte}, which jointly propagate and reweight $K$ generation trajectories, and Best-of-$N$ (BoN) \citep{beirami2025theoretical}, which selects the highest-reward sample among $N$ draws from $\ppre(x_0)$. The drawback of these methods is that their parameters ($K \geq 2$ and $N \geq 2$ respectively) correspond to fixed multipliers to the compute required to service a single model query. Moreover, existing theoretical analyses require $K$ and $N$ to grow asymptotically large. 
In contrast, LCBs offer \emph{non-asymptotic}, end-to-end total-variation bounds on sampling error, with a tunable approximation–compute tradeoff governed by a continuous hyperparameter $\delta$ rather than discrete choices of $K$ or $N$. LCBs also offer the flexibility to significantly boost reward by scaling compute by a factor much smaller than $2$.

\section{Preliminaries}\label{sec:preliminaries}
We take a pretrained Diffusion Model (DM) to be a Markov process specified by a set of conditional distributions $\{\ppre_t(x_t | x_{t + 1})\}_{t =0}^{T - 1}$ together with an initial distribution $\ppre_T(x_T)$ over a desired sampling space $\mathcal{X}$. $T$ is the number of generation steps used by the model. Here $\cal X$ can be either continuous or discrete, where the former corresponds to continuous diffusion models (e.g., DDPM \citep{ho2020denoising}) and the latter, discrete ones (e.g., D3PM \citep{austin2021structured}).  Our setting also includes autoregressive LLMs with bounded output size (where $\mathcal X = V^T$ is the set of strings of tokens). The setting where $\ppre$ is prompt conditioned is also covered by redefining $x_t = [x_t, \textrm{prompt}]$. 

Independently of the DM, we consider a (possibly non-differentiable) target reward function $r_0: \mathcal{X} \rightarrow [-B_0, B_0]$, where $B_0 > 0$ is a constant. The alignment temperature in \eqref{eq:tradeoff} is a constant $\alpha > 0$, which we may take to be $\alpha = 1$ without loss of generality, by defining $r := r_0/\alpha$ with $B := B_0/\alpha$.

Having fully specified the joint distribution of $x_0, ..., x_T$ under $\ppre$, we freely marginalize and condition on any subset of timesteps. From \citep{uehara2025tutorial}, it is known that the solution to \eqref{eq:tradeoff}, $p^*(x_0) \propto e^{r(x_0)}\ppre(x_0)$, is obtained by marginalizing over $x_1, ..., x_T$ in the following ``aligned process'':
\begin{align}\label{eq:pstar conditional definition}
p^*_t(x_t | x_{t + 1}) &:= e^{v_t(x_t) - v_{t + 1}(x_{t + 1})}\ppre_t(x_t | x_{t + 1});\quad 
p^*(x_T) := e^{v_T(x_T) - v_{T + 1}}\ppre(x_T),
\end{align}
where $v_t: \mathcal{X} \rightarrow [-B, B]$ is the well-known \emph{soft-value function} \citep{levine2018rlasinference}:
\begin{align}\label{eq:soft value definition} 
    v_t(x_t) &:= \log \E_{x_0 \sim \ppre(\cdot | x_t)}[e^{r(x_0)}],  \ t \in [T]; \quad 
    v_{T+1} := \log \E_{x_0 \sim \ppre}[e^{r(x_0)}],
\end{align}
so-called because $v_t(x_t)$ is the value of \eqref{eq:tradeoff} where the model is started at time $t$ conditioned on the state $x_t$. 
Together, Equations \eqref{eq:pstar conditional definition} and \eqref{eq:soft value definition} suggest the following three step alignment process:\\
\textbf{Step 1:} Obtain $n$ trajectories from the pretrained model, $\ppre$, and form the dataset $\{\left (x^i_T, ..., x^i_0, r(x^i_0)\right)\}_{i \in [n]}$;\\
\textbf{Step 2:} Form an estimate of the soft-value function, $\hat{v}_t(x_t)$ using the data from Step 1. Also estimate any auxiliary quantities required for sampling;\\
\textbf{Step 3:} Sample from the process with tilted kernels $\hat{p}(x_t | x_{t + 1}) \propto e^{\hat{v}_t(x_t)}\ppre(x_t | x_{t + 1})$.

The costs in Steps 1--2 are one-time (or infrequent) preprocessing costs, whereas Step 3 is paid for each generated sample. This paper therefore focuses primarily on efficiently solving the sampling problem in Step 3, which is nontrivial when $\mathcal{X}$ is high-dimensional.
In order to state fully end-to-end TV bounds, however, we also briefly give a recipe for consistently estimating $\hat{v}_t$ in Appendix \ref{apdx:value estimation}.  

\section{Warm up: Rejection Sampling with Soft-Value Estimates}\label{sec:rejection sampling}

In the body of this paper, we assume access to an estimator, $\hat{v}_t$, of the soft value function, $v_t$. Plugging $\hat{v}_t$ in for $v_t$ in \eqref{eq:pstar conditional definition} leads to the estimated aligned process
\[\hat{p}_t(x_t | x_{t + 1}) \propto e^{\hat{v}_t(x_t)}\ppre_t(x_t | x_{t + 1}); \quad \hat{p}_T(x_T) \propto e^{\hat{v}_T(x_T)}\ppre_T(x_T).\] 

The following theorem \textbf{motivates sampling from the process $\hat{p}(x_0)$} by showing that samples from $\hat{p}(x_0)$ are close to samples from the optimal alignment distribution $p^*(x_0)$ defined in \eqref{eq:tradeoff}, provided that $\hat{v}_t$ is a good estimate of $v_t$.

\begin{theorem}[TV error of $\hat{p}$]
\label{thm:tv error rejection sampling}
The total variation error between $p^*(x_0)$ and $\hat{p}_0(x_0)$ is bounded as
\[ \textstyle d_{TV}\left(\hat{p}(x_0), p^*(x_0)\right) \leq \frac{e^{2B}}{2}\sqrt{\sum_{t = 0}^{T} \|\hat{v}_t - v_t\|^2_{L^2(p^{pre}_t)}}.\]
\end{theorem}

In Appendix \ref{apdx:value estimation}, we provide a consistent nonparametric estimator for $v_t$ as an auxiliary result. Our theory can also be applied to the case where $\{\hat{v}_t\}$ is any guidance signal, which need not estimate $v_t$. Though the interpretation in Equation \eqref{eq:tradeoff} is lost in such cases, the exponentially tilted kernels are still relevant, practical targets. For example, when $\hat{v}_t(x_t)$ is derived as a distilled a reward model or a heuristic proxy for final reward, it is still natural to target $\hat{p}(x_t | x_{t + 1}) \propto e^{\hat{v}_t(x_t)}\ppre_t(x_t | x_{t + 1})$ \citep{li2024derivativefreeguidancecontinuousdiscrete}. 


\textbf{How can we sample from $\hat{p}(x_0)$?} Since the reward signal is bounded as $|r(x_0)| \leq B$, we have $|v_t(x_t)| \leq B$ for all $t$. Thus, we may take $|\hat{v}_t(x_t)| \leq B$ WLOG. \footnote{If $\hat{v}_t$ exceeds $B$, clipping it reduces its MSE, and this can be absorbed into the definition of $\hat{v}_t$.} For any $x_{t + 1} \in \mathcal{X}$, we then have:
\begin{equation}\label{eq:exact rej sampling bound}
\frac{\hat{p}_t(x_t | x_{t + 1})}{\ppre_t(x_t | x_{t + 1})} \propto \frac{e^{\hat{v}_t(x_t)}\ppre_t(x_t | x_{t + 1})}{\ppre_t(x_t | x_{t + 1})} \leq e^{\hat{v}_t(x_t)}\leq e^{B}.
\end{equation}
By standard rejection sampling theory, this implies that, given $x_{t + 1}$, accepting proposals $x_t \sim \ppre_t(\cdot | x_{t + 1})$ with probability $e^{\hat{v}_t(x_t) - B}$ gives exact samples from $\hat{p}(x_t | x_{t + 1})$. Thus, using rejection sampling at every timestep $t = T, \dots, 0$ produces exact samples from $\hat{p}(x_0)$, as follows.

\begin{figure}[t]
\centering

\begin{minipage}[t]{0.48\textwidth}
\hrule
\captionof{algorithm}{Exact rejection sampling for tilted diffusion transitions}
\label{alg:exact_rejection_sampling}
\hrule

\footnotesize
\begin{algorithmic}[1]
\INPUT Sampling access to $p^{\mathrm{pre}}_{T}(x_T)$ and diffusion kernels
$\{p^{\mathrm{pre}}_{t}(x_t \mid x_{t+1})\}_{t=0}^{T-1}$;
$\{\hat v_t:\mathcal{X}\to\mathbb{R}\}_{t=0}^{T}$; $B$

\REPEAT
    \STATE Propose $x_T \sim p^{\mathrm{pre}}_{T}(\cdot)$

\UNTIL{$\mathrm{Unif}(0,1) \le \exp\!\big(\hat v_T(x_T) - B\big)$}

\FOR{$t = T-1, T-2, \dots, 0$}
    \REPEAT
        \STATE Propose $x_t \sim p^{\mathrm{pre}}_{t}(\cdot \mid x_{t+1})$
    \UNTIL{$ \mathrm{Unif}(0,1) \le \exp(\hat{v}_t(x_t) - B)$}
\ENDFOR
\STATE \textbf{return} $x_0$
\end{algorithmic}

\hrule
\end{minipage}
\hfill
\begin{minipage}[t]{0.48\textwidth}
\hrule
\captionof{algorithm}{Approximate rejection sampling with baseline}
\label{alg:baseline rejection sampling}
\hrule

\footnotesize
\begin{algorithmic}[1]
\INPUT All inputs to Algorithm~\ref{alg:exact_rejection_sampling}, \textbf{plus}
baseline functions $\{B_t:\mathcal{X}\to\mathbb{R}\}_{t=1}^{T + 1}$

\REPEAT
    \STATE Propose $x_T \sim p^{\mathrm{pre}}_{T}(\cdot)$
\UNTIL{$ \mathrm{Unif}(0,1) \le \exp\left({\hat v_T(x_T) - B_{T+1}(\bot)}\right)$}

\FOR{$t = T-1, T-2, \dots, 0$}
    \REPEAT
        \STATE Propose $x_t \sim \ppre_{t}(\cdot \mid x_{t+1})$
    \UNTIL{$\mathrm{Unif}(0, 1) \le \exp\left(\hat v_t(x_t) - B_{t+1}(x_{t+1})\right)$}
\ENDFOR
\STATE \textbf{return} $x_0$
\end{algorithmic}

\hrule
\end{minipage}

\end{figure}

\begin{proposition}[Rejection sampling for $\hat{p}$]\label{prop:rejection sampling is exact}
The sampler in Algorithm \ref{alg:exact_rejection_sampling} produces exact trajectories from $\hat{p}(x_T, \dots, x_0)$, and, in particular, $x_0$ is distributed as $\hat{p}(x_0)$.
\end{proposition}

The proofs for Theorem \ref{thm:tv error rejection sampling} and Proposition \ref{prop:rejection sampling is exact} are provided in Appendix \ref{apdx:rejection sampling theory}.


\noindent\textbf{Proposal Complexity}
Although Algorithm~\ref{alg:exact_rejection_sampling} perfectly targets $\hat{p}_t(x_t | x_{t + 1})$, it
requires a large number of proposals from the pretrained model at each iteration. 
Indeed, the expected total number of proposals, $\E[N]$, with Algorithm \ref{alg:exact_rejection_sampling} is approximately
\begin{equation*}
     \E[N] \approx \sum_{t = 0}^{T-1} \E_{x_{t + 1} \sim \hat p_{t + 1}}\left[\frac{1}{\E_{x_t \sim \ppre_t(\cdot | x_{t + 1})}[e^{\hat{v}_t(x_t) - B}]}\right].
\end{equation*}
According to our problem definition, $B$  might be a loose upper bound for $|r(x_0)|$, in which case the only obtainable bound is $\E[N] \leq e^{2B}T$. 
Furthermore, since $B = B_0/\alpha$ implicitly subsumes the temperature term, stronger alignment (i.e. lower $\alpha$) is computationally infeasible via rejection sampling. 
To remedy this, we introduce a theory of sampling with baselines, which will ultimately allow us to adaptively set efficient acceptance probabilities in complex tasks.

\section{Efficient Sampling with Baseline Functions}

Intuitively, the issue with Algorithm \ref{alg:exact_rejection_sampling} is that it is only effective when each $x_t$ satisfies $\hat{v}_t(x_t) \approx B$. If not, the probability of states with high acceptance rate will be exponentially small.

We now explore the use of ``baseline functions'' for adaptively setting acceptance probabilities. The idea is to speed up sampling by \emph{adaptively and approximately} setting the constant on the RHS of \eqref{eq:exact rej sampling bound} in an $x_{t + 1}$-dependent way. To facilitate this, we introduce a user-specified parameter $\delta > 0$ that trades off between sampling complexity and approximation quality. 

\begin{definition}\label{def:joint baseline}
We say $B_{t + 1}(x_{t + 1})$ is a \textbf{joint baseline} for $f_t: \mathcal{X} \rightarrow \R$ under $(\ppre_t(\cdot | x_{t + 1}), q_{t + 1}(x_{t + 1}))$ at level $\delta$, if 
\[\Pr_{\substack{x_{t + 1} \sim q_{t + 1} \\ x_t \sim \ppre_t(\cdot | x_{t + 1})}}[f_t(x_t) > B_{t + 1}(x_{t + 1})] \leq \delta. \]
\end{definition}

When we take $f_t := \hat{v}_t$ in Definition \ref{def:joint baseline}, $B_{t + 1}(x_{t + 1})$ can (approximately) take the role of $B$ in \eqref{eq:exact rej sampling bound}. Intuitively, joint baselines allow the sampler to decide early in the denoising process whether a trajectory will be good or bad, and adaptively allocate compute.

For a fixed set of baseline functions $\{B_{t}\}_{t = 1}^{T + 1}$, we can produce samples using Algorithm \ref{alg:baseline rejection sampling}. Let us define $\hat{q}_T(x_T)$ as the distribution induced by accepting proposals from $\ppre_T(x_T)$ with probability $\min\{1, e^{\hat{v}_T(x_T) - B_{T + 1}}\}$, where $B_{T + 1}$ is a constant-valued baseline. Then, we inductively define 
\begin{equation}\label{eq:hat q inductive def}
\hat q_t(x_t) = \int \hat q_{t + 1}(x_{t + 1})\hat q_t(x_t | x_{t + 1})dx_{t + 1},
\end{equation}
where $\hat q_t(x_t | x_{t + 1})$ is the distribution induced by performing rejection sampling with proposal $\ppre_t(x_t | x_{t + 1})$ and acceptance probability $\min\{1, e^{\hat v_t(x_t)- B_{t + 1}(x_{t + 1})}\}$. 

We say a set of joint baselines  $\{B_{t}\}$ is \emph{mutually compatible} if $B_{t + 1}(x_{t + 1})$ is a joint baseline under $(\ppre_t(\cdot | x_{t}), \hat{q}_{t + 1}(x_{t+1}))$, for $\hat{q}_{t + 1}$ defined as in \eqref{eq:hat q inductive def}. Mutual compatibility ensures that each baseline remains a valid joint baseline under the distribution induced by applying all baselines together, making the set of baselines self-consistent, and allowing per-step approximation errors to compose.

Before showing how baseline functions can be learned using data, we analyze Algorithm \ref{alg:baseline rejection sampling}'s  TV error and proposal complexity in terms of fine-grained properties of the baseline functions used. 
 
\subsection{Fine-grained TV bounds}\label{sec:finegrained tv theory}

To analyze the sampling error of Algorithm \ref{alg:baseline rejection sampling}, we introduce the following worst-case \emph{coverage assumption}. In particular, a joint baseline $B_{t + 1}(\cdot)$ has \textbf{worst-case coverage at level $c \in (0,1 )$} if 
\[\exists c \in (0, 1),\quad \Pr_{x_t \sim \ppre_t(\cdot | x_{t + 1})}[\hat v_t(x_t) \geq B_{t + 1}(x_{t + 1})] \leq c,\quad \forall x_{t + 1}. \]
Importantly, we do not require that $c\rightarrow 0$ for our results, and it suffices for $c$ to be a constant like $\frac{1}{2}$. The worst-case coverage assumption is mainly made for ease of exposition. In Appendix \ref{apdx:coverage}, we show that it can be weakened to an empirically verifiable ``high-probability coverage assumption.'' 

We now state a TV bound for the transitions induced by Algorithm \ref{alg:baseline rejection sampling}, depending on fine-grained properties of the baseline function. The following is proved as Lemma \ref{lem:tv mgf apdx} in the Appendix.

\begin{lemma}[TV:MGF Lemma]\label{lem:tv mgf body}
Fix an estimate of the soft-value function $\hat{v}_t$ and an arbitrary distribution $q_{t + 1}(x_{t + 1})$. Suppose that $B_{t + 1}(x_{t + 1}) = b_{t + 1}(x_{t + 1}) + \tau$ is a joint baseline for $\hat{v}_t$ under $(\ppre_t(\cdot | x_{t + 1}), q_{t + 1}(x_{t + 1}))$ at level $\delta \in (0, 1)$, with worst-case coverage at level $c \in (0,1 )$.
Suppose there is $\lambda \geq 2$ so that
$e^{-\lambda \tau}M(\lambda)\leq \delta$
where \[M(\lambda) := \E_{\substack{x_{t + 1} \sim q_{t + 1} \\ x_t \sim \ppre_t(\cdot | x_{t + 1})}}\left[e^{\lambda (\hat{v}_t(x_t) - b_{t + 1}(x_{t + 1}))}\right].\]
Then, for the $\hat{q}_t(x_t | x_{t + 1})$ induced by Algorithm \ref{alg:baseline rejection sampling}, we have the sampling error bound:
\begin{equation*}
\E_{x_{t + 1} \sim q_{t + 1}}[d_{TV}(\hat{q}_t(\cdot | x_{t + 1}), \hat{p}_t(\cdot | x_{t + 1}))]  \\ \lesssim \textstyle \frac{\delta^{1 - 2/\lambda}}{(1 - c)^2}(M(\lambda)M(-\lambda))^{1/\lambda} .\label{eq:tv mgf expression}
\end{equation*}
\end{lemma}

Lemma \ref{lem:tv mgf body} establishes that a baseline $b_{t + 1}(\cdot ) + \tau$ that is ``Chernoff certified'' --  in the sense that the Chernoff bound reproduces the joint baseline condition (Definition \ref{def:joint baseline}) -- can be evaluated in terms of the MGF $\hat{v}_t - \hat{b}_{t + 1}$. 
Lemma \ref{lem:tv mgf body}'s expression can (a) be used explicitly \textbf{as an optimization objective} in order to learn baselines that directly target low sampling error; and (b) be analytically controlled in special cases (e.g. Corollary \ref{cor:subgaussian value corollary}), in order to furnish \textbf{witnesses} that certify low error for (a). 

Summing up over $t$ further yields an upper bound on $d_{TV}(\hat{q}(x_0), \hat{p}(x_0))$, as follows.


\begin{lemma}[Lemma \ref{lem:joint TV decomp apdx} in the Appendix]\label{lem:joint TV decomp body}
The output of Algorithm \ref{alg:baseline rejection sampling} satisfies
\begin{equation*}
\textstyle d_{TV}(\hat{q}(x_0), \hat{p}(x_0)) \leq d_{TV}(\hat{q}_T, \hat{p}_T)
+ \sum_{t = 0}^{T - 1} \E_{x_{t +1} \sim \hat q_{t + 1}}[d_{TV}(\hat{q}_t(\cdot | x_{t + 1}), \hat{p}_t(\cdot | x_{t + 1}))].
\end{equation*}
\end{lemma}

\noindent\textbf{Proposal complexity.} To see why baseline sampling might be more efficient, we bound the number of proposals from $\ppre_t$ required at each step of Algorithm \ref{alg:baseline rejection sampling} in terms of the MGF of $\hat{v}_t - B_{t + 1}(x_{t + 1})$. 

\begin{lemma}[Lemma \ref{lem:prop complexity generic apdx} in the Appendix] \label{lem:prop complexity generic}
Let $B_{t + 1}$ be a joint baseline for $\hat{v}_t$ with worst-case coverage at level $c$. Let $N_t$ be the number of proposals used to sample $x_t\sim \hat{q}_t(\cdot | x_{t + 1})$. Then, the expected number of proposals for Algorithm \ref{alg:baseline rejection sampling} is
\begin{align*}
    &\textstyle \E[N_t | x_{t + 1}] \leq \frac{1}{(1 - c)^2}\E_{x_t \sim \ppre_t(\cdot | x_{t + 1})}[e^{B_{t + 1}(x_{t + 1}) - \hat{v}_t(x_t)}].
\end{align*}
\end{lemma}
Given $x_{t + 1}$, if $B_{t + 1}(x_{t + 1})$ is a good estimate of $\hat{v}_t(x_t)$, then the upper bound above can be far smaller than $\E_{x_t \sim \ppre_t(\cdot | x_{t + 1})}[e^{B - \hat{v}_t(x_t)}]$, the approximate cost of na\"ive rejection sampling.

\subsection{Learnable Chernoff Baselines: Learning Baselines from Data}\label{sec:lcb theory}

We now use our theory to train joint baselines for the estimated soft-value function. As with learning the soft-value functions $\{\hat{v}_t\}$, learning baselines $\{B_{t + 1}\}$ corresponds to a single upfront cost that is then amortized by cheaper inferences. 

Let $\delta \in (0, 1)$ be a  \textbf{user-specified} parameter that selects a tradeoff between sampling accuracy and efficiency. Fix a timestep $t \in [T - 1]$. According to Definition \ref{def:joint baseline}, let us inductively construct a  \textbf{joint baseline} for $\{\hat{v}_t\}$ under $(\hat{q}_{t + 1}, \ppre_t)$, where $\hat{q}_{t + 1}$ is the marginal distribution of $x_{t + 1}$ under the first $T - t$ steps of sampling. After doing this, we will see how to learn mutually compatible baselines $B_{T + 1} \dots B_1$ across all timesteps.  

 
Note that all probabilities and expectations in this section are defined with respect to the joint distribution over $(x_{t + 1}, x_{t })$ specified by $x_{t + 1} \sim \hat{q}_{t + 1}$ and $x_t | x_{t + 1} \sim \ppre_t(x_t | x_{t + 1})$. 
Consider parameterizing the baseline function as  $ B_{t + 1}(x_{t + 1}) =  b_{t + 1}(x_{t + 1}) + \tau$, where $b_{t + 1} \in H$ is a member of a hypothesis class $H$ and $\tau$ is constant with respect to $x_{t + 1}$. Letting $\lambda \geq 0$ and using Chernoff bound, we can analyze the condition in Definition \ref{def:joint baseline} as
\begin{align*}
\textstyle \Pr_{\substack {x_{t + 1} \sim \hat{q}_{t + 1} \\ x_t \sim \ppre_t(\cdot | x_{t + 1})}}[\hat v_t(x_t) - {b}_{t + 1}(x_{t + 1}) \geq \tau] \leq e^{-\lambda \tau}\E_{\substack {x_{t + 1} \sim \hat{q}_{t + 1} \\ x_t \sim \ppre_t(\cdot | x_{t + 1})}}\left[e^{\lambda (\hat{v}_t(x_t) -  b_{t + 1}(x_{t + 1}))}\right].
\end{align*}
To ensure that $B_{t + 1}$ is a baseline at level $\delta$, the RHS must be $\leq \delta$, leading to a $(\lambda, b_{t + 1})$-dependent choice:
\[\textstyle \tau \gets \tau_{\lambda,  b} := \frac{1}{\lambda}(\log\frac{1}{\delta}+ \log \E e^{\lambda (\hat{v}_t(x_t) - b_{t + 1}(x_{t + 1})}).\]
In other words, for any $b_{t + 1}$, we can select $\tau_{\lambda, b}$ in order to make $B^{(\lambda, b)}_{t + 1}(x_{t + 1}) := b_{t + 1}(x_{t + 1}) + \tau_{\lambda, b}$ a joint baseline at level $\delta$, where $\lambda > 0$ is a free parameter.  

Moreover, for any choice of $\lambda \geq 2$ and $b_{t + 1}$, the baseline $B^{(\lambda, b)}_{t + 1}$ satisfies the Chernoff condition of the TV-MGF Lemma (Lemma \ref{lem:tv mgf body}), which then gives us:
\begin{align*}
\E_{x_{t + 1} \sim q_{t+1}}&[d_{TV}\left(\hat{q}_t(\cdot | x_{t + 1}), \hat{p}_t(\cdot | x_{t + 1})\right)] \lesssim \delta e^{\frac{2}{\lambda}\log{\frac{1}{\delta}+\frac{1}{\lambda}\log M(\lambda) + \frac{1}{\lambda}\log M(-\lambda)}},
\end{align*}
where $M(\lambda) := \E[e^{\lambda (\hat{v}_t(x_t) - b_{t + 1}(x_{t + 1}))}]$. 

We propose choosing $\lambda \in (2, \Lambda)$ and $b_{t + 1} \in H$ in order to minimize the right hand side, for some fixed ceiling $\Lambda \in \R_{>2}$. Taking logs leads to the objective function
\begin{align*}
J_t(\lambda, b_{t + 1}) :=& \textstyle \frac{1}{\lambda}\log\E[e^{\lambda (\hat v_t(x_t) - b_{t + 1}(x_{ t+ 1}))}] \textstyle + \frac{1}{\lambda}\log \E [e^{\lambda(b_{ t+ 1}(x_{ t+ 1}) - \hat v_t (x_t))}] + \frac{2}{\lambda}\log \frac{1}{\delta} .
\end{align*}
We use samples to approximate the expectations and produce $\hat J_t: (2, \Lambda) \times H \rightarrow \R$:
\begin{align*}
\hat J_t(\lambda, b_{t + 1}) &:= \textstyle\frac{1}{\lambda}\log\hat\E[e^{\lambda (\hat v_t(x_t) - b_{t + 1}(x_{ t+ 1}))}] \textstyle + \frac{1}{\lambda}\log \hat\E [e^{\lambda(b_{ t+ 1}(x_{ t+ 1}) - \hat v_t (x_t))}] + \frac{2}{\lambda}\log \frac{1}{\delta} .
\end{align*}
In particular, we may safely choose $(\hat{\lambda}, \hat{b}_{t + 1}) \in \arg\min_{\lambda \in (2, \Lambda), b \in H}\hat J_t(\lambda, b)$ and then construct $\hat\tau_{\hat\lambda,  \hat b} = (\log\frac{1}{\delta}+ \log \hat \E e^{\hat \lambda (\hat{v}_t(x_t) - \hat b_{t + 1}(x_{t + 1})})/\hat \lambda$ using the same data.
This is borne out by the following theorem, proved as Corollary \ref{cor:lcb learning bound apdx} in the Appendix, which establishes uniform convergence results for the empirical risk minimizer of $J_t$.

\begin{proposition}[Learning the LCB]\label{prop:lcb learning bounds}

Given $m$ samples $(x^i_{t + 1}, x^i_t) \sim (\hat{q}_{t + 1}, \ppre_t)$ and the associated $\hat{J}_t$, let $\hat{\lambda}, \hat{b}_{t + 1} \in \arg \min_{\lambda \in (2, \Lambda), b \in H}\hat J(\lambda, b)$. Let $J^*_t$ be the minimum value of $J_t$ (assuming realized). Then, with probability $1 - \eta$ over the data,
\[J_t(\hat\lambda, \hat{b}_{t + 1}) - J^*_t\leq 2\epsilon_0 \ \ \  \textrm{    and    }\ \ \ \hat \tau_{\hat \lambda, \hat{b}_{t + 1}} \in \tau_{\hat{\lambda}, \hat{b}_{t + 1}} \pm \epsilon_0.\]
where all empirical quantities are estimated with the same samples and $\epsilon_{0} := 4e^{4B\Lambda}(\hat{\mathcal{{R}}}_m(H) + \sqrt{\frac{\log4\Lambda m/\eta}{2m}} ) + 16B/m$, and $\hat{\mathcal{{R}}}_m$ is the standard empirical Rademacher complexity (cf. \citep{Mohri2018FML}). 

\end{proposition}

Thus, for each $t \in [T]$, we can generate a Learnable Chernoff Baseline (LCB), which is simply a joint baseline in the form $\hat B_{t + 1}(x_{t + 1}) = B^{(\lambda, b)}_{t + 1}(x_{t + 1}) = \hat b_{t + 1}(x_{t + 1}) + \hat \tau_{\hat \lambda, \hat b} + \epsilon_0$. Given LCBs $\hat{B}_{T + 1} \dots \hat{B}_1$ for each soft-value function $\hat{v}_T, \dots \hat{v}_0$, we can then generate samples according to Algorithm \ref{alg:baseline rejection sampling}.

\subsubsection{Training procedure}\label{subsec:training procedure}

Theoretically, Proposition \ref{prop:lcb learning bounds} requires samples $\{(x^i_{t + 1}, x^i_t)\}_{i = 1}^m \sim ({\hat{q}_{t + 1}, \ppre_t})^{\otimes m}$ to form $\hat{J}_t$ and $\hat{\tau}_{\hat{\lambda}, \hat b_{t + 1}}$. Crucially, we must train on $\{x^i_{t + 1}\}_{i = 1}^m \sim \hat{q}_{t + 1}$ from the sampling distribution associated to the baselines $B_{T + 1} \cdots B_{t + 2}$ in order to keep $\{B_{T + 1} \dots B_{t + 1}\}$ mutually compatible.

In Algorithm \ref{alg:two pass lcb training} of Appendix \ref{apdx:training alg}, we describe a training procedure for generating a full set of mutually compatible baselines $B_{T + 1} \dots B_1$. Given $\hat B_{T + 1} \dots \hat B_{t + 2}$, Algorithm \ref{alg:two pass lcb training} trains $\hat{B}_{t + 1}$ by generating i.i.d. $x^1_{t + 1} \dots x^m_{t + 1} \sim \hat{q}_{t + 1}$ from scratch, and then sampling $y^i_t \sim \ppre_t(\cdot | x^i_{t + 1})$ for each $i \in [m]$. This produces i.i.d. training pairs of the kind required for Proposition \ref{prop:lcb learning bounds}. A simple union bound over $T+1$ steps of training allows the good event of Proposition \ref{prop:lcb learning bounds} to hold simultaneously for all learned baselines, with $\epsilon_0$ replaced by \eqref{eq:epsilon 0 T}. 

Algorithm \ref{alg:two pass lcb training} requires $O(mT^2)$ computation. However, if $H$ is finite, we may reuse a single set of particles, and train $B_{T + 1} \dots B_1$ in time $O(mT)$ (Algorithm \ref{alg:one pass lcb training}) with analogous guarantees. 

Due to space constraints, we give details and pseudocode for each training procedure in Appendix \ref{apdx:training alg}. We discuss other practical aspects of optimizing $J_t$ in Appendix \ref{apdx:optimization properties}.

\begin{figure*}[!t]
\centering
\includegraphics[width=1\textwidth]{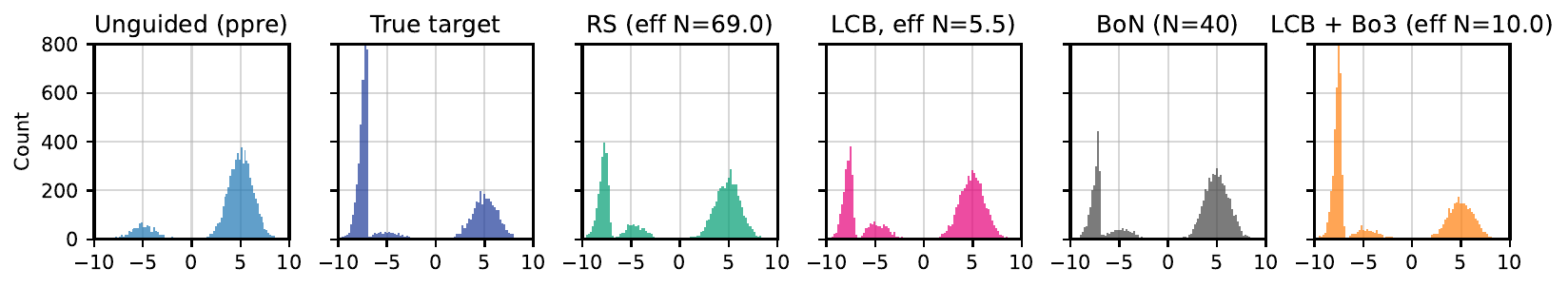}
\caption{Setting I (Gaussian mixture): LCB matches RS using $8\%$ of the proposals. It uses $14\%$ of the proposals required by Bo$N$ for comparable alignment. By chaining LCB with BoN, we can efficiently boost alignment to compensate for error in value estimation. }\label{fig:1d hists}
\end{figure*}

\begin{figure*}[!b]
\centering
\captionsetup{font=small,labelfont=bf}

\begin{minipage}[c]{0.48\textwidth}
\centering
\scriptsize
\setlength{\tabcolsep}{2.2pt}
\renewcommand{\arraystretch}{0.95}
\resizebox{\linewidth}{!}{%
\begin{tabular}{@{}lccc@{}}
\toprule
\textbf{Method} & \textbf{Rew.} & \textbf{Perf. Rate (\%)} & \textbf{Eff. $N$} \\
\midrule
RS 0.5 & $0.822 \pm 0.023$ & $64.40 \pm 4.20$ & $1.56 \pm 0.06$ \\
RS 0.3 & $0.863 \pm 0.020$ & $70.60 \pm 3.99$ & $1.98 \pm 0.13$ \\
RS 0.1 & $0.923 \pm 0.016$ & $83.20 \pm 3.28$ & $12.60 \pm 3.62$ \\
\midrule
LCB 0.5 & $0.827 \pm 0.023$ & $65.00 \pm 4.18$ & $1.16 \pm 0.01$ \\
LCB 0.3 & $0.885 \pm 0.021$ & $74.20 \pm 3.84$ & $1.21 \pm 0.02$ \\
LCB 0.1 & $0.917 \pm 0.017$ & $81.80 \pm 3.38$ & $1.59 \pm 0.06$ \\
\midrule
Best of 2 & $0.893 \pm 0.017$ & $74.60 \pm 3.82$ & $2.00 \pm 0.00$ \\
Unguided & $0.766 \pm 0.017$ & $52.00 \pm 3.10$ & $1.00 \pm 0.00$ \\
\bottomrule
\end{tabular}%
}
\end{minipage}
\hfill
\begin{minipage}[c]{0.48\textwidth}
\centering
\includegraphics[height=4cm]{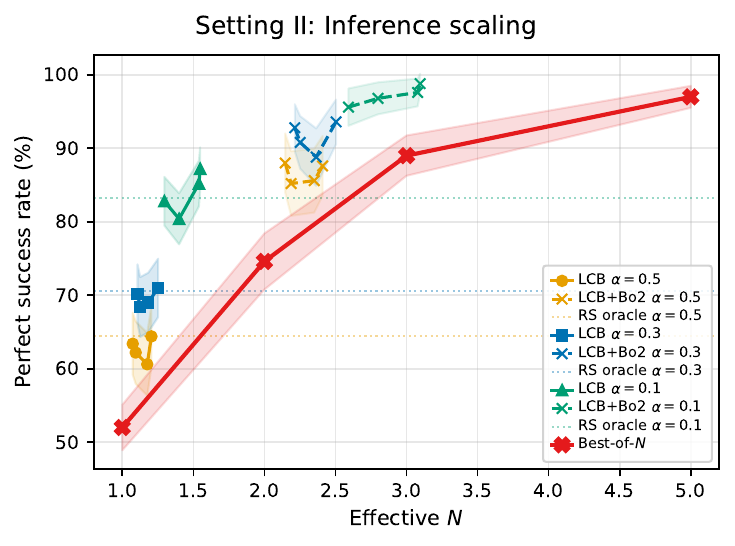}
\end{minipage}

\caption{
Setting II: We run LCBs with $\alpha = 0.5, 0.3, 0.1$, $\delta = 0.5$. Left: LCB imitates RS distribution, but is much more efficient than Bo2 at boosting reward statistics. Right: Reward/compute inference scaling plot. For each $\alpha$, we perform inference with $\delta \in \{0.8, 0.6, 0.5, 0.3\}$, displayed as connected lines. When composed with a Bo2 pass, these policies saturate the reward quickly. Every LCB-based method lies above the BoN Pareto front. CI $95\%$.
}
\label{fig:success_efficiency}
\end{figure*}

\subsubsection{Properties of LCBs}\label{sec:properties of lcbs}

As a corollary of the TV:MGF Lemma (Lemma \ref{lem:tv mgf body}) and Proposition \ref{prop:lcb learning bounds}, we obtain the following bound on the sampling error of LCBs, proved in Appendix \ref{apdx:lcb main thm proofs}.

\begin{theorem}\label{thm:lcb tv bound}
Let $\hat \lambda, \hat b_{t + 1}$ be the ERMs of $\hat{J}_t$. Let $\hat B_{t + 1} := \hat{b}_{t + 1}(x_{t + 1}) + \hat{\tau}_{\hat {\lambda}, \hat{b}_{t + 1}} + \epsilon_0$ be the associated LCB, and assume that it also satisfies worst-case coverage at level $c \in (0, 1)$. Then, on the good training event from Proposition \ref{prop:lcb learning bounds}, using $\hat B_{t + 1}$ in Algorithm \ref{alg:baseline rejection sampling} leads to: 
\[\textstyle \E_{x_{t + 1} \sim \hat{q}_{t + 1}}[d_{TV}(\hat{q}_{t}(\cdot | x_{t + 1}), \hat{p}(\cdot | x_{t + 1}))] \leq \frac{2\delta}{(1 - c)^2} e^{J^*_t + 2\epsilon_0}.\]
\end{theorem}

\paragraph{End-to-end sampling guarantee} Consider $\hat{B}_{T + 1} \dots \hat{B_1}$ learned from Algorithm \ref{alg:two pass lcb training}. As we saw in Section \ref{subsec:training procedure}, the good training event of Proposition \ref{prop:lcb learning bounds} holds uniformly over time steps. Thus, with probability $1 - \eta$ over the training process, using those baselines in Algorithm \ref{alg:baseline rejection sampling} induces a sampling distribution $\hat{q}(x_0)$ satisfying 
\begin{equation}\label{eq:e2e tv}d_{TV}(\hat{q}(x_0),\hat{p}(x_0)) \leq \frac{2\delta}{(1 - c)^2}\sum_t e^{J^*_t + 2\epsilon_{0, T}}.\end{equation}

Suppose $\hat v_t(x_t)$ is $\sigma^2_t$-sub-Gaussian under $\ppre(x_t | x_{t + 1})$, where $\sigma_t$ does not depend on $x_{ t + 1}$. For $b^*(x_{t + 1}) := \E[\hat{v}(x_{t}) | x_{t + 1}]$ and $\lambda^* := \sqrt{\frac{2\log\frac{1}{\delta}}{\sigma_t^2}}$, we have the upper bound $J(\lambda^*, b^*) \leq 2\sigma_t\sqrt{2\log\frac{1}{\delta}}$. This leads to the following corollary, which quantifies the LCB performance in the subgaussian case.
\begin{corollary}\label{cor:subgaussian value corollary}
Consider the setting of Theorem \ref{thm:lcb tv bound}. Further, suppose there exists a constant $\sigma_t^2$ so that $\hat{v}_t(x_t)$ is $\sigma_t^2$ subgaussian under $\ppre_t(x_t | x_{t+1})$, for each $x_{t + 1}$. Further assume that $\delta \leq e^{-2\sigma_t^2}$, $\Lambda \geq \sqrt{\frac{2\log1/\delta}{\sigma_t^2}}$ and that $\E[\hat{v}_t(x_t) | x_{t + 1} = (\cdot)] \in H$. Then, given good training from Proposition \ref{prop:lcb learning bounds}:
\begin{align*}
\E_{x_{t + 1} \sim \hat{q}_{t + 1}}&[d_{TV}(\hat{q}_{t}(\cdot | x_{t + 1}), \hat{p}(\cdot | x_{t + 1}))] \leq \textstyle \frac{2\delta}{(1 - c)^2} e^{2\sigma_t\sqrt{2\log\frac{1}{\delta}} + 2\epsilon_0}.
\end{align*}
This bound may likewise be composed to bound final sampling error, as in \eqref{eq:e2e tv}.
\end{corollary}

A more general version of Corollary \ref{cor:subgaussian value corollary} is proved in the Appendix as Corollary \ref{cor:subgaussian val corollary apdx} without the restriction that $\delta < e^{-2\sigma^2_t}$.
For the proposal complexity, we have 
that 
$\frac{1}{(1 - c)^2}e^{J(\hat{\lambda}, \hat{b}_{t + 1}) + 2\epsilon_0}$ upper bounds the number of proposals used by the LCB $\hat B _{t + 1}$ (Appendix Lemma \ref{lem:apdx lcb objective bounds proposals}).

\subsection{Example: Multivariate Gaussian Mixture DDPM}\label{sec:mog theory}

As an example, we consider the special case of our setting where $\{\ppre_t\}_{t = 0}^T$ targets a Gaussian mixture with uniform covariances across clusters. 
Namely, define $m^k \in \R^d$, $\Sigma \in \R^{d\times d}$, and let:
\begin{align*}
& \textstyle p_m(x_0) = \sum_{k = 1}^K \pi_k N(m^k, \Sigma);\quad  p_m(x_{t+1} | x_t) \sim \sqrt{ \alpha_t}x_t + N(0, \beta_tI); \quad \beta_t := (1 -  \alpha_t)
\end{align*}
for all $t \in [T - 1]$. Then, we take as our model $\ppre_t(x_t | x_{t + 1}) = p_m(x_t | x_{t + 1})$, $\ppre_T(x_T) = p_m(x_T)$ (i.e., assume that we exactly invert the forward noising process\footnote{Since we assume perfect inversion, this is not quite the same as running DDPM \cite{ho2020denoising} on $p_m$}). We define $\tilde \beta_t = \|\mathrm{Cov}_{p_m}[x_t | x_{t + 1}, k]\|$, where $\|\cdot\|$ is the operator norm and $k$ is any latent cluster center. 

In this setting, Corollary \ref{cor:subgaussian value corollary} applies with $\sigma = O\left(\beta_t + \sqrt{\tilde \beta_t}\right)$, as we state in the following Theorem (proved in the Appendix as Theorem \ref{thm:mog thm apdx}). 

\begin{theorem}\label{thm:mog thm}
Consider the setting of $\ppre$ described above and assume $r: \R^d \rightarrow [-B, B]$ is Lipschitz. Let $a_{t + 1}(x_{t + 1}) = \E_{p_m}[v_t(x_t) | x_{t + 1}]$. For any $x_{t+1}$, we have that $v_t(x_t) - a_{t + 1}(x_{t + 1})$ is conditionally sub-gaussian under $\ppre(\cdot | x_{t+1})$, with parameter $\sigma_a = O(\beta_t + \sqrt{\tilde \beta_t})$, where $O(\cdot)$ hides dimension-independent factors of $\{m^k\}$, $\Sigma$, and $L_r$.

Therefore, in the setting of Corollary \ref{cor:subgaussian value corollary}, given $\hat{v}_t \gets v_t$, LCB sampling achieves the bound 
\begin{align*}
\E_{x_{t + 1} \sim \hat{q}_{t + 1}}&[d_{TV}(\hat{q}_{t}(\cdot | x_{t + 1}), \hat{p}(\cdot | x_{t + 1}))] \leq \textstyle \frac{2\delta}{(1 - c)^2} e^{O(\sqrt{\tilde \beta_t} + \beta_t)\sqrt{\log\frac{1}{\delta}} + 2\epsilon_0}.
\end{align*}

\end{theorem}

In Appendix \ref{apdx:MoG}, we also \textbf{bound the expected number of proposals} when the witness baseline functions are plugged into Algorithm \ref{alg:baseline rejection sampling}.

\section{Empirical Results}\label{sec:empircs}

We validate our theory in continuous and discrete experimental settings, verifying that LCB samples correctly emulate Algorithm \ref{alg:exact_rejection_sampling}'s Rejection Sampling (RS) at a fraction of the cost. We also compare with Best-of-N (BoN), which is a strong baseline for model agnostic alignment \citep{ huang2025bestofn}. In order to discuss efficiency of rejection-style procedures, we introduce the quantity ``effective N'' as the average inference compute multiplier induced by requerying the pretrained model.

\subsection{Setting I: Mixture of Gaussians}\label{sec:mog empirics}
We run score-based DDPM \citep{ho2020denoising} on the 2-d mixture of Gaussians $p_{\mathrm{target}} = \ppre = 0.05\times \mathcal{N}([-5, 0]^T, I) + 0.95\times \mathcal{N}([5, 0]^T, I)$. The reward function is $r(x) = 1\{x < -7\}$. This setting tests LCBs in the challenging case where the reward region has very small mass ($\ppre(r^{-1}(1)) \approx 0.001$) under the base model, $\ppre(x_0)$.

In Figure 1, we see in the middle 2 panels that LCBs closely match the RS samples from $\hat{p}(x_0)$, with an effective $N$ of $5.5$ rather than $69$, matching the behavior we expect from Corollary \ref{cor:subgaussian value corollary}. Samples from LCB and RS both differ from $p^*(x_0)$ due to error in value estimation, but by composing with a small Bo3 pass, LCBs can cheaply be amplified to match $p^*(x_0)$ closely. Details in Appendix \ref{apdx:mog experiments}.

\subsection{LLaDA}\label{sec:llada experiments}

\begin{figure}[t]
    \centering
    \includegraphics[width = 0.8\textwidth]{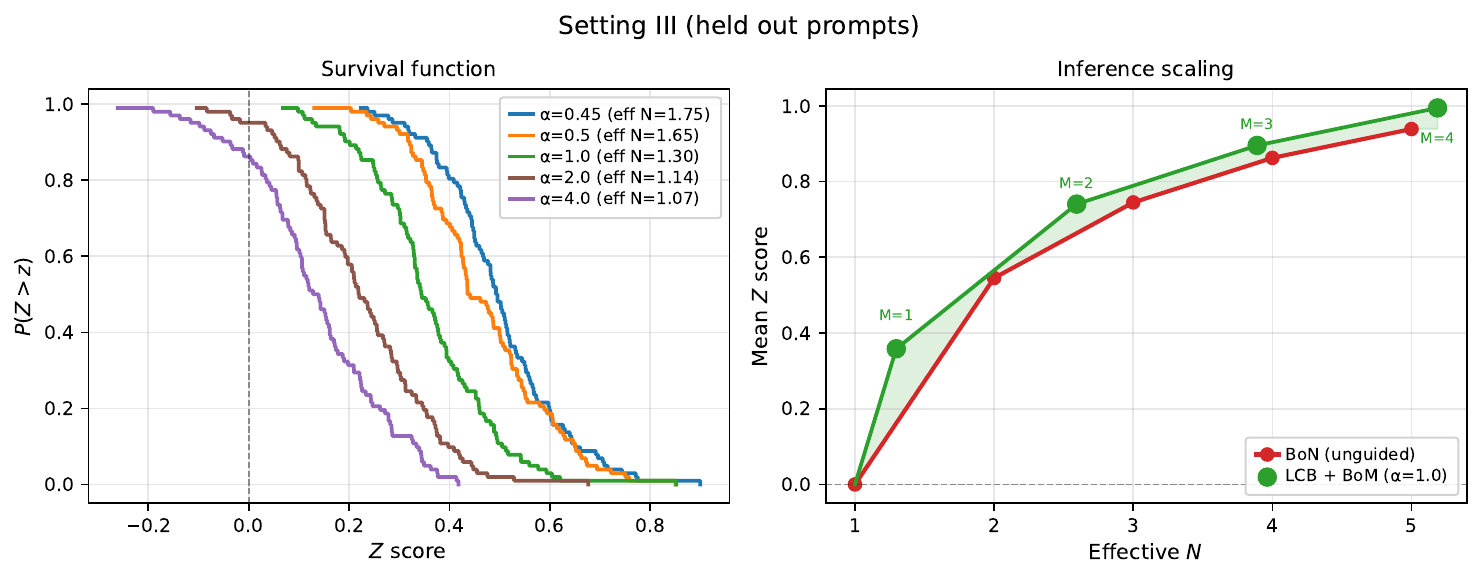}
    \caption{
        Setting III results on held-out prompts. Left: survival function of per-prompt $Z$ scores (calculated under the unguided distribution), shows that guidance offers significant per-reward improvement. 
        Right: $Z$-score/compute inference scaling plot. LCB with $\alpha = 1$ is composed with BoM passes. The optimal path traced by $\alpha = 1$ is shown in green.
    }
    \label{fig:setting-iii-survival-scaling}
\end{figure}
We now evaluate the efficiency of LCBs on \texttt{LLaDA-8B} \citep{nie2025large}, a masked language diffusion model of 8 billion parameters.

\paragraph{Setting II: Verifiable reward}
We first test in a fixed-prompt setting so that we can test with a verifiable syntax-based reward. We sample LLaDA responses to the prompt: \emph{``Write a creative story of exactly three sentences; the first word of each sentence should start with A, B, and C, respectively.''}

The reward function, $r(x_0)$, dispenses a reward of $1/3$ for each first-letter constraint that is met. 
Notably, this syntax-based reward function has the interesting property that future reward depends on long-range syntactic decisions that can change late in generation (e.g. the model can place the first “.” on the last step). In isolating this phenomenon, this setting shows that LCBs can make complex inferences in a long-horizon high-dimensional environment.

In Figure 2, we show that LCBs correctly emulate $\hat{p}(x_0)$ while remaining more efficient than BoN and RS. In Appendix \ref{apdx:llada experiments}, we offer more details and an extensive comparison with SMC methods.

\paragraph{Setting III: Prompt-conditioned reward model}

Finally, we seek to demonstrate that LCBs generalize effectively in the prompt-conditioned LLaDA setting. Specifically, we first train LCBs on prompt/response pairs, and then we evaluate the performance on a holdout set of prompts. 
We create 512 prompts of the form \emph{``Complete the sentence in a polite but firm tone: ``Thank you for your $\langle s\rangle$, but...''''}, where $s$ is from a set of ``stems'' that require very different completions (e.g. “bid,” “confidential tip,” etc.). The generated content is then scored by Skywork-Reward-V2-Qwen3-1.7B \citep{liu2025skywork}. We train $\hat{v}_t$ (about $82$ million parameters) by finetuning on a training set of $410$ prompts, and we will evaluate on $102$ held-out prompts. When training the LCB, we let $b_{t + 1}(x_{t + 1}) := v_{t + 1}(x_{t + 1})$, so learning the LCB reduces to learning $\{\lambda_{t + 1}\}, \{\tau_{t + 1}\}$: a simple $1$d unimodal optimization problem. \textbf{Therefore, the LCB carries essentially no training cost apart from learning $\hat{v}_t$.} 

In Figure 3, we demonstrate that LCBs generalize to the held-out prompts, shifting prompt-induced reward rightward as $\alpha$ shrinks. While we utilized a lightweight setup without hyperparameter optimization, we anticipate that carefully scaling the value model and dataset will substantially improve performance by reducing the noise in value predictions. Yet, even with this minimal setup, LCB-based results strictly dominate the Pareto front for BoN. We offer more results in Appendix \ref{apdx:setting III prompted llada}.

\section{Conclusion}

We introduced Learnable Chernoff Baselines (LCBs) as a method for performing quick sampling from tilted Markov processes. We have shown theoretically and experimentally LCBs --unlike rejection sampling-- remain efficient and accurate in high dimensions. Our experimental results should be construed as validating our theory, and larger-scale evaluation of LCBs is important future work. Two opposite but complementary future directions are using training-free guidance signals instead of $\hat{v}_t$, and scaling $\hat{v}_t$ models to more data. As with all guidance methods, LCBs are limited by the quality of their guidance signal, which is an active area of study in the literature \citep{li2024derivativefreeguidancecontinuousdiscrete}. Black box guidance methods are dual use, and should not be used with anti-social reward functions. 

\section{Acknowledgments and Disclosures}
The work of Sunil Madhow and Yu-Xiang Wang was partially supported by NSF DMS 2134214, TILOS Institute at UCSD, as well as a gift by Google.
The work of Yuchen Liang, Ness Shroff, and Yingbin Liang was supported in part by the U.S. National Science Foundation under the NSF AI Institute (AI-EDGE) 2112471 and under the grant ECCS-2515482. The work of Ness Shroff was also supported in part by the Army Research Laboratory under Cooperative Agreement Number W911NF-23-2-0225, by the U.S. National Science Foundation under the grants CNS-2312836, CNS-2225561, and CNS-2239677, and by the Office of Naval Research under grant N00014-24-1-2729. The views and conclusions contained in this document are those of the authors and should not be interpreted as representing the official policies, either expressed or implied, of the Army Research Laboratory or the U.S. Government. The U.S. Government is authorized to reproduce and distribute reprints for Government purposes notwithstanding any copyright notation herein.

\bibliography{citations.bib}
\bibliographystyle{abbrvnat}
\nocite{*}

\appendix

\appendix
\onecolumn
\section{Notation}

For a natural number, $a$, we let $[a] = \{1, \dots, a\}$. For a random variable $X$, an expectation $\E[f(X)]$, and $n$ realizations $x_1, \dots x_n$ of $X$, we define $\hat{\E}[f(X)] := \frac{1}{n}\sum_i f(x_i)$. For two probability measures $P$ and $Q$, the total-variation distance between them is defined to be $d_{TV}(P, Q) = \sup_A |Q(A) - P(A)|$, where the supremum is over events. When $Q$ and $P$ have densities $q$ and $p$ with respect to a shared measure $\lambda$, this definition is equivalent to $\frac{1}{2}\int_x |q(x) - p(x)|d\lambda(x)$. $D(p || q)$ is the KL divergence between $p$ and $q$. For a random variable $X: \Omega \rightarrow \R$, we define its $L^p$ norm as $\|X\|_{L^p}^p = \int |X(\omega)|^p d\omega$. 
Also see Table~\ref{table:notations} for other notations used in this paper.

\begin{table}[t]
\centering
\begin{tabular}{p{0.12\linewidth} p{0.26\linewidth} p{0.56\linewidth}}
\hline
\textbf{Process} & \textbf{Role} & \textbf{Definition / how it is sampled} \\
\hline

$\ppre$ &
Pretrained (base) diffusion process &
Initial distribution $\ppre_T(x_T)$ and reverse-time kernels $\{\ppre_t(x_t\mid x_{t+1})\}_{t=0}^{T-1}$. Sampling: draw $x_T\sim \ppre_T$ and then for $t=T-1,\dots,0$, sample $x_t\sim \ppre_t(\cdot\mid x_{t+1})$. \\[0.6em]

$p^*$ &
Ideal reward-aligned target process (true soft values) &
Defines the desired aligned marginal $p^*(x_0)\propto e^{r(x_0)}\,\ppre(x_0)$. It can be represented by reverse-time kernels obtained by tilting $\ppre_t(\cdot\mid x_{t+1})$ using the true soft-value function $\{v_t\}_{t=0}^T$. \\[0.6em]

$\hat p$ &
Value-tilted process using learned $\hat v$ (exact RS target) &
Reverse-time kernels are the learned-value tilts
\[
\hat p_t(x_t\mid x_{t+1}) \propto e^{\hat v_t(x_t)}\,\ppre_t(x_t\mid x_{t+1}),
\]
(with an analogous definition for $\hat p_T$). This is the process targeted by \emph{exact} per-step rejection sampling when using $\hat v_t$. \\[0.6em]

$\hat q$ &
Baseline-induced process (approximate sampler) &
The process induced by baselined rejection sampling: propose from $\ppre_t(\cdot\mid x_{t+1})$ and accept with probability
\[
\min\{1,\exp(\hat v_t(x_t)-B_{t+1}(x_{t+1}))\}.
\]
The resulting accepted transition defines $\hat q_t(\cdot\mid x_{t+1})$, and marginals are obtained by propagation
\[
\hat q_t(x_t)=\int \hat q_{t+1}(x_{t+1})\,\hat q_t(x_t\mid x_{t+1})\,dx_{t+1}.
\]
\\
\hline
\end{tabular}
\caption{Summary of the base process $\ppre$, ideal aligned process $p^*$, learned-value tilted process $\hat p$, and baseline-induced sampling process $\hat q$.}
\label{table:notations}
\end{table}

\section{Related works}\label{apdx:related work}

The objective presented in \eqref{eq:tradeoff} and its analysis via soft-value functions (Section \ref{sec:preliminaries}) are common starting points for research in safety and alignment, though recent work has challenged the use of KL divergence as regularization \citep{huang2025correcting}. The conditional kernels $\hat{p}(x_t | x_{t + 1})$ can be targeted both by fine-tuning methods and inference-time methods. 

Finetuning approaches include applying PPO to reward plus divergence \citep{fan2023dpok}, or ignoring the divergence constraint altogether \citep{ICLR2024_14f75513}. 

To our knowledge, the most classical examples of  inference-time alignment procedures are classifier guidance -- which steers continuous diffusion processes by means of an external classifier-- and classifier-free guidance \citep{ho2022classifierfree}, which moves the classifier into the base model and sets a guidance scale. Both are primarily applied to continuous Gaussian diffusion. Classifier guidance corresponds to the special case of our setting where $r(x_0) = \log p(\texttt{class} | x_0)$, where $\texttt{class}$ is the target class. The principled extension of classifier guidance is to inject $\nabla_x \hat{v}_t(x_t)$ at every step of diffusion. 

Recent works use the soft-value formulation to perform guidance in LLMs by scoring prefixes \citep{deng2023rad, liu2024inference}. In the diffusion context, \citep{li2024derivativefreeguidancecontinuousdiscrete, yoon2025psisampler} use SMC techniques to sample from a distribution proportional to $\E[r(x_0) | x_t]\ppre(x_t | x_{t + 1})$. \citep{kim2025inferencetimescaling} seeks to extend such inference-time techniques to deterministic flow models, which is relevant to this work as our work also primarily addresses the case where $\ppre(\cdot | x_{t + 1})$ is not a Dirac delta. This paper was inspired by \citep{uehara2025tutorial}, which sets down an RL-inspired framework for alignment. Unlike existing works, the proposals in this paper (a) provide theoretical guarantees on the sampling distribution (b) scale compute adaptively according to the problem instance, where SMC methods scale inference-time compute by an integer corresponding to the number of particles.

Finally, it has been widely observed that Best-of-N (BoN), which simply draws $N$ samples from $\ppre(x_0)$ and picks the sample with the highest reward, is a very difficult baseline to ``beat'' \citep{lin2024bonbon, huang2025bestofn, beirami2025theoretical}. Despite its lack of adaptivity to the base diffusion model, BoN is empirically shown to achieve very competitive ``win rates'' (rate at which BoN output beats unaligned output) at a fixed KL budget \citep{beirami2025theoretical}. However, it is not well-known what kind of variational problem BoN is solving, except in simple special cases \citep{10619456}, and BoN is known to suffer from ``reward hacking'' for large $N$ \citep{huang2025bestofn}. More importantly, like SMC, BoN locks us into at least doubling the cost of inference (with the minimum choice of $N = 2$), which may be beatable with more adaptive methods. 

Technically, the LCB objective we propose is related to tilted empirical risk minimization \citep{li2023tiltederm}, though we arrived at it independently.


\section{Proofs for Section \ref{sec:rejection sampling}}\label{apdx:rejection sampling theory}

\subsection{Proof of Theorem \ref{thm:tv error rejection sampling}}

By Pinsker's inequality and the data processing inequality, we have 
\[d_{TV}\left(\hat{p}(x_0), p^*(x_0)\right) \leq 
\sqrt{\frac{1}{2}\mathrm{KL}\left(p^*(x_0) \| \hat p(x_0)\right)}  
\leq \sqrt{\frac{1}{2}\mathrm{KL}\left(p^*(x_T \dots x_0)\| \hat{p}(x_T \dots x_0)\right)}\]

In turn, by the chain rule of KL divergence and the Markov assumption:
\[\mathrm{KL}\left(p^*(x_T \dots x_0), \hat{p}(x_T \dots x_0)\right) =\mathrm{KL}(p^*_T || \hat{p}_T) + \sum_{t = 0}^{T - 1} \E_{x_{t + 1} \sim p^*_{t + 1}}[\mathrm{KL}(p^*_t(\cdot | x_{t + 1}) || \hat{p}_t(\cdot | x_{t + 1})]\]

Since $p^*_{t + 1}(x_{t + 1}) \leq e^{2B}\ppre_{t + 1}(x_{t + 1})$, we can conclude 

\begin{multline}\label{eq:pinsker boound}d_{TV}\left(\hat{p}(x_0), p^*(x_0)\right) \leq \\ \sqrt{\frac{1}{2}\left(\mathrm{KL}(p^*_T || \hat{p}_T) + e^{2B}\sum_{t = T - 1}^0 \E_{x_{t + 1} \sim \ppre_{t + 1}(x_{t + 1})}[\mathrm{KL}( p^*_t(\cdot | x_{t + 1}) || \hat{p}_t(\cdot | x_{t + 1})\right)}\end{multline}

The following Lemma now tells us how to control the KL divergence between exponential tilts of the same base distribution. 

\begin{lemma}
Let $q_a(x) \propto e^{a(x)}p(x)$ and $q_b(x) \propto e^{b(x)}p(x)$, with $|a(x)|\leq B$ and $|b(x)| \leq B$. Then 
\[\mathrm{KL}(q_a || q_b) \leq \frac{e^{2B}\|a(x) - b(x)\|_{L^2(p)}^2}{2}\]
\end{lemma}

\begin{proof}
Let us define $\Delta(x) := b(x) - a(x)$. Define $\phi(t) := \log \E_{x \sim q_a} e^{t\Delta(x)}$. Let $Z_a$ and $Z_b$ be the normalizing constants for $q_a$ and $q_b$ 
respectively. 

Notice that 

\[\frac{Z_b}{Z_a} = \int \frac{e^{b(x)}p(x)}{\int e^{a(x)}p(x)dx}dx = \int \frac{e^{b(x)-a(x)}e^{a(x)}p(x)}{\int e^{a(x)}p(x)dx}dx = \E_{x \sim q_a}[e^{\Delta(x)}]\]

We have
\[\mathrm{KL}(q_a || q_b) = {\E_{x \sim q_a}[a(x) - b(x)]} + {\log{\frac{Z_b}{Z_a}}}\]
\begin{equation}\label{eq:ab kl divergence}= \E_{x \sim q_a}[-\Delta(x)] + \phi(1) \end{equation}

We now Taylor expand $\phi$ to the second order. Since $\phi(t)$ is the log-partition function of the natural exponential family with sufficient statistic $\Delta(x)$ and base distribution $q_a$, we have

\[\phi'(t) = \E_{x \sim h_t}[\Delta(x)]\]

where $h_t(x) = e^{t\Delta(x)- \phi(t)}q_a(x)$

Also, for $t \in [0, 1]$
\begin{align*}
\phi''(t) =& \Var_{x \sim h_t}[\Delta(x)] \leq \E_{x \sim h_t}[\Delta(x)^2]\\
=& \E_{x \sim p}\left[\frac{e^{a(x)}}{\E_p[e^{a(x)}]}\frac{e^{t(b(x) - a(x))}}{\E_p[\frac{e^{a(x)}}{\E_p[e^{a(x)}]}e^{t(b(x) - a(x))}]}\Delta^2\right]\\
=& \E_p [\frac{e^{(1 - t)a(x) + t b(x)}}{\E[e^{(1 - t)a(x) + t b(x)}]}\Delta^2]\\
\leq& e^{2B}\|[\Delta\|_{L^2(p)}^2
\end{align*}

since the $(1 - t, t)$ average of $a(x)$ and $b(x)$ is always between $-B$ and $B$. 

By Taylor's theorem with Lagrange-type residual, there exists $\xi \in (0, 1)$ such that 
\begin{align*}
\phi(1) &= \phi(0) + \phi'(0) + \phi''(\xi)/2 \\
&= \E_{q_a}[\Delta] + \phi''(\xi)/2\\
&\leq \E_{q_a}[\Delta]  + \frac{e^{2B}}{2}\|\Delta\|_{L^2(p)}^2
\end{align*}

Plugging this bound on $\phi(1)$ into \eqref{eq:ab kl divergence} completes the proof. 
\end{proof}

We now just need to apply the above lemma to \eqref{eq:pinsker boound}. 

From taking $p(x) := \ppre_T(x)$, $a(x) := v_T(x)$, $b(x) := \hat{v}_T(x)$, we have:

\[\mathrm{KL}(p^*_T || \hat{p}_T) \leq e^{2B}\|v_T - \hat{v}_T\|_{L^2(\ppre_T)}^2.\]

From taking $p(x) := \ppre_t(x | x_{t + 1})$, $a(x) := v_t(x)$, $b(x) := \hat{v}_t(x)$, we have:

\[\mathrm{KL}\left(p^*_t(\cdot | x_{t + 1}) || \hat{p}_t(\cdot | x_{t + 1})\right) \leq e^{2B}\|v_t - \hat{v}_t\|_{L^2(\ppre_t(\cdot | x_{t + 1}))}^2/2.\]

And upon taking an outer expectation over $x_{t + 1} \sim \ppre_{t + 1}(x_{t + 1})$, we use the tower property of expectations to get

\[\E_{x_{t + 1} \sim \ppre_{t + 1}}\left[\mathrm{KL}\left(p^*_t(\cdot | x_{t + 1}) || \hat{p}_t(\cdot | x_{t + 1})\right)\right] \leq e^{2B}\|v_t - \hat{v}_t\|_{L^2(\ppre_t)}^2/2.\]

This leads to the conclusion:

\[d_{TV}\left(\hat{p}(x_0), p^*(x_0)\right) \leq \frac{e^{2B}}{2}\sqrt{\sum_{t = 0}^T\|\hat{v}_t - v_t\|_{L^2(\ppre_t)}^2}.\]

\subsection{Proof of Proposition \ref{prop:rejection sampling is exact}}
\begin{proposition}
\label{prop:apdx rejections sampling is exact}
    Assume $x_T \dots x_0$ is generated by Algorithm \ref{alg:exact_rejection_sampling}. For any $t = 0, ... T$, we have that $x_t$ is distributed according to $\hat{p}_t(\cdot)$.
\end{proposition}
\begin{proof}
By induction. The base case is standard rejection sampling. Now suppose that $x_{t + 1} \sim \hat p_{t + 1}(\cdot)$. We will show that $x_t \sim \hat{p}_t (\cdot).$

Let $\tilde p_t(x_t) = \ppre_t(x_t \mid x_{t + 1})e^{\hat{v}_t(x_t)}$ be the unnormalized target distribution. Let $q_t(x_t) = \ppre_t(x_t \mid x_{t + 1})$ be the proposal distribution. It holds that $\tilde p_t(x_t) \leq e^B q_t(x_t)$. 

The joint likelihood of proposing $x_t$ and accepting it is 
\[
q_t(x_t) \frac{e^{\hat v_t (x_t)}}{e^B} = \tilde p_t(x_t)/ e^B,
\]
meaning that the marginal probability of accepting a draw is 
\[
\int q_t(x_t) \frac{e^{\hat v_t (x_t)}}{e^B}dx_t = \frac{Z_t(x_{t+1})}{e^B},
\]
where $Z_t(x_{t+1})$ is the normalizing constant of $\tilde p_t(\cdot)$ (given $x_{t+1}$).

Thus the distribution of $x_t$ conditioned on acceptance is 
\[
\frac{\tilde p_t(x_t)}{e^B}\times \frac{e^B}{Z_t(x_{t+1})} = \frac{\tilde p_t(x_t)}{Z_t(x_{t+1})}.
\]

Thus, $x_t \sim \hat{p}_t(x_t \mid x_{t + 1})$. Given the inductive hypothesis, we're done. 
\end{proof}

\section{General sampling theory for baselines}

\subsection{Conditional Baselines}

In addition to the notion of a \textbf{joint} baseline, we now introduce \textbf{conditional} baselines, which satisfy the baseline condition $x_{t+1}$-a.e. In the language of the body of the paper, the following definition is equivalent to saying $B_{t + 1}$ satisfies a worst-case coverage condition at level $\delta$. 

\begin{definition}\label{def:conditional baseline}
We say $B_{t + 1}(x_{t + 1})$ is a \textbf{conditional} baseline for $f_t: \mathcal{X} \rightarrow \R$ under $\ppre_t(\cdot | x_{t + 1})$ at level $\delta$, if for all $x_{t + 1} \in \mathcal{X}$
\[\Pr_{x_t \sim \ppre_t(\cdot | x_{t + 1})}[f_t(x_t) > B_{t + 1}(x_{t + 1})] \leq \delta .\]
\end{definition}

Conditional baselines are clearly stronger than joint baselines, by the law of total expectation. When we establish the theory background for Theorem \ref{thm:mog thm} in Appendix \ref{apdx:MoG}, we will also derive conditional baselines for the Mixture of Gaussians setting. 

\subsection{Proposal Complexity}
\begin{lemma}\label{lem:prop complexity generic apdx}[Lemma \ref{lem:prop complexity generic}]
Let $B_{t + 1}$ be a conditional baseline for $\hat{v}_t$ at level $\eta$, and $N_{x_{t + 1}}$ is the number of proposals used to sample $x_t\sim \hat{q}_t(\cdot | x_{t + 1})$. Then, for any distribution on $x_{t + 1}$, the number of samples, $N_t$, used by Algorithm \ref{alg:baseline rejection sampling} to sample $x_t$ 
\[\E[N_t | x_{t + 1}] \leq \frac{1}{(1 - \eta)^2}\E_{x_t}[e^{B_{t + 1}(x_{t + 1}) - \hat{v}_t(x_t)}].\]
\end{lemma}
\begin{proof}
Define $\eta := \Pr[F | x_{t + 1}]$, where $F = \{x_t \ :\  \hat{v}_t(x_t) \geq B_{t + 1}(x_{t + 1})\}$. Irrespective of whether $B_{t + 1}$ is a conditional baseline or a joint baseline, we have that 

\begin{align*}
\E[N | x_{t + 1}] &= \frac{1}{\E_{x_t}[\min\{1, e^{\hat{v}_t(x_t) - B_{t + 1}(x_{t + 1})}\}]}\\
&= \frac{1}{\eta + (1 - \eta)\E[e^{\hat{v}_t(x_t) - B_{t + 1}(x_{t + 1})}| F^c]}\\
& \leq \frac{1}{(1 - \eta)\E[e^{\hat{v}_t(x_t) - B_{t + 1}(x_{t + 1})} | F^c]}\\
&\leq \frac{1}{(1 - \eta)}\E[e^{B_{t + 1}(x_{t + 1}) - \hat{v}_t(x_t)} | F^c] \quad \tag{(Jensen's applied to $x \mapsto 1/x$)}\\
&\leq \frac{1}{(1 - \eta)}\E[e^{B_{t + 1}(x_{t + 1}) - \hat{v_t}(x_t)}1\{F^c\}]\frac{1}{\Pr_{x_{t + 1}}[F^c]}\\
&= \frac{1}{(1 - \eta)^2}\E[e^{B_{t + 1}(x_{t + 1}) - \hat{v_t}(x_t)}1\{F^c\}]\\
&\leq \frac{1}{(1 - \eta)^2}\E_{x_t}[e^{B_{t + 1}(x_{t + 1}) - \hat{v_t}(x_t)}]
\end{align*}

Taking an outer expectation $x_{t + 1}$ completes the proof.
\end{proof}

\begin{lemma}[LCB objective upper bounds expected number of proposals]\label{lem:apdx lcb objective bounds proposals}
Fix a timestep $t \in [T]$. For any $\lambda \geq 1$ and $b_{t + 1}(\cdot)$, let $B^{\lambda, b}_{t + 1}$ be the associated (population) LCB under $x_{t+1} \sim \hat{q}_{t + 1}$. Assume $B^{\lambda, b}_{t + 1}$ satisfies a worst-case coverage assumption at level $c \in (0, 1)$.  Let $N$ be the number of proposals used to sample $x_t\sim \hat{q}_t(\cdot | x_{t + 1})$. Then the expected number of proposals, jointly over the source $x_{t+1}$ and draws $x_t | x_{t + 1}$
\[\E_{x_{t + 1} \sim \hat{q}_{t + 1}}[N_t] \leq \frac{1}{(1 - c)^2}e^{J(\lambda, b_{t + 1})} \]
\end{lemma}
\begin{proof}
From Lemma \ref{lem:prop complexity generic}, we have that 

\[\E_{x_{t + 1}}[N] \leq \frac{1}{(1 - c)^2}\E_{x_{t + 1}\sim \hat{q}_{t + 1}}\E_{x_t  \sim \ppre_t(x_t | x_{t + 1})}[e^{B^{\lambda, b}_{t + 1}(x_{t + 1}) - \hat{v}_t(x_t)}]\]

By the definition of LCBs, $B^{\lambda, b}_{t + 1}(x_{t + 1}) = b_{t + 1}(x_{t + 1}) + \tau_{\lambda, b}$, which gives

\[\E_{x_{t + 1}}[N] \leq \frac{1}{(1 - c)^2}\E_{x_{t + 1}\sim \hat{q}_{t + 1}}\E_{x_t  \sim \ppre_t(x_t | x_{t + 1})}[e^{b_{t + 1}(x_{t + 1}) - \hat{v}_t(x_t)}]e^{\tau_{\lambda, b}}\]

since  $\tau_{\lambda, b} := \frac{1}{\lambda}\log \E_{x_{t + 1} \sim \hat{q}_{t + 1}(x_{t + 1})} \E_{x_t \sim \ppre (\cdot | x_{ t= 1})}[e^{\lambda\left(\hat{v}_t(x_t) - b_{t + 1}(x_{t + 1})\right)}] + \frac{\log\frac{1}{\delta}}{\lambda}$, we have the upper bound 

\begin{align*}&\E_{x_{t + 1}}[N] \\ &\leq \frac{1}{(1 - c)^2}\E_{x_{t + 1}\sim \hat{q}_{t + 1}}\E_{x_t  \sim \ppre_t(x_t | x_{t + 1})}[e^{b_{t + 1}(x_{t + 1}) - \hat{v}_t(x_t)}]\\
&\qquad \cdot \left(\E_{x_{t + 1} \sim \hat{q}_{t + 1}(x_{t + 1})} \E_{x_t \sim \ppre (\cdot | x_{ t= 1})}[e^{\lambda \left(\hat{v}_t(x_t) - b_{t + 1}(x_{t + 1})\right)}]\right)^{1/\lambda}\delta^{-1/\lambda}\end{align*}

Applying Jensen's inequality with $z \mapsto z^{\lambda}$ to the first term shows that
\begin{multline*}
    \E_{x_{t + 1}\sim \hat{q}_{t + 1}}\E_{x_t  \sim \ppre_t(x_t | x_{t + 1})}[e^{b_{t + 1}(x_{t + 1}) - \hat{v}_t(x_t)}] \\
    \leq \left(\E_{x_{t + 1}\sim \hat{q}_{t + 1}}\E_{x_t  \sim \ppre_t(x_t | x_{t + 1})}[e^{\lambda\left(b_{t + 1}(x_{t + 1}) - \hat{v}_t(x_t)\right)}]\right)^{\frac{1}{\lambda}}.
\end{multline*}

Therefore,
\begin{align*}
    \E_{x_{t + 1}}[N] &\leq \frac{1}{(1 - c)^2}e^{\frac{1}{\lambda}\log M(\lambda) + \frac{1}{\lambda}M(-\lambda) + \frac{1}{\lambda}\log\frac{1}{\delta}} \\
    &\leq \frac{1}{(1 - c)^2}e^{\frac{1}{\lambda}\log M(\lambda) + \frac{1}{\lambda}M(-\lambda) + \frac{2}{\lambda}\log\frac{1}{\delta}} \\
    &= \frac{1}{(1 - c)^2}e^{J(\lambda, b_{t + 1})}. \qedhere
\end{align*}
\end{proof}

The previous Lemma carries over directly to empirical LCBs by noting that $\hat{B}^{\lambda, b}_{t + 1}(x_{t + 1}) \leq B^{\lambda, b}_{t + 1} + 2\epsilon_0$, meaning we simply incur an extra scaling factor of $e^{2\epsilon_0}$.

\subsection{TV Analysis}
In this section, we present our results about the TV error of sampling with baselines. First, we give the classic decomposition into stepwise error. 

\begin{lemma}[TV timestep decomposition, Lemma \ref{lem:joint TV decomp body} in the main body]
    \label{lem:joint TV decomp apdx}
\[d_{TV}(\hat{q}(x_0), \hat{p}(x_0)) \leq d_{TV}(\hat{q}_T, \hat{p}_T) + \sum_{t = 0}^{T - 1} \E_{x_{t +1} \sim \hat q_{t + 1}(\cdot)}[d_{TV}(\hat{q}(x_t | x_{t + 1}), \hat{p}(x_t | x_{t + 1}))]\]
\end{lemma}
\begin{proof}

We work by (backward) induction. As our inductive hypothesis, we take
\[d_{TV}(\hat{q}_k(x_k), \hat{p}_k(x_k)) \leq d_{TV}(\hat{q}_T, \hat{p}_T) + \sum_{t = k}^{T - 1} \E_{x_{t + 1} \sim \hat{q}_{t + 1}}\left[d_{TV}(\hat q(x_t | x_{t + 1}), \hat p(x_t | x_{t + 1}))\right]\]

The base case at $k = T$ is definitional.

Denote by $\hat{q}_{k + 1} \hat{Q}_k := \int \hat{q}_{k + 1}(x_{k + 1})\hat q_k(x_k | x_{k + 1}) dx_{k + 1}$, and likewise define $\hat{p}_{k + 1} \hat P_k$, $\hat{q}_{k + 1} \hat P_k$, and so on.

We have
\begin{align*}
&d_{TV}(\hat{q}_{k + 1} \hat{Q}_k, \hat{p}_{k + 1} \hat{P}_k)\\
\leq& d_{TV}(\hat{q}_{k + 1} \hat{Q}_k, \hat{q}_{k + 1}\hat{P}_k) + d_{TV}(\hat{q}_{k + 1}\hat{P}_k, \hat{p}_{k + 1} \hat{P}_k).
\end{align*}

The first term is clearly
\[\E_{\hat{q}_{k + 1}}[d_{TV}(\hat{q}(x_k | x_{k + 1}), \hat{p}(x_k | x_{k + 1}))].\]

The second term is bounded by
\[d_{TV}(\hat{q}_{k + 1}, \hat{p}_{k + 1}),\]
as applying a transition kernel to two distributions cannot increase their TV distance.

Applying the inductive hypothesis to the second term completes the argument.
\end{proof}

The remainder of the section is concerned with bounding the summands. We begin with a simple measure-theoretic lemma.

\begin{lemma}\label{lem:const Radon}
Let $Q$ and $P$ be measures on a space $\mathcal{X}$ with $Q \ll P$. Assume there exists a set $E \subset\mathcal{X}$ so that $\frac{dQ}{dP}(x)$ is constant on $E$. Then 
\[d_{TV}(Q, P) \leq Q(E^c) + P(E^c)\]
\end{lemma}
\begin{proof}
To begin, observe that
\[d_{TV}(Q, P) = \frac{1}{2}\int_\mathcal{X} |\frac{dQ}{dP} - 1|dP = \frac{1}{2}\left(\int_E |\frac{dQ}{dP} - 1|dP + \int_{E^c}|\frac{dQ}{dP} - 1|dP\right).\]

Let $\kappa := \frac{dQ}{dP}|_E$. Then, 

\[Q(E) = \int_E \frac{dQ}{dP}dP = \kappa P(E) .\]

Therefore, we compute
\begin{align*}
\int_E \left|\frac{dQ}{dP} - 1\right|dP =& |\kappa - 1|\int_E dP\\
=& |Q(E) - P(E)|\\
=& |Q(E^c) - P(E^c)|\\
\leq& Q(E^c) + P(E^c).
\end{align*}
Moreover, by the triangle inequality we have
\[\int_{E^c}|\frac{dQ}{dP} - 1|dP \leq  \int_{E^c} \frac{dQ}{dP}dP + \int_{E^c}dP =  Q(E^c) + P(E^c).\]

Therefore
\[d_{TV}(Q, P) \leq \frac{1}{2}\left( Q(E^c) + P(E^c) + Q(E^c) + P(E^c)\right) = Q(E^c) + P(E^c).\]
\end{proof}

We now argue that under approximate rejection sampling, the ratio condition from the previous lemma is satisfied on the complement of the ``Exceedance Region'' -- the part of the sampling space where the envelope fails. 

\begin{lemma}\label{lem:generic error approx sampling}
Consider a target distribution $p(x) = \tilde p(x)/Z_p$, where $\tilde p$ is known but $Z_p$ is not. Consider a proposal distribution $\mu$ and constant $B \in \R_+$, with an exceedance set
\[F = \{x \ : \frac{\tilde p(x)}{\mu(x) B} >  1\}\]

Let $q$ be the distribution induced by accepting proposals from $\mu$ with probability \[a(x) = \min \{1, \frac{\tilde p(x)}{\mu(x) B}\}\]
Then $\frac{q(x)}{p(x)}$ is constant on $E := F^c$, and 

\begin{equation}\label{eq:cond baseline tv}d_{TV}(q, p)\leq \frac{\mu(F)}{\frac{Z_p p(E)}{B}} + p(F)\end{equation}

Moreover, if $p(x)/\mu(x) \in  (1/M, M)$ for all $x$, then we have 

\begin{equation}\label{eq:joint baseline tv}
d_{TV}(q, p) \leq \max\{\frac{MB}{Z_P}, 1\}\mu(F) +  M \mu(F)
\end{equation}

\end{lemma}
\begin{proof}
$q$ is given as 
\[q(x) = \frac{1\{x \in E\}\times \frac{\tilde p(x)}{\mu(x)B}\times \mu(x) + 1\{x \notin E\}\times  \mu(x)}{Z_q}\]

where \[Z_q = \int 1\{x \in E\}\times \frac{\tilde p(x)}{\mu(x)B}\times \mu(x) + 1\{x \notin E\}\times  \mu(x) dx \]
\[= Z_p p(E)/B + \mu(E^c) \geq Z_p p(E)/B\]

Clearly $q(x)/p(x)$ is constant on $E$ and Lemma \ref{lem:const Radon} applies:

\[d_{TV}(q, p) \leq q(F) + p(F)\]

The first term can be analyzed as 

\[q(F) = \int_F \frac{\mu(x)}{Z_q} = \frac{\mu(F)}{Z_q} \leq \frac{\mu(F)}{\frac{Z_p p(E)}{B}}\]

To prove the the final inequality, we repeat the argument, but instead lower bound $Z_q$ as
\begin{align*}Z_q &= Z_p p(E)/B + \mu(F) \\ 
&\geq Z_pM\mu(E)/BM + \mu (F) \\ 
&= (1 - \mu(F))Z_p/BM+\mu(F) \\ 
&\geq \min\{Z_p/BM, 1\}\end{align*}

where we use the fact that the average of $Z_p/BM$ and $1$ is at least $\min\{1, Z_p/BM\}$. We also perform the change of measure $p(F) \leq M\mu(F)$.
\end{proof}

Specializing to the setting of the paper, we obtain the ``Exceedance Error Lemma.'' Below, for any set $H \subset \mathcal X$, we should interpret $\ppre_t(H)$ conditional on $x_{t + 1}$, i.e. as $\ppre_t(H | x_{t + 1})$. Notably, under a joint baseline condition, we do not assume that this probability is bounded when $H = F$. 

\begin{lemma}[Exceedance Error Lemma]\label{lem:excedence error}
Let $x_{t + 1} \sim q_{t + 1}$. Consider the target distribution $\hat{p}(x_t | x_{t + 1}) \propto e^{\hat{v}_t(x-t)}\ppre_t(x_t | x_{t + 1})$. Consider also the sampling distribution $q_t(x_t | x_{t + 1})$ induced by accepting proposals from $\ppre_t(x_t | x_{t + 1})$ with probability $\min\{1, e^{\hat{v}_t(x_t) - B_{t + 1}(x_{t + 1})}\}$. Let $F(x_{t + 1})$ be the ``exceedance region'' 
\[F(x_{t + 1}) = \{x_t \ : \ \hat v_t(x_t) - {B}_{t + 1}(x_{t + 1}) > 0\} \]
Then 
\[\E_{x_{t + 1} \sim q_{t + 1}} \left[d_{TV}(q_t(\cdot | x_{t +1}), \hat{p}_t(\cdot | x_{t + 1})\right] \leq \E_{x_{t + 1} \sim q_{t + 1}}[ A_{x_{t + 1}}] + \E_{x_{t + 1\sim  q_{t + 1}}}[{B}_{x_{t + 1}}]\]
where
\[{A}_{x_{t + 1}} := \frac{\ppre_t(F)}{\E_{\ppre_t}[1\{E\}e^{\hat v_t(x_t) -  B_{t + 1}(x_{t + 1})}]}\]
\[{B}_{x_{t + 1}} := \frac{\E_{x_t \sim \ppre_t} [1\{F\} e^{\hat v_t(x_t)}]}{\E_{x_t \sim \ppre_t }[e^{\hat v_t(x_t)}]}\]
\end{lemma}
\begin{proof}
Fix $t$ and condition on $x_{t + 1}$. In Lemma \ref{lem:generic error approx sampling}, take $\tilde p \gets e^{\hat{v}_t(x_t)} \ppre(x_t | x_{t + 1})$, $\mu \gets \ppre_t(x_t | x_{t + 1})$, $B \gets e^{{B}_{t + 1}(x_{t + 1})}$ and let $q_t(x_t | x_{t + 1})$ be induced by rejection sampling with acceptance probability $\min \{e^{\hat{v}_t(x_t) - {B}_{t + 1}(x_{t + 1})}, 1\}$. We let $A_{x_{t + 1}}$ and $B_{x_{t + 1}}$ be the first and second terms of \eqref{eq:cond baseline tv} respectively.

Using the fact that $Z_p = \E_{\ppre_t}[e^{\hat{v}_t(x_t)}]$, we have 

\[A_{x_{t + 1}} =  \frac{\ppre_t(F)}{\E[e^{\hat{v}_t(x_t)}]\E[1\{E\}\frac{e^{\hat{v}_t(x_t)}}{\E[e^{\hat{v}_t(x_t)}]}]e^{-{B}_{t + 1}(x_{t + 1})}}\]
\[= \frac{\ppre(F)}{\E[1\{E\}e^{\hat{v}_t(x_t) - {B}_{t + 1}(x_{t + 1})}]}\]

all conditional on $x_{t + 1}$. Also:

\[B_{x_{t + 1}} = \E[\frac{e^{\hat v_t(x_t)}}{\E[e^{\hat v_t(x_t)}]}1\{F\}]\]

We can now take an outer expectation wrt $x_{t + 1} \sim q_{t + 1}(\cdot)$.
\end{proof}

As a corollary, we immediately have (relatively weak) TV bounds for fully arbitrary conditional and joint baselines. 

\begin{theorem}\label{thm:general baseline tv bound}
Under the conditions of Lemma \ref{lem:excedence error}, and under the further assumption that $B_{t + 1}$ is a $\delta$-conditional Baseline for $\hat{v}_t$ with respect to $\ppre_t$, we have:
\[\E_{x_{t + 1}} \left[d_{TV}\left(\hat{q}_t(x_t | x_{t + 1}), \hat{p}_t(x_t | x_{t + 1})\right)\right] \leq e^{2B}\frac{\delta}{1 - \delta} + e^{2B} \delta\]

If $B_{t + 1}$ is a $\delta$-joint Baseline for $\hat{v}_t$ with respect to the pair $(\hat{q}_{t + 1}, \ppre_t)$, then 

\[\E_{x_{t + 1}} \left[d_{TV}\left(\hat{q}_t(x_t | x_{t + 1}), \hat{p}_t(x_t | x_{t + 1})\right)\right] \leq e^{6B}\delta + e^{2B}\delta\]
\end{theorem}
\begin{proof}
The conditional baseline claim follows from \eqref{eq:cond baseline tv}, paired with $-B \leq \hat{v}_t(x_t), \hat{B}_{t + 1} \leq B$, together with $\ppre_t(F) \leq \delta$ and $\ppre_t(E) \geq 1 - \delta$ $\hat{q}_{t + 1}$-almost-surely.

The joint baseline claim follows from further taking $M = e^{2B}$ in \eqref{eq:joint baseline tv}
\end{proof}

The issue with the above Theorem is that it requires that we take $\delta \asymp e^{-c_0B}\delta_0$ in order to achieve TV-error of $\delta_0$. In general, such a $\delta$-Baseline would be so conservative that there would be no advantage over exact rejection sampling. This motivates fine-grained TV bounds.

\subsubsection{Fine-grained TV}

Until now, $B_{t + 1}$ has been a completely generic baseline function. As mentioned in the commentary beneath Lemma \ref{lem:tv mgf body}, we will shortly introduce the assumption that $B_{t + 1}$ is parameterized as $B_{t + 1}(x_{t + 1}) = b_{t + 1}(x_{t + 1}) + \tau_{t + 1}$, and that $B_{t + 1}$ is \emph{Chernoff-certified}, in the sense that running a Chernoff argument jointly over $x_t$ and $x_{t + 1}$ with some exponent $\lambda \geq 2$ is sufficient to prove that $B_{t + 1}$ is a baseline wrt $(q_{t + 1}, \ppre_t)$ at level $\delta$:

\[e^{-\lambda \tau_{t + 1}}\E_{t, t + 1}[e^{\lambda s_{t + 1}}] \leq \delta\]

where we define $s_{t + 1} := \hat{v}_t(x_t) - {b}_{t + 1}(x_{t + 1})$ and $\E_{t,t+1}[\cdot]$ denotes expectation under the relevant joint law of $(x_t,x_{t+1})$. For the analytic baselines of Section \ref{sec:mog theory}, such a condition can be verified directly. 

\textbf{However}, we also want our proofs to apply to \textbf{learned} baselines, which may only satisfy the Chernoff certificate assumption \textbf{with high probability over their training data}. In the setting of Section \ref{sec:lcb theory}, recall that $(\hat{\lambda}, \hat b)$ are chosen from random training data, and then we set $B_{t + 1}(x_{t + 1}) := \hat b(x_{t + 1}) + \hat{\tau}_{\hat{\lambda}, \hat{b}} + \epsilon_0$. For whichever $(\hat{\lambda}, \hat{b})$ are chosen by the optimizer, $\hat{\tau}_{\hat{\lambda}, \hat b} + \epsilon_0$ is an upper bound of $\tau_{\hat \lambda, \hat{b}}$ with probability $1 - \eta$ over the training data, where $\tau_{\hat{\lambda}, \hat{b}}$ is the slack that makes the Chernoff certificate hold with exponent $\hat{\lambda}$. Therefore, in the case of data-driven baselines, the Chernoff certificate holds \textbf{on a good event}, $\mathcal G$, with probability $1 - \eta$ over the independent probability space defined by the training process. This \emph{training} probability space is independent to the probability space defined by the \emph{inference} process.

All cases of interest are subsumed by allowing the Chernoff certificate to hold on an event,  $\mathcal{G}$, from an independent probability space, since we can take $\mathcal{G}$ to have measure 1 in the case of analytic baselines. Due to this (essentially optional) dependence on $\mathcal{G}$, the TV:MGF Lemma uses the notation $\hat b_{t + 1}$ and $\hat \tau_{t + 1}$, though it can equally well be read for fixed $b_{t + 1}$ and $\tau_{t + 1}$, where $\epsilon_0 = 0$ typically.

In the following, we define $\hat{s}_{t + 1} := \hat{v}_t(x_t) - \hat{b}_{t + 1} (x_{t + 1})$ and $M(\lambda) := \E_{x_{t + 1}}\E_{x_t}[e^{\lambda \hat{s}_{t + 1}}]$
where $\lambda \in \R$ and the distributions of $x_{t + 1}$ and $x_t$ are those of the Exceedance Error Lemma. 

In order to maintain notational convenience, we nest expectations with the understanding that the outer expectation is over $x_{t + 1} \sim q_{t + 1}$ and the inner expectation is with respect to $x_t \sim \ppre_t(\cdot | x_{t + 1})$

\begin{lemma}[TV:MGF Lemma]\label{lem:tv mgf apdx}
Fix $\lambda \geq 1$. Consider the setting of Lemma \ref{lem:excedence error}, and suppose additionally that $\hat B_{t + 1}(x_{t + 1}) = \hat{b}_{t + 1}(x_{t + 1}) + \hat{\tau}_{t + 1} + \epsilon_0$, where $\hat{b}_{t + 1}$ and $\hat \tau_{t + 1}$ are  quantities depending on external randomness, $Z$\footnote{wrt the training process}. Let $\tau_{t + 1}$ be the constant that fixes
\[e^{-\lambda  \tau_{t + 1}}\E_{x_t, x_{t + 1}}[e^{\lambda \hat s}] = \delta\]

On a good event $\mathcal{G}$ of $Z$, assume that $\hat{\tau}_{t + 1} \in \tau_{t + 1} \pm \epsilon_0$.


Assume that $\ppre(F | x_{t + 1}) \leq c$ a.s. Then, on $\mathcal{G}$, we have:


\[\E[A_{x_{t + 1}}] \leq \frac{\delta^{1 - 2/\lambda}}{(1 - c)^2}\left ({M(-\lambda)}{M(\lambda)}\right)^{\frac{1}{\lambda}}e^{2\epsilon_0}\]

If $\lambda \geq 2$, we may drop the $e^{2\epsilon_0}$ term. 

For any $\omega > 2$ (we do not necessarily need $\omega = \lambda$):

\[\E[B_{x_{t + 1}}] \leq \delta^{1 - \frac{2}{\omega
}}(M(-\omega)\left(M(\omega)\right)^{1/\omega}\]



\end{lemma}

\begin{proof}

Define $\tilde \tau_{t + 1} := \hat \tau_{t + 1} + \epsilon_0$ to be the effective slack.

On $\mathcal{G}$, $\tilde \tau_{t + 1} \geq \tau_{t + 1}$. Therefore, 

\[e^{-\lambda\tilde \tau_{t + 1}} \E_{x_t, x_{t + 1}}[e^{\lambda \hat s}] \leq \delta\]

Moreover, on $\mathcal{G}$, we have $\tilde\tau_{t + 1} \leq \tau_{t + 1} + 2\epsilon_0$

We now bound terms from the Exceedance Error Lemma. 

\textbf{The B part}

\begin{align*}
\E[B_{x_{t + 1}}] &= \E\left[\frac{\E[1\{F\}e^{\hat{v}_t(x_t)}]}{\E[e^{\hat{v}_t(x_t)}]}\right]\\ 
&= \E\left[\frac{\E[1\{F\}e^{\hat{v}_t(x_t) - \hat{b}_{t + 1}(x_{t + 1})}]}{\E[e^{\hat{v}_t(x_t) - \hat{b}_{t + 1}(x_{t + 1})}]}\right]  \tag{Multiply by $\frac{e^{-\hat{b}_{t + 1}(x_{t + 1})}}{e^{-\hat{b}_{t + 1}(x_{t + 1})}} = 1$ inside the outer expectation} \\
&= \E\left[\frac{\E[1\{F\}e^{\hat s)}]}{\E[e^{\hat s}]}\right] \\
&\leq \E\left[\E[1\{F\}e^{\hat s}] \E[e^{-\hat s}]\right]  \tag{Jensen's with $x \mapsto 1/x$}\end{align*}

For any three positive RVs $X, Y, Z > 0$ and any three Holder-compatible exponents, $\alpha, \beta, \gamma \geq 1$ with $\frac{1}{\alpha} + \frac{1}{\beta} + \frac{1}{\gamma} = 1$, we have that $\E[\E[XY]\E[Z]] \leq \E[X^\alpha]^{1/\alpha}\E[Y^\beta]^{1/\beta}\E[Z^\gamma]^{1/\gamma}$, where on the RHS the expectations are taken jointly over the two expectations on the left. Applying this fact with $X \gets 1\{F\}$, $Y \gets e^{\hat{s}}$, $Z \gets e^{-\hat{s}}$, we have

\[\E[B_{x_{t + 1}}] \leq \E[\E[1\{F\}^{\alpha}]]^{1/\alpha}\E[\E[e^{\beta \hat{s}}]]^{1/\beta}\E[\E[e^{-\gamma \hat{s}}]]^{1/\gamma}\]

Choose $\beta = \gamma = \omega > 2$. Then we're left with $\alpha^{-1} = 1 - \frac{2}{\omega}$. Then we get:

\[\E[B_{x_{t + 1}}] \leq \delta^{1 - 2/\omega}(M(\omega)M(-\omega))^{1/\omega}\]

\textbf{The A part}

\textbf{Case $\lambda \in [1, 2]$}

 Begin by applying Holder's inequality with dual norm pair $(\omega, \lambda)$ with $\omega = \frac{\lambda}{\lambda - 1}$ (recall $\lambda \geq 1)$:

\[\E[A_{x_{t + 1}}] \leq \E_{x_{t + 1}}[\E_{x_t}[1\{F\}]^{\omega}]^{\frac{1}{\omega}}\E[(\frac{1}{e^{-\tilde {\tau}_{t + 1}}\E[1\{E\}e^{\hat{s}_{t + 1}}]})^{\lambda}]^{\frac{1}{\lambda}}\]

The first term is bounded by $\delta^{1/\omega} = \delta^{1 - \frac{1}{\lambda}}$.  We can pull $e^{-\tilde{\tau}_{t + 1}}$ out of both expectations to get the simplified expression

\begin{align*}&\delta^{1 - \frac{1}{\lambda}}e^{\tilde{\tau}_{t + 1}}\E\left[\left(\frac{1}{\E[1\{E\}e^{\hat{s}_{t + 1}}]}\right)^{\lambda}\right]^{\frac{1}{\lambda}} \\
&= \delta^{1 - \frac{1}{\lambda}}e^{\tilde{\tau}_{t + 1}}\E\left[\left(\frac{1}{\ppre(E)\E[e^{\hat s_{t + 1}} | E]}\right)^{\lambda}\right]^{\frac{1}{\lambda}}\\
&\leq \delta^{1 - \frac{1}{\lambda}}e^{\tilde{\tau}_{t + 1}}\E\left[\left(\frac{\E[e^{-\hat{s}_{t + 1}} | E]  }{\ppre(E)}\right)^\lambda\right]^{\frac{1}{\lambda}}  \tag{Jensen's inequality applied fo $\frac{1}{z}$, $z > 0$} \\
&\leq \delta^{1 - \frac{1}{\lambda}}e^{\tilde{\tau}_{t + 1}}\E\left[\left(\frac{\E[1\{E\}e^{-\lambda\hat{s}_{t + 1}}]}{\ppre(E)^{2\lambda}}\right)\right]^{\frac{1}{\lambda}} \tag{Jensen's inequality applied to $z^{\lambda}$; $\lambda \geq 1$} \\ 
&\leq \delta^{1 - \frac{1}{\lambda}}e^{\tilde{\tau}_{t + 1}}\E\left[\left(\frac{\E[e^{-\lambda\hat{s}_{t + 1}}]}{(1 - c)^{2\lambda}}\right)\right]^{\frac{1}{\lambda}} \\ 
&= \frac{\delta^{1 - \frac{1}{\lambda}}e^{\tilde{\tau}_{t + 1}}}{(1 - c)^2}M(-\lambda)^{\frac{1}{\lambda}}.\end{align*}

We can now recall that $e^{\tilde{\tau}_{t + 1}} \leq e^{\tau_{t + 1} + 2\epsilon_0} = e^{2\epsilon_0}\frac{\delta^{-1/\lambda}}{M(\lambda)^{-1/\lambda}}$ to get 

\[\E[A_{x_{t + 1}}] \leq \frac{\delta^{1 - 2/\lambda}}{(1 - c)^2}\left ({M(-\lambda)}{M(\lambda)}\right)^{\frac{1}{\lambda}}e^{2\epsilon_0}.\]

\textbf{Case $\lambda \geq 2$}

Starting from $A_{x_{t + 1}}$, and replacing $\ppre(F) = \E[1\{F\}]$ (inner expectation only), and pulling the $e^{\tilde{\tau}_{t + 1}}$ out of $e^{\hat{B}_{t + 1}}$, we have:

\begin{align*}\E[A_{x_{t + 1}}] &= \E\left[e^{\tilde\tau_{t + 1}}\frac{\E[1\{F\}]}{\E[1\{E\}e^{\hat{s}_{t + 1}}]}\right] \\ 
&\leq \E\left[e^{\tilde\tau_{t + 1}}\frac{\E[e^{(\lambda - 1)(\hat{s}_{t + 1}- \tilde{\tau}_{t+ 1})}]}{\E[1\{E\}e^{\hat{s}_{t + 1}}]}\right].\end{align*}

Where we have used $1\{\hat{s}_{t + 1} \geq \tilde \tau_{t + 1}\} \leq e^{(\lambda - 1)(\hat{s}_{t + 1} - \tilde\tau_{t + 1})}$ for $\lambda \geq 2> 1$. Continuing, we have that the above is equal to

\begin{align*}&= e^{-(\lambda - 2)\tilde{\tau}_{t + 1}}\E\left[\frac{\E[e^{(\lambda - 1)\hat{s}_{t + 1}}]}{\E[1\{E\}e^{\hat{s}_{t + 1}}]}\right] \\
&= e^{-(\lambda - 2)\tilde{\tau}_{t + 1}}\E\left[\frac{\E[e^{(\lambda - 1)\hat{s}_{t + 1}}]}{\E[e^{\hat{s}_{t + 1}} | E]\ppre(E)}\right] \\
&\leq e^{-(\lambda - 2)\tilde{\tau}_{t + 1}}\E\left[\frac{\E[e^{(\lambda - 1)\hat{s}_{t + 1}}]\E[e^{-\hat{s}_{t + 1}} | E]}{\ppre(E)}\right]  \tag{Jensen's inequality applied to $1/z$, $z > 0$} \\ 
&= e^{-(\lambda - 2)\tilde{\tau}_{t + 1}}\E\left[\frac{\E[e^{(\lambda - 1)\hat{s}_{t + 1}}]\E\left[e^{-\hat{s}_{t + 1}} 1\{E\}\right]}{\ppre(E)^2}\right] \tag{since $\E[\cdot | E]) = \E\left[\cdot \times 1\{E\}\right]/\Pr[E]$} \\
&\leq e^{-(\lambda - 2)\tilde{\tau}_{t + 1}}\E\left[\frac{\E\left[e^{(\lambda - 1)\hat{s}_{t + 1}}\right]\E\left[e^{-\hat{s}_{t + 1}}\right]}{\ppre(E)^2}\right] \tag{since $e^{-\hat{s}_{t + 1}} > 0$} \\
&\leq \underbrace{\frac{e^{-(\lambda - 2)\tilde{\tau}_{t + 1}}}{(1 - c)^2}}_{(*)}\underbrace{\E\left[\E\left[e^{(\lambda - 1)\hat{s}_{t + 1}}\right]\E\left[e^{-\hat{s}_{t + 1}}\right]\right]}_{(**)}.\end{align*}

Using Holder's inequality to pull the outer expectation (over $x_{t + 1}$) apart with any dual exponents $p, q \in [1, \infty]$ ($\frac{1}{p} + \frac{1}{q} = 1$):

\[(**)\leq \E\left[\E[e^{(\lambda - 1)\hat{s}_{t + 1}}]^p\right]^{1/p}\E\left[\E[e^{-\hat{s}_{t + 1}}]^{q}\right]^{1/q}. \]

Since $z^p$ and $z^q$ are convex functions for $z > 0$, and $\E[e^{(\lambda - 1)\hat{s}_{t + 1}}], \E[e^{-\hat{s}_{t + 1}}] > 0$, we can apply Jensen's inequality to arrive at:

\[(**) \leq \E\left[\E[e^{p(\lambda - 1)\hat{s}_{t + 1}}]\right]^{1/p}\E\left[\E[e^{-q\hat{s}_{t + 1}}]\right]^{1/q} = M\left(p(\lambda - 1)\right)^{1/p}M(-q)^{1/q}.\]

Under the choice of dual exponents $p := \frac{\lambda}{\lambda - 1}$ and $q := \lambda$, we have 

\[(**) \leq M(\lambda)^{\frac{\lambda - 1}{\lambda}}M(-\lambda)^{\frac{1}{\lambda}}.\]

Moreover, 
\[e^{-(\lambda - 2)\tilde{\tau}_{t + 1}} = (e^{-\lambda \tilde{\tau}_{t + 1}})^{\frac{\lambda - 2}{\lambda}},\]

and as $z \mapsto z^{(\lambda - 2)/\lambda}$ is an increasing function (since $\lambda \geq 2$), we may use $e^{-\lambda \tilde{\tau}_{t + 1}} \leq \frac{\delta}{M(\lambda)}$ to conclude that 

\[(*) \leq \frac{1}{(1 - c)^2}\left(\frac{\delta}{M(\lambda)}\right)^{\frac{\lambda - 2}{\lambda}}.\]

Putting together the bounds on $(*)$ and $(**)$, we have 

\[\E[A_{x_{t + 1}}] \leq \frac{\delta ^{1 - \frac{2}{\lambda}}}{(1 - c)^2}\left( M(\lambda)M(-\lambda)\right)^{1/\lambda}.\]

\end{proof}

When combined with the Exceedance Error Lemma \ref{lem:excedence error}, Lemma \ref{lem:tv mgf apdx} yields the Lemma \ref{lem:tv mgf body} from the body of the paper. 

\subsection{Proofs for Section \ref{sec:properties of lcbs}}\label{apdx:lcb main thm proofs}

The results of Section \ref{sec:properties of lcbs} are corollaries of the previous results.

\begin{theorem}[Theorem \ref{thm:lcb tv bound}]
Let $\hat \lambda, \hat b$ be the ERMs of $\hat{J}$. Let $\hat B^{(\hat \lambda,\hat b)}_{t + 1} := \hat{b}_{t + 1}(x_{t + 1}) + \hat{\tau}_{\hat {\lambda}, \hat{b}_{t + 1}} + \epsilon_0$ be the associated LCB, and assume that it is also satisfies worst-case coverage at level $c \in (0, 1)$. Then, on the good training event from Proposition \ref{prop:lcb learning bounds} 
\[\textstyle \E_{x_{t + 1} \sim \hat{q}_{t + 1}}[d_{TV}(\hat{q}_{t}(\cdot | x_{t + 1}), \hat{p}(\cdot | x_{t + 1}))] \leq \frac{2\delta}{(1 - c)^2} e^{J^* + 2\epsilon_0}\]
\end{theorem}
\begin{proof}
We suppress time subscripts. 

By definition, the LCB only fits $\hat{\lambda} \geq 2$. We have also already assumed worst-case coverage at level $c$. Therefore, by the Proposition \ref{prop:lcb learning bounds}, the assumptions of the TV:MGF Lemma (Lemma \ref{lem:tv mgf apdx}) that $\hat{\tau}_{\hat{\lambda}, \hat{b}} \in \tau_{\hat{\lambda}, \hat{b}} \pm \epsilon_0$ is satisfied. 

The TV:MGF Lemma hence gives the bound

\[\E_{x_{t + 1}} \left[ d_{TV}\left(\hat{q}(x_t | x_{t + 1}), \hat{p}(x_t | x_{t + 1})\right) \right] \leq \frac{2\delta}{(1 - c)^2}e^{J(\hat{\lambda}, \hat{b})}\]

Moreover, the right-hand side is bounded using the other part of Proposition \ref{prop:lcb learning bounds}, which asserts that $J(\hat{\lambda}, \hat{b}) \leq J^* + 2\epsilon_0$, where $\hat{\lambda}$ and $\hat{b}$ are jointly ERMs. 

\[\frac{2\delta}{(1 - c)^2}e^{J(\hat{\lambda}, \hat{b})} \leq \frac{2\delta}{(1 - c)^2}e^{J^* + 2\epsilon_0}\]

Combining the inequalities completes the proof. 

\end{proof}

We now prove Corollary \ref{cor:subgaussian value corollary}'s bound on TV distance for conditionally sub-Gaussian values. The primary case is when $\delta$ is small relative to the intrinsic variance of the value, which gives the $\tilde O(\delta)$ bound quoted in Corollary \ref{cor:subgaussian val corollary apdx}. However, we additionally derive a $\tilde O(\sqrt{\delta})$ bound that holds for arbitrarily large $\delta$. 
\begin{corollary}[Corollary \ref{cor:subgaussian value corollary}]\label{cor:subgaussian val corollary apdx}
Let $\hat B^{(\hat \lambda,\hat b)}_{t + 1} := \hat{b}_{t + 1}(x_{t + 1}) + \hat{\tau}_{\hat {\lambda}, \hat{b}_{t + 1}}+ \epsilon_0$ be the associated LCB, and assume that it is also satisfies worst-case coverage at level $c \in (0, 1)$. Suppose there exists a constant $\sigma_t^2$ so that $\hat{v}_t(x_t)$ is $\sigma_t^2$ conditionally subgaussian given $x_{t + 1}$, for each $x_{t + 1}$. Assume that $\Lambda \geq \max\{4, \frac{\sqrt{2\log{\frac{1}{\delta}}}}{\sigma_t}\}$. Further assume that $\E[\hat{v}_t(x_t) | x_{t + 1}] \in H$. Assume the good training event from Proposition \ref{prop:lcb learning bounds}. Then, if $\delta \leq e^{-2\sigma_t^2}$:
\begin{align*}
\E_{x_{t + 1} \sim \hat{q}_{t + 1}}&[d_{TV}(\hat{q}_{t}(\cdot | x_{t + 1}), \hat{p}(\cdot | x_{t + 1}))] \leq \textstyle \frac{2\delta}{(1 - c)^2} e^{2\sigma_t\sqrt{2\log\frac{1}{\delta}} + 2\epsilon_0}.
\end{align*}

If $\delta > e^{-2\sigma_t^2}$, then 
\[\E_{x_{t + 1} \sim \hat{q}_{t + 1}}[d_{TV}(\hat{q}_{t}(\cdot | x_{t + 1}), \hat{p}(\cdot | x_{t + 1}))] 
\leq \frac{2\sqrt{\delta}}{(1 - c)^2}e^{4\sigma^2_t + 2\epsilon_0}\]

\end{corollary}
\begin{proof}
Again, we suppress timestep subscripts in our notation. Picking up from the previous theorem, it suffices to bound $J^*$ in this special case. Recall 
\[J(\lambda, b) = \frac{1}{\lambda}\log \E_{x_{t+1}}\E_{x_t}[e^{\lambda(\hat v(x_t)-b(x_{t+1}))}] + \frac{1}{\lambda}\log \E_{x_{t+1}}\E_{x_t}[e^{-\lambda(\hat v(x_t)-b(x_{t+1}))}] + \frac{2}{\lambda}\log\frac{1}{\delta}\]

By the subgaussianity assumption, for any fixed $x_{t + 1}$ and $\lambda$, when we let $b^*(x_{t + 1}) = \E[\hat{v}(x_t) |x_{t + 1}]$

\[\E_{x_t}[e^{\lambda (\hat{v}(x_t) - b(x_{t + 1})}] \leq e^{\lambda^2\sigma_t^2/2}\]
\[\E_{x_t}[e^{-\lambda (\hat{v}(x_t) - b(x_{t + 1})}] \leq e^{\lambda^2\sigma_t^2/2}\]

So that, once we take the outer expectation over $x_{t + 1}$, and logs

\[J(\lambda, b^*) \leq \lambda\sigma_t^2 + \frac{2\log{\frac{1}{\delta}}}{\lambda}\]

The minimizing value of $\lambda$ subject to the constraint $\lambda \geq 2$ is $\lambda^* := \max\{\frac{\sqrt{2\log\frac{1}{\delta}}}{\sigma_t}, 2\} $. Notice $\lambda^*$ is feasible since $\lambda^* \leq \Lambda$. Using $\lambda = \frac{\sqrt{2\log\frac{1}{\delta}}}{\sigma_t}$ as a witness if it is greater than 2, and $\lambda = 4$ as a witness otherwise gives:

\[J(\lambda^*, b^*) \leq \begin{cases}
2\sigma_t\sqrt{2\log\frac{1}{\delta}} \quad \mathrm{ if  } \ \ \  \delta \leq e^{-2\sigma_t^2} \\ 
4\sigma_t^2 + \log\sqrt\frac{1}{\delta} \quad \mathrm{otherwise }
\end{cases}\]

Therefore, when $\delta \leq e^{-2\sigma_t^2}$, we have the bound stated in Corollary \ref{cor:subgaussian value corollary} from Theorem \ref{thm:lcb tv bound}:
\begin{align*}
\E_{x_{t + 1} \sim \hat{q}_{t + 1}}[d_{TV}(\hat{q}_{t}(\cdot | x_{t + 1}), \hat{p}(\cdot | x_{t + 1}))] 
\leq 2\frac{\delta}{(1 - c)^2} e^{2\sigma_t\sqrt{2\log\frac{1}{\delta}} + 2\epsilon_0}.
\end{align*}

For large $\delta > e^{-2\sigma^2}$, we have the bound 

\[\E_{x_{t + 1} \sim \hat{q}_{t + 1}}[d_{TV}(\hat{q}_{t}(\cdot | x_{t + 1}), \hat{p}(\cdot | x_{t + 1}))] 
\leq \frac{2\sqrt{\delta}}{(1 - c)^2}e^{4\sigma^2_t + 2\epsilon_0}\]

\end{proof}

\section{Mixture of Gaussian Proofs}\label{apdx:MoG}

Below, we prove a result where the mixture does not have uniform covariance: $p_m \sim \sum_k \pi_k N(m^k, \Sigma^k)$, though we only use the result for $\Sigma^k = \Sigma$

\begin{proposition}\label{prop:apdx mog dist facts}
Fix a timestep $t \in [T]$. Under the above model, fix a component $k \in [K]$. Then $x_0 | x_t, k \sim N(\mu^k_{0 | t}(x_t),\Sigma^k_{0 | t})$
with 
\[\mu^k_{0 | t}(x_t) = A^k_tx_t + b^k_t\]
where analytical expressions for $\Sigma^k_{0 | t}$, $A^k_t$ and $b^k_t$ are boxed in the proof. 

Furthermore
\[\gamma^k_t(x_t) := \ppre(k | x_t) = \frac{\pi_kN(m^k_t, \Sigma^k_t)}{\sum_j \pi_j N(m^j_t, \Sigma^j_t)}(x_t)\]
where analytic expressions for $m^k_t$ and $\Sigma^k_t$ are boxed in the proof.

As a consequence, we may express
\[\ppre(x_0 | x_t) = \sum_{k = 1}^K \gamma^k_t(x_t)N(\mu^k_{0 | t}(x_t), \Sigma^k_{0 | t})\]
\end{proposition}
\begin{proof}
Suppose the target is 
\[p_m(x) = \sum_{k = 1}^M \pi_k N(m^k, \Sigma^k)\]

In the forward process, we have the Markov chain $k \rightarrow x_0 \rightarrow x_t$, meaning that

\[x_t | x_0, k = \sqrt{\bar\alpha_t}x_0 + N(0, \sigma_t^2I)\] with $\sigma_t^2 = 1 - \bar\alpha_t$.

By Gaussian conditioning laws, we have that the joint conditioned on $k$ is

\[[x_0, x_t] | k \sim N([m^k, m^k_t], \begin{pmatrix}
    \Sigma^k & \sqrt{\bar\alpha_t}\Sigma^k \\ \sqrt {\bar\alpha_t}\Sigma^k  & \Sigma^k_t
\end{pmatrix})\]
with \fbox{$m^k_t = \sqrt{\bar\alpha_t} m^k$}, and  \fbox{$\Sigma^k_t = \bar\alpha_t\Sigma^k + \sigma_t^2I$}.

and so we have by conditioning

\[x_0 | x_t, k \sim N(\mu^k_{0 | t}(x_t) , \Sigma^k_{0 | t})\]
with  \fbox{$\Sigma^k_{0 | t} = \Sigma^k - {\bar\alpha_t}\Sigma^k(\Sigma^k_{t})^{-1}\Sigma^k$} and\fbox{$\mu^k_{0 | t}(x_t) = m^k + \sqrt{\bar\alpha_t}\Sigma^k(\Sigma^{k}_t)^{-1}(x_t - \sqrt{\bar\alpha_t}m^k)$}

which we may re-express as \[\mu^k_{0 | t} =  A^t_kx_t + b^k_t\]

setting \fbox{$A^k_t = \sqrt{\bar\alpha_t}\Sigma^k (\Sigma^k_t)^{-1}$} and\fbox{$b^t_k = m^k - \sqrt{\bar\alpha_t}A^k_t m^k$}

To calculate responsibilities $p_m(k | x_t)$, we apply Bayes rule ($p_m(x_t | k)$ can be read off from the display for $p_m(x_t | x_0, k)$)
\[\gamma^k_t(x_t) := p(k | x_t) = \frac{\pi_kN(m^k_t, \Sigma^k_t)}{\sum_j \pi_j N(m^j_t, \Sigma^j_t)}(x_t)\].

This finally allows us to express
\[p(x_0 | x_t) = \sum_{k = 1}^M \gamma^k_t(x_t)N(\mu^k_{0 | t}(x_t), \Sigma^k_{0 | t})\]

\end{proof}

We now specialize to the case where $\Sigma^k = \Sigma \ \ \ \forall k$, meaning we drop the $k$ superscripts from the notation of the previous proposition wherever it makes no difference. Notice that $\Sigma^k_{0 | t}$ and $A^k_t$ from the previous proposition then become constant over $k$, while $\mu^k_{0 | t}(x_t)$ still depends on $k$. To reflect this, we notate $\Sigma_{0 | t} := \Sigma^k_{0 | t} = \mathrm {Cov}(x_0 | x_t, k)$ and $\Sigma_t := \Sigma^k_t = \mathrm {Cov}(x_t | k)$. Similarly, $\mathrm {Cov}(x_t | x_{t + 1}, k)$ is constant in $k$, and is therefore notated as $\Sigma_{t | t + 1}$.

Recall the definition of a conditional baseline, Definition \ref{def:conditional baseline}. Also, define $\Sigma_{t | t + 1} := \mathrm {Cov}(x_t | x_{t + 1}, k)$ to be the class-conditional backwards variance, which does not depend on $k$.

\begin{proposition}\label{prop:mog oracle baseline apdx}[Conditional sub-Gaussianity and analytic baseline]
Consider $a^k_{t+1}(x_{t + 1}) = \E_{p_m}[v_t(x_t) | x_{t + 1}, k]$ and $a_{t+1}(x_{t + 1}) = \E[a^k_{t+1}(x_{t + 1}) | x_{t + 1}]$. Then, conditioned on $x_{t + 1}$, we have that $v_t(x_t) - a_{t + 1}(x_{t + 1})$ is subgaussian with parameter $\sigma_a^2(x_{t + 1}) := L_v^2 \tilde \beta_t + \Delta_{t+1}(x_{t+1})^2$, where $L_v$ is a Lipschitz parameter of $v_t$, $\tilde \beta_t = \|\Sigma_{t | t + 1}\|$, and $\Delta_{t+1} = \max_k |a^k_{t+1}(x_{t + 1}) - a_{t+1}(x_{t + 1})|$.

As a consequence, $a_{t}(x_{t + 1}) + \sigma_a(x_{t + 1})\sqrt{2\ \log{\frac{1}{\delta}}}$ is a conditional baseline for $v_t$ at level $\delta$.

Moreover, for all $x_{t + 1}$ simultaneously $\sigma_a(x_{t + 1}) = O(\beta_t + \sqrt{\tilde \beta_t})$, where $O(\cdot)$ hides $x_{t + 1}$-independent and dimension-free constants of $m^k$ and $\Sigma$.
\end{proposition}
\begin{proof}
\textbf{Lipschitzness of $v_t$}

We begin by establishing Lipschitzness of $v_t(x_t) = \log \E [e^{r(x_0)} | x_t]$, similar to the previous proposition. To do this, we write 

\[v_t(x_t) = \log \sum_{k} \gamma^k_t(x_t)\E[e^{r(x_0)} | x_t, k]  = \log \sum_k e^{g^k_t(x_t) + \log \gamma^k_t(x_t)}\]

where $g^k_t(x_t) = \log \E[e^{r(x_0)} | x_t, k]$. By the property of the softmax, it suffices to show Lipschitzness of $g^k_t$ and $\log \gamma^k_t$ individually. 

First we establish Lipschitzness of $\log \gamma^k_t(x_t)$ by bounding its gradient. Plugging in the expression for $\gamma^k_t$ from the previous proposition

\[\nabla \log \gamma^k_t(x) = \nabla_x \log \frac{\pi_k N(x; m^k_t, \Sigma_t)}{\sum_j \pi_j N(x; m^j_t, \Sigma_t)}\]
\[= \nabla_x \frac{-(x- m^k_t)\Sigma_t^{-1}(x - m^k_t)}{2} - \nabla_x \log\sum_j \pi_jN(x; m^j_t, \Sigma_t) \]
\[= \Sigma^{-1}_t(m^k_t - x) - \frac{\sum_j \pi_j \nabla_x N(x; m^j_t)}{\sum_j \pi_j N(x; m^j_t, \Sigma_t)}\]
\[= \Sigma^{-1}_t(m^k_t - x)- \frac{\sum_j \pi_j  N(x; m^j_t)\Sigma^{-1}_t(m^j_t - x)}{\sum_j \pi_j N(x; m^j_t, \Sigma_t)}\]
\[= \Sigma^{-1}_t(m^k_t - x) - \sum_j \gamma^j_t(x)\Sigma_t^{-1}(m^j_t - x)\]
\[= \Sigma_t^{-1}(m^k_t - \sum_j \gamma^j_t(x)m^j_t)\]

Taking norms on each side gives 

\[\|\nabla \log \gamma^k_t(x)\| \leq 2\|\Sigma_t^{-1}\|\max_j \|m^j_t\| \leq 2\|\Sigma_t^{-1}\|\max_j \|m^j\|,\]

where the final inequality is because $\|m^j_t\| \leq \|m^j\|$ for all $j$ and $t$.

We now derive Lispchitzness of $g^k_t$. Since $x_0 | x_t, k \sim N(\mu^k_{0 | t}(x_t), \Sigma_{0 | t})$, we have $x_0 = \mu^k_{0 | t}(x_t) + z$ where $z \sim N(0, \Sigma_{0 | t})$. Fixing a realization of $z$, we have for any $x_t$ and $y_t$

\[r(\mu^k_{0 | t}(x_t) + z) \leq r(\mu^k_{0 | t}(y_t) + z) + L_r\|\mu^k_{0 | t}(x_t)  - \mu^k_{0 | t}(y_t)\|, \]
so
\[e^{r(\mu^k_{0 | t}(x_t) + z)} \leq e^{r(\mu^k_{0 | t}(y_t) + z)} e^{L_r\|\mu^k_{0 | t}(x_t)  - \mu^k_{0 | t}(y_t)\|} ,\]
so
\[\E _z[e^{r(\mu^k_{0 | t}(x_t) + z)}] \leq \E_z[e^{r(\mu^k_{0 | t}(y_t) + z)}] e^{L_r\|\mu^k_{0 | t}(x_t)  - \mu^k_{0 | t}(y_t)\|}, \]
so
\[\log \E _z[e^{r(\mu^k_{0 | t}(x_t) + z)}] \leq \log \E_z[e^{r(\mu^k_{0 | t}(y_t) + z)}] + L_r\|\mu^k_{0 | t}(x_t)  - \mu^k_{0 | t}(y_t)\| ,\]
and by Lipschitzness of $\mu^k_{0 | t}(\cdot) = A_t(\cdot) + b^k_t $ (given in Proposition \ref{prop:apdx mog dist facts}):

\[\log \E _z[e^{r(\mu^k_{0 | t}(x_t) + z)}] \leq \log \E_z[e^{r(\mu^k_{0 | t}(y_t) + z)}] + L_r\|A_t\|\|x_t - y_t\|,\]

which is the same as 

\[g^k_t(x_t) \leq g^k_t(y_t) + L_r\|A_t\|\|x_t - y_t\|.\]

Since $x_t$ and $y_t$ are interchangeable, this completes the proof that $g^k_t$ is Lipschitz, and therefore shows that $v_t$ is Lipschitz with parameter $L_{v_t} \leq L_r\|A_t\| + 2\|\Sigma_t^{-1}\|\max_k\|m_k\|$.

\textbf{Subgaussianity from Lipschitz Concentration}

Since $a^k_{t+1}(x_{t + 1}) = \E[v_t(x_t) | x_{t + 1, k}]$, and $x_t | x_{t + 1}, k$ is Gaussian, we have by Gaussian concentration of Lipschitz functions
\[\E[e^{\lambda(v_t(x_t) - a^k_{t+1}(x_{t + 1}))} | x_{t + 1}, k] \leq e^{\lambda^2\tilde\beta_{t}L_{v_t}^2/2}\]

where the uniform covariance assumption ensures that $\|\mathrm{Cov}[x_t | x_{t + 1}, k]\| =  \tilde \beta_t$ is independent of $k$. We can therefore observe that

\begin{align*}\E[e^{\lambda(v_t(x_t) - a_{t+1}(x_{t + 1}))} | x_{t + 1}, k] 
&= \E[e^{\lambda(v_t(x_t) - a^k_{t+1}(x_{t + 1}) + a^k_t (x_{t + 1}) - a_{t+1}(x_{t + 1}))} | x_{t + 1}, k] \\
&= e^{\lambda(a_{t+1}(x_{t + 1}) - a^k_{t+1}(x_{t + 1}))}\E[e^{\lambda(v_t(x_t) - a^k_{t+1}(x_{t + 1}))} | x_{t + 1}, k] \\ 
&\leq e^{\lambda(a_{t+1}(x_{t + 1}) - a^k_{t+1}(x_{t + 1}))} e^{\lambda^2\tilde\beta_{t}L_{v_t}^2/2}.
\end{align*}


Thus 
\begin{align*}\E[e^{\lambda(v_t(x_t) - a_{t+1}(x_{t + 1}))} | x_{t + 1}] &\leq e^{\lambda^2\tilde\beta_{t }L_{v_t}^2/2}\sum_k \gamma^k_{t + 1}(x_{t + 1}) e^{\lambda (a_{t+1}(x_{t + 1}) - a^k_{t+1}(x_{t + 1}))} \\ 
&\leq e^{\lambda^2\tilde\beta_{t}L_{v_t}^2/2}\times e^{\lambda^2(2\Delta_{t+1}(x_{t + 1}))^2/8},\end{align*}

by applying Hoeffding's inequality to the RV that takes value $a_{t+1}(x_{t + 1}) - a^k_{t+1}(x_{t + 1}) \in [- \Delta_{t+1}, \Delta_{t+1}]$ with probability $\gamma^k_{t + 1}(x_{t + 1})$. 

This completes the proof that $v_t(x_t) - a_t(x_{t + 1})$ is subgaussian conditioned on $x_{t + 1}$ only, with parameter $\sigma_a^2(x_{t + 1}) = \tilde \beta_{t }L_{v_t}^2 + \Delta_{t+1}(x_{t + 1})^2$.

The definition of the conditional baseline follows from applying the corresponding subgaussian tail bound. 

\textbf{Uniform upper bound on $\sigma_a(x_{t + 1})$}

It remains to show that we have a uniform (in $x_{t + 1}$) upper bound on $\sigma^2_a = \tilde \beta_{t} L_{v_t}^2 + \Delta_{t+1}(x_{t+1})^2$. $L_{v_t}$ is already independent of $x_{t + 1}$, so it suffices to upper bound $\Delta_{t+1}(x_{t+1})$. For $k \neq l$, consider $x^k_t \sim p_m(\cdot | x_{t + 1}, k)$ and $x^l_t \sim p_m(\cdot | x_{t + 1}, l)$.  Consider $z \sim \mathcal{N}(0, \Sigma_{t | t + 1})$. Marginally (i.e. for $x^k_t$ and $x^l_t$ individually), we have $x^k_t = \mu_{t | t+1}^k(x_{t + 1}) + z$ and $x^l_t = \mu_{t | t+1}^l(x_{t + 1}) + z$. Thus $a^k_{t + 1}(x_{t + 1}) = \E[v_t(x_t^k)]$ and $a^l_{t + 1}(x_{t + 1}) = \E[v_t(x_t^l)]$. Therefore

\begin{align*}|a^k_{t+1}(x_{t + 1}) - a^l_{t+1}(x_{t + 1})| &= \left| \E_z[v_t(\mu_{t | t+1}^k(x_{t + 1}) + z)] - \E_z[v_t(\mu_{t | t+1}^l(x_{t + 1}) + z))]\right| \\
&= \left|\E_z\left[v_t(\mu_{t | t+1}^k(x_{t + 1}) + z) - v_t(\mu_{t | t+1}^l(x_{t + 1}) + z)\right]\right|.\end{align*}

By Lipschitzness of $v_t$, we have

\[\E_z[\left|v_t(x^k_t) - v_t(x^l_t)\right|] \leq L_{v_t} \E_z \left[\left\|\mu_{t | t+1}^k(x_{t + 1}) + z - \left(\mu_{t | t+1}^l(x_{t + 1}) + z\right)\right\|\right] = L_{v_t}\|\mu_{t | t+1}^k(x_{t + 1}) - \mu_{t | t+1}^l(x_{t + 1})\|\]

As in Proposition \ref{prop:apdx mog dist facts}, $\mu^k_{t | t+1}$ is an affine function, and only the bias is dependent on $k$ and $l$ (under our uniform covariance assumption): $\mu^k_{t | t + 1}(x_{t + 1}) = Ax_{t + 1} + b^k_{t + 1}$ with $A = \sqrt{\alpha_t}\Sigma_t\Sigma_{t + 1}^{-1}$ and $b^k_{t +1} = \beta_t \Sigma_{t + 1}^{-1} m^k_t$. Because $\|m^k_t\| \leq \max_k \|m^k\|$, we have 

\[\|\mu_{t | t+1}^k(x_{t + 1}) - \mu_{t | t+1}^l(x_{t + 1})\| \leq \beta_t \|\Sigma^{-1}_{t + 1} \|\|m^k_t - m^l_t\| \leq 2\beta_t \|\Sigma^{-1}_{t + 1} \|\max_k\|m^k\| \]

We can now conclude our argument

\begin{align*}\Delta_{t+1}(x_{t+1}) &= \max_k |a^k_{t + 1}(x_{t + 1}) - a_{t + 1}(x_{t + 1})|\\
&= \max_k |a^k_{t + 1}(x_{t + 1}) - \sum_l  \gamma^l_{t + 1}(x_{t + 1}) a^l_{t + 1}(x_{t + 1})| \\
&= \max_k |\sum_l \gamma^l_{t + 1}(x_{t + 1}) \left(a^k_{t + 1}(x_{t + 1}) - a^l_{t + 1}(x_{t + 1})\right) \\
&\leq 2L_{v_t}\beta_t \|\Sigma_{t+1}^{-1}\|\max_k \|m^k\|,\end{align*}

so that $\sigma^2_a(x_{t + 1}) = \tilde \beta_{t }L_{v_t}^2 + \Delta_{t+1}(x_{t + 1})^2 = O(\tilde \beta_t + \beta_t^2)$.
\end{proof}

\subsection{TV and Proposal Complexity for Analytic Baselines}

We will shortly use the theory we established in the previous section to give oracle bounds for LCBs. However, we pause to note that, in the special case of Gaussian mixtures, the analytic baselines have more favorable properties when used in Algorithm \ref{alg:baseline rejection sampling}. This is essentially due to the fact that they are conditional baselines, rather than merely joint baselines.

We now prove a TV bound for the sampling algorithm induced by plugging in $B_{t + 1}(x_{t + 1}) := a_{t + 1}(x_{t + 1}) + \sigma_a(x_{t + 1})\sqrt{2\log\frac{1}{\delta}}$ as a baseline in Algorithm \ref{alg:baseline rejection sampling}. Notice that we average over $\sigma_a(x_{t+1})$, rather than paying for the worst case $\sup_{x_{t+1}} \sigma_a(x_{t+1})$ when we use conditional baselines. 

\begin{proposition}\label{prop:apdx mixture baseline tv}[TV bound for baseline of Proposition \ref{prop:mog oracle baseline apdx}]
Consider the (conditional) baseline $B_{t + 1}(x_{t + 1}, \delta) := a(x_{t + 1}) + \iota(x_{t + 1})$ for ${v}_t$. Assume $\delta \leq \exp (-2\sup_{x_{t + 1}}\sigma_a(x_{t + 1})^2) \leq \exp(-2\sigma_a^2)$.

For any distribution $q_{t + 1}$ on $x_{t + 1}$, the sampling distribution $\hat{q}(\cdot | x_{t + 1})$ from timestep $t$ of Algorithm \ref{alg:baseline rejection sampling} satisfies:
\[\E_{x_{t + 1} \sim q_{t + 1}}[d_{TV}(\hat{q}(x_t | x_{t + 1}, \hat{p}(x_t | x_{t + 1})))]\]
\[ \leq \frac{\delta}{(1 - \delta)^2}\times\E_{x_{t + 1}}[e^{2\sigma_a(x_{t + 1})\sqrt{2\log\frac{1}{\delta}}}]\]
\end{proposition}
\begin{proof}
Fix $x_{t + 1}$, and take $q_{t + 1} = \delta_{x_{t+1}}$ to be the Dirac measure at $x_{t + 1}$ \footnote{The Dirac measure assigns measure $1$ to any set containing $x_{t + 1}$ and $0$ to all others}.  We now apply the TV:MGF Lemma (Lemma \ref{lem:tv mgf body}) to $b_{t + 1}(x_{t + 1}) = a_{t + 1}(x_{t + 1}) + \iota(x_{t + 1})$ and $\tau = 0$. We now verify that the Chernoff condition holds at $\lambda \gets \frac{\sqrt{2\log\frac{1}{\delta}}}{\sigma_a(x_{t +1})}$, which is bigger than $2$ under the hypothesized condition on $\delta$. 

\[\E_{x_{t + 1 \sim \delta_{x_{t + 1}}}}\E_{x_t \sim \ppre(\cdot | x_{t + 1})}[e^{\lambda(v_t(x_t) - a_{t + 1}(x_{t + 1}) - \iota(x_{t + 1})}] = e^{-\lambda \iota(x_{t + 1})}\E_{x_t}[e^{\lambda (v_t(x_t) - a_{t + 1}(x_{t + 1})}]\]
\[\leq e^{-\lambda \sigma_a(x_{t + 1})\sqrt{2\log\frac{1}{\delta}} + \lambda^2\sigma_a(x_{t + 1})^2/2}\]
\[= e^{-2\log\frac{1}{\delta} + \log\frac{1}{\delta}} = \delta\]

Now that we have verified the conditions for the TV:MGF Lemma, we can evaluate the bound it provides. Specifically, for the same choice of $\lambda$, we obtain  

\begin{align*}&\frac{2}{\lambda}\log \frac{1}{\delta } + \frac{1}{\lambda}\log M(\lambda) + \frac{1}{\lambda}\log M(-\lambda) \\
&\leq \frac{2}{\lambda}\log\frac{1}{\delta} + \frac{1}{\lambda}\log e^{-\lambda \iota (x_{t + 1})}e^{\lambda^2\sigma_a(x_{t + 1})^2/2} + \frac{1}{\lambda}\log e^{\lambda \iota (x_{t + 1})}e^{\lambda^2\sigma_a(x_{t + 1})^2/2}  \\
&= \frac{2}{\lambda}\log\frac{1}{\delta} + 2\lambda \sigma_a(x_{t + 1})^2  \\
&\leq  2\sigma_a(x_{t + 1})\sqrt{2\log\frac{1}{\delta}}\end{align*}

Now that we have shown the result for $q_{t + 1}(\cdot) = \delta_{x_{t + 1}}(\cdot)$, the arbitrary case follows from linearity of expectation. Specifically, by the TV:MGF Lemma, we have shown that, for any $x_{t + 1}$:

\[d_{TV}(\hat{q}(\cdot | x_{t + 1}), \hat{p}(\cdot | x_{t + 1})) \leq \frac{\delta}{(1 - \delta)^2}e^{2\sigma_a(x_{t +1})\sqrt{2\log\frac{1}{\delta}}}\]

and we may take expectations on either side with respect to an arbitrary $\hat{q}_{t + 1}(x_{t + 1})$
\end{proof}

\begin{remark}
Like Corollary \ref{cor:subgaussian value corollary} (Corollary \ref{cor:subgaussian val corollary apdx}), the previous result was proven by appealing to the TV:MGF Lemma. In this case, however, $\lambda$ is not freely minimized, but rather is chosen in order to reproduce the Chernoff certificate with $\tau = 0$. 
\end{remark}

The next result shows that the number of proposals required grows slower than $1/\delta^\epsilon$ for any $\epsilon> 0$.

\begin{theorem}[Proposal complexity of analytic baseline]
Consider the (conditional) baseline $B_{t + 1}(x_{t + 1}, \delta) := a(x_{t + 1}) + \iota(x_{t + 1})$ for ${v}_t$. Let $N$ be the number of proposals used to sample $x_t\sim \hat{q}_t(\cdot | x_{t + 1})$. Then, for any distribution on $x_{t + 1}$, and $x_t \sim \ppre_t(\cdot | x_{t + 1})$
\[\E_{x_t}[N | x_{t + 1}] \leq \frac{1}{(1 - \delta)^2}e^{\sigma_a(x_{t+1})\sqrt{2\log\frac{1}{\delta}} + \sigma_a(x_{t + 1})^2/2}\]
\end{theorem}
\begin{proof}
From Lemma \ref{lem:prop complexity generic apdx}, we have the bound 
\[\E[N | x_{t+1}] \leq \frac{1}{(1 - \delta)^2}\E[e^{B_{t+1}(x_{t+1}) - v_t(x_t)}]\]
\[= \frac{1}{(1 - \delta)^2}\E[e^{a_{t+1}(x_{t+1}) - v_t(x_t)}]e^{\sigma_a(x_{t+1})\sqrt{2\log\frac{1}{\delta}}}\]
\[\leq \frac{1}{(1 - \delta)^2}e^{\sigma_a(x_{t+1})^2/2}e^{\sigma_a(x_{t+1})\sqrt{2\log\frac{1}{\delta}}}\]
\[\leq \frac{1}{(1 - \delta)^2}e^{\sigma_a(x_{t+1})^2/2 +\sigma_a(x_{t+1})\sqrt{2\log\frac{1}{\delta}}}\]

where the second inequality follows from conditional subgaussianity of $v_t(x_t) - a_{t+1}(x_{t+1})$ (using the MGF characterization with exponent $1$).

\end{proof}
\subsection{Consequences for LCB}

Finally, we restate that using $a_{t + 1}$ as a witness for $b_{t + 1}$ in the LCB objective reproduces almost the same error. 

\begin{theorem}\label{thm:mog thm apdx}[Theorem \ref{thm:mog thm}]
Let $a_{t + 1}(x_{t + 1}) = \E_{p_m}[v_t(x_t) | x_{t + 1}]$. Then $v_t(x_t) - a_{t + 1}(x_{t + 1})$ is conditionally sub-gaussian conditioned on any $x_{t + 1}$, with parameter $\sigma_a^2 = O(\beta_t^2 + \tilde \beta_t)$, where $O(\cdot)$ hides dimension-independent factors of $\{m_k\}$ and $\Sigma$, and a full expression for $\sigma_a$ can be found in Appendix \ref{apdx:MoG}. 

Therefore, in the setting of Theorem \ref{cor:subgaussian value corollary}, LCB sampling achieves the bound 
\begin{align*}
\E_{x_{t + 1} \sim \hat{q}_{t + 1}}[d_{TV}(\hat{q}_{t}(\cdot | x_{t + 1}), \hat{p}(\cdot | x_{t + 1}))]
\leq  \frac{2\delta}{(1 - c)^2} e^{O(\sqrt{\tilde \beta_t} + \beta_t)\sqrt{\log\frac{1}{\delta}} + 2\epsilon_0}.
\end{align*}
\end{theorem}
\begin{proof}
 Follows from putting the uniform upper bound on $\sigma_a$ from Lemma \ref{prop:mog oracle baseline apdx} into Corollary \ref{cor:subgaussian value corollary}   
\end{proof}

\section{Proofs for LCB objective}

\subsection{Uniform Convergence (Proposition \ref{prop:lcb learning bounds})}

\begin{lemma}\label{lem:logsumexp uniform}
Let $f:\mathcal Z\to[-B,B]$ be fixed, let $\Lambda > 2$, and let
$H\subset\{\mathcal X\to[-B,B]\}$ be a hypothesis class.
Consider any distribution over pairs $(z,x)\in\mathcal Z\times\mathcal X$ and
$m$ i.i.d.\ samples $\{(z_i,x_i)\}_{i=1}^m$, from which we form $\hat{\E}$.
Then, with probability at least $1-\delta$ over the sample,
\begin{multline*}
\sup_{\lambda\in[1,\Lambda]\cup [-\Lambda, -1]}\ \sup_{b\in H}
\left|
\frac{1}{\lambda}\log \E\!\left[e^{\lambda (f(z)-b(x))}\right]
-
\frac{1}{\lambda}\log \hat{\E}\!\left[e^{\lambda (f(z)-b(x))}\right]
\right|
\\
\le
2e^{4\Lambda B}\,\hat{\mathcal R}_m(H)
+
2e^{4\Lambda B}\sqrt{\frac{\log\!\big(\frac{4\Lambda m}{\delta}\big)}{2m}}
+
\frac{8B}{m},
\end{multline*}
where
\[
\hat{\mathcal R}_m(H)
=\E_{\epsilon}\!\left[\sup_{b\in H}\frac{1}{m}\sum_{i=1}^m \epsilon_i b(x_i)\ \Big|\ x_1,\dots,x_m\right]
\]
is the (conditional) empirical Rademacher complexity of $H$.
\end{lemma}

\begin{proof}
Fix $\lambda\in[1,\Lambda] \cup [-\Lambda, -1]$. Consider the classes
\[
G=\{(z,x)\mapsto f(z)-b(x): b\in H\},
\qquad
F_\lambda=\{(z,x)\mapsto e^{\lambda g(z,x)}: g\in G\}.
\]
By Talagrand's contraction principle,
\[
\hat{\mathcal R}_m(F_\lambda)
:=\E_{\epsilon}\!\left[\sup_{h\in F_\lambda}\frac{1}{m}\sum_{i=1}^m \epsilon_i h(z_i,x_i)\ \Big|\ z^m,x^m\right]
\le |\lambda| e^{2|\lambda| B}\hat{\mathcal R}_m(G).
\]
Moreover,
\[
\hat{\mathcal R}_m(G)
=\E_{\epsilon}\!\left[\sup_{b\in H}\frac{1}{m}\sum_{i=1}^m \epsilon_i(f(z_i)-b(x_i))\ \Big|\ z^m,x^m\right]
=\hat{\mathcal R}_m(H),
\]
since $f$ is fixed and $\sum_i\epsilon_i f(z_i)$ drops out of the supremum.

By the symmetrization bound (\cite{Mohri2018FML}, Theorem 3.3),
for this fixed $\lambda$, with probability at least $1-\delta_\lambda$,
\[
\sup_{b\in H}\left|
\E\!\left[e^{\lambda(f(z)-b(x))}\right]
-
\hat{\E}\!\left[e^{\lambda(f(z)-b(x))}\right]
\right|
\le
2\hat{\mathcal R}_m(F_\lambda)
+
2e^{2|\lambda| B}\sqrt{\frac{\log(1/\delta_\lambda)}{2m}}
\]
and hence
\[
\sup_{b\in H}\left|
\E\!\left[e^{\lambda(f(z)-b(x))}\right]
-
\hat{\E}\!\left[e^{\lambda(f(z)-b(x))}\right]
\right|
\le
2|\lambda |e^{2|\lambda| B}\hat{\mathcal R}_m(H)
+
2e^{2|\lambda| B}\sqrt{\frac{\log(1/\delta_\lambda)}{2m}}.
\]

Now convert to $\frac{1}{\lambda}\log(\cdot)$. Since $f-b\in[-2B,2B]$,
both $\E[e^{\lambda(f-b)}]$ and $\hat{\E}[e^{\lambda(f-b)}]$ lie in
$[e^{-2|\lambda |B},e^{2|\lambda| B}]$. On this interval,
$u\mapsto \frac{1}{\lambda}\log u$ is $(e^{2|\lambda |B}/|\lambda|)$-Lipschitz, so
\[
\sup_{b\in H}\left|
\frac{1}{\lambda}\log \E\!\left[e^{\lambda(f-b)}\right]
-
\frac{1}{\lambda}\log \hat{\E}\!\left[e^{\lambda(f-b)}\right]
\right|
\le
\frac{e^{2|\lambda| B}}{|\lambda|}\;
\sup_{b\in H}\left|
\E\!\left[e^{\lambda(f-b)}\right]
-
\hat{\E}\!\left[e^{\lambda(f-b)}\right]
\right|.
\]
Substituting the previous bound and using $1 \leq |\lambda|\leq \Lambda$ gives, for fixed $\lambda$,
with probability at least $1-\delta_\lambda$,
\[
\sup_{b\in H}\left|
\frac{1}{\lambda}\log \E\!\left[e^{\lambda(f-b)}\right]
-
\frac{1}{\lambda}\log \hat{\E}\!\left[e^{\lambda(f-b)}\right]
\right|
\le
2e^{4\Lambda B}\hat{\mathcal R}_m(H)
+
2e^{4\Lambda B}\sqrt{\frac{\log(1/\delta_\lambda)}{2m}}.
\]

It remains to make the bound uniform over $\lambda\in[1,\Lambda] \cup [-\Lambda, -1]$.
Define the $1/m$-grid
\[
\mathcal N_+ := 1 + \Big\{ \frac{j}{m}: j=0,1,\dots,J\Big\},
\qquad
J:=\left\lfloor m(\Lambda-1)\right\rfloor,
\]

and 

\[\mathcal{N} = \mathcal{N}_+ \cup (-\mathcal N_+)\]

so that $|\mathcal N|=2(J+1)\leq 4\Lambda m$ and for every $\lambda\in[-\Lambda, -1]\cup [1,\Lambda]$
there exists $\lambda'\in\mathcal N$ with $|\lambda-\lambda'|\le 1/m$.
Apply the fixed-$\lambda$ bound with $\delta_{\lambda'}:=\delta/|\mathcal N|$
and take a union bound over $\lambda'\in\mathcal N$. With probability at least $1-\delta$,
for all $\lambda'\in\mathcal N$,
\begin{align*}
\sup_{b\in H}\left|
\frac{1}{\lambda'}\log \E\!\left[e^{\lambda'(f-b)}\right]
-
\frac{1}{\lambda'}\log \hat{\E}\!\left[e^{\lambda'(f-b)}\right]
\right|
& \le
2e^{4\Lambda B}\hat{\mathcal R}_m(H)
+
2e^{4\Lambda B}\sqrt{\frac{\log\!\big(\frac{|\mathcal N|}{\delta}\big)}{2m}}\\
&\le
2e^{4\Lambda B}\hat{\mathcal R}_m(H)
+
2e^{4\Lambda B}\sqrt{\frac{\log\!\big(\frac{4\Lambda m}{\delta}\big)}{2m}}.
\end{align*}

Now fix any $\lambda\in[1,\Lambda] \cup [-\Lambda, -1]$ and choose $\lambda'\in\mathcal N$ with
$|\lambda-\lambda'|\le 1/m$. For any fixed $b\in H$, let $U:=f(z)-b(x)\in[-2B,2B]$ and define
\[
\psi(\lambda):=\frac{1}{\lambda}\log \E[e^{\lambda U}],
\qquad
\hat\psi(\lambda):=\frac{1}{\lambda}\log \hat{\E}[e^{\lambda U}].
\]
We have 
\[
\psi'(\lambda)=\frac{1}{\lambda^2}\Big(\lambda\,\E_{\lambda}[U]-\log \E[e^{\lambda U}]\Big),
\]
where $\E_{\lambda}$ denotes expectation under the tilted law with density $\propto e^{\lambda U}$.
Since $\E_{\lambda}[U]\in[-2B,2B]$ and $\log \E[e^{\lambda U}]\in[-2B|\lambda|,2B|\lambda|]$, we have
$|\psi'(\lambda)|\le 4B/|\lambda|\le 4B$ for all $\lambda$ with $|\lambda|\ge 1$. The same bound holds for $\hat\psi'(\lambda)$,
since $\hat{\E}[e^{\lambda U}]\in[e^{-2B|\lambda|},e^{2B|\lambda|}]$. Hence
\[
|\psi(\lambda)-\psi(\lambda')|\le 4B|\lambda-\lambda'|\le \frac{4B}{m},
\qquad
|\hat\psi(\lambda)-\hat\psi(\lambda')|\le \frac{4B}{m}.
\]
Therefore, uniformly over $b\in H$,
\[
\left|\psi(\lambda)-\hat\psi(\lambda)\right|
\le
\left|\psi(\lambda')-\hat\psi(\lambda')\right|+\frac{8B}{m}.
\]
Taking the supremum over $b\in H$ and using the grid bound at $\lambda'$
completes the proof.
\end{proof}

\begin{corollary}[Proposition \ref{prop:lcb learning bounds}]\label{cor:lcb learning bound apdx}
\end{corollary}
\begin{proof}
Let $\hat{\lambda}, \hat{b}$ be the empirical minimizers of $\hat{J}$, given a dataset $(x^i_{t + 1}, x^i_t)$. 

Apply Lemma \ref{lem:logsumexp uniform} to both $\frac{1}{\hat \lambda}\log \hat \E[e^{\hat\lambda (\hat{v}_t(x_t) - \hat{b}_{t + 1}(x_{t + 1})})]$ and $\frac{1}{\hat{\lambda}}\log \hat \E[e^{-\hat \lambda(\hat{v}_t(x_t) - \hat{b}_{t + 1}(x_{t + 1})})]$ to get 

\[\frac{1}{\hat\lambda}\log \E[e^{\hat \lambda (\hat{v}_t(x_t) - \hat{b}_{t + 1}(x_{t + 1})})] \leq \frac{1}{\hat\lambda}\log \hat \E[e^{\hat \lambda (\hat{v}_t(x_t) - \hat{b}_{t + 1}(x_{t + 1})})] + \epsilon_{0, \Lambda}\]

\[\frac{1}{\hat \lambda}\log \E[e^{-\hat{\lambda}(\hat{v}_t(x_t) - \hat{b}_{t + 1}(x_{t + 1}))})] \leq \frac{1}{\hat \lambda}\log \hat \E[e^{-\hat\lambda(\hat{v}_t(x_t) - \hat{b}_{t + 1}(x_{t + 1})}))] + \epsilon_{0, \Lambda}\]

with $\epsilon_{0, \Lambda} := e^{4B\Lambda}(2\hat{\mathcal R}_m(H) + 2\sqrt{\frac{\log (4\Lambda m/\eta)}{2m}}) + 8B/m$.

By assumption, then, we have (on the good event from Lemma \ref{lem:logsumexp uniform}):

\[J(\hat{\lambda}, \hat{b}) - 2\epsilon_{0, \Lambda}\leq \hat{J}(\hat{\lambda}, \hat{b}) \leq \hat{J}(\lambda^*, b^*) \leq J(\lambda^*, b^*) + 2\epsilon_{0, \Lambda}\]

which gives the desired result with $\epsilon_0 := 2\epsilon_{0, \Lambda}$.

On this same event, the conclusion 

\[\hat \tau_{\hat{\lambda}, \hat b} \in \tau_{\hat{\lambda}, \hat{b}} \pm \epsilon_0\]

clearly holds. 
\end{proof}

\subsection{Optimization properties}\label{apdx:optimization properties}

It is not the case that $J$ is globally convex in $\lambda$. However, it still admits a unique minimizer, which can be found using standard 1-d optimization methods. In this discussion, $P_{\lambda}$ refers to the exponentially tilted version of the $\hat{s}$ law, $P$: $dP_{\lambda}/dP = e^{\lambda \hat{s} - \psi(\lambda)}$. In our context, $P \gets \hat{q}_{t + 1}(x_{t + 1}) \ppre(x_t | x_{t + 1})$ is a joint law over $x_{t + 1}, x_t$

\begin{proposition}\label{prop:lambda optimizable symmetric apdx}
[Uniqueness of $\lambda$ minimizer of $J_{\pm}$]
Fix $\hat v_t(x_t)$ and $\hat b_{t+1}(x_{t+1})$, and write $\hat s := \hat v_t-\hat b_{t+1}$.
Let
\[
I:=\{\lambda\ge 0:\ \psi(\lambda)<\infty\ \text{and}\ \psi(-\lambda)<\infty\},
\qquad
\psi(\lambda):=\log \E[e^{\lambda \hat s}].
\]
Assume $\hat s$ is not almost surely constant, $I\neq\{0\}$, and
\[
\sup_{\lambda\in I}\Big(D(P_{\lambda}\|P)+D(P_{-\lambda}\|P)\Big)\ \ge\ \log(1/\delta),
\]
where $P_{\lambda}$ is the exponentially tilted law $dP_{\lambda}\propto e^{\lambda \hat s}\,dP$.
Then the objective
\[
J(\lambda)\ :=\ \frac{\psi(\lambda)+\psi(-\lambda)+\log(1/\delta)}{\lambda}
\]
has a unique minimizer $\lambda^*\in I$.
\end{proposition}

\begin{proof}
Let $c = \log{1/\delta}$. $J$ is differentiable on $I$, with stationarity condition

\[\psi'(\lambda)/\lambda - \psi(\lambda) /\lambda^2  - \psi'(-\lambda)/\lambda - \psi(-\lambda)/\lambda^2 - c/\lambda^2 = 0\]
\[\iff \lambda \E_{P_{\lambda}}[\hat{s}] - \psi(\lambda)  - \lambda \E_{P_{-\lambda}}[\hat{s} ] - \psi(-\lambda)= c\]
\[\iff D(P_{\lambda} || P) + D(P_{-\lambda} || P )= c\]

We know that $\zeta(\lambda) := D(P_{\lambda} || P) + D(P_{-\lambda} || P)$ is a continuous, strictly increasing function: since $\hat s$ is non-constant, we have $ \zeta'(\lambda) := \lambda\Var_{P_{\lambda}}[\hat{s}] + \lambda \Var_{P_{-\lambda}}[\hat{s}] > 0$. Also $\zeta(0) := 0$. 

Thus, the critical point can only be attained for one $\lambda \in I$. Moreover, the condition that $\sup_{\lambda \in I }\zeta(\lambda) \geq  \log{1/\delta}$ ensures that the critical point is reached.  

We have also seen that, $J' \leq 0$ to the left of the critical point and $J' \geq 0$ to the right of the critical point. Thus $J$ is unimodal.
\end{proof}

\subsection{Iterative baseline learning algorithms}\label{apdx:training alg}

Theoretically, when learning LCBs $B_{T + 1} \dots B_1$, we must ensure they are mutually compatible, to avoid covariate shift during training. This is the subject of this section. We first give a two-pass algorithm, whose training complexity scales like $O(T^2m)$ where $m$ is the size of the training set for each LCB. Then we give a one-pass algorithm whose training complexity is $O(Tm)$, but requires additional mild assumptions to be provably correct. Our experiments are guided by the latter algorithm.

\textbf{We emphasize that this process is a one-time upfront cost, and no subsequent training of the baseline is necessary.}

\subsubsection{Two-Pass Algorithm}

Algorithm \ref{alg:two pass lcb training} shows a two-pass scheme for learning data-dependent baselines $B_{T + 1} \dots B_1$. The idea is to iterate over time-steps $t = T \dots 0$, learning $B_{t + 1}$ at each step $t$. At step $t$, having learned $B_{T + 1} \dots B_{t + 2}$, we may evolve an i.i.d. set of particles $\{x_{t + 1}^i\}_{i = 1}^m \overset{iid}{\sim} \hat q_{t + 1}$ by applying Algorithm \ref{alg:baseline rejection sampling}. We may then gather $y^i_t \sim \ppre_t(\cdot | x_{t + 1}^i)$, and then train on the LCB objective with pairs $\{(x^i_{t + 1}, y^i_t)\}_{i \in [m]}$. Union bounding over all steps $t \in [T]$ leads to the guarantee that, with high probability, the guarantee of Theorem \ref{thm:lcb tv bound} holds simultaneously at all timesteps, and that $B_{T + 1} \dots B_1$ is mutually compatible (up to factors of $\log T$). 

To accommodate the union bound, the statistical slack used in Proposition \ref{prop:lcb learning bounds} must be increased very slightly:
\begin{equation}\label{eq:epsilon 0 T}
\epsilon_{0, T} :=  4e^{4B\Lambda}(\hat{\mathcal{{R}}}_m(H) + \sqrt{\frac{\log4(T+1)\Lambda m/\eta}{2m}} ) + 16B/m
\end{equation}

\begin{algorithm}
\caption{Two-Pass LCB Training}\label{alg:two pass lcb training}
\begin{algorithmic}[1]

\STATE \textbf{Input:} particle count $m$, confidence level $\delta$, terminal time $T$.

\STATE Draw $x_T^1,\dots,x_T^m \sim \ppre_T(\cdot)$.
\STATE Compute values $\hat v_T(x_T^i)$ for $i\in[m]$.
\STATE Fit scalar $(\hat\lambda_{T+1},\hat\tau_{T+1})$ by ERM on
$\{\hat v_T(x_T^i)\}_{i=1}^m$
\STATE Set $\hat B_{T+1}\leftarrow \hat\tau_{T+1}+\epsilon_{0, T}$.

\FOR{$t = T$ \textbf{down to} $1$}


    \FOR{$i=1$ \textbf{to} $m$}

        \REPEAT
            \STATE Sample candidate $x' \sim \ppre_T(\cdot)$.
            \STATE Accept with prob.\
            $\min\{1,\exp(\hat v_T(x')-B_{T+1})\}$.
        \UNTIL{accepted}
        \STATE Set $x_T^i \leftarrow x'$.

        \FOR{$s=T-1$ \textbf{down to} $t$}
            \REPEAT
                \STATE Sample candidate
                $x' \sim \ppre_s(\cdot\mid x_{s+1}^i)$.
                \STATE Accept with prob.\
                $\min\{1,\exp(\hat v_s(x')-B_{s+1}(x_{s+1}^i))\}$.
            \UNTIL{accepted}
            \STATE Set $x_s^i \leftarrow x'$.
        \ENDFOR

    \ENDFOR


    \FOR{$i=1$ \textbf{to} $m$}
        \STATE Sample
        $y_t^i \sim \ppre_{t-1}(\cdot\mid x_t^i)$.
    \ENDFOR

    \STATE Fit $(\hat\lambda_t,\hat b_t)$ on
    $\{(x_t^i,y_t^i)\}_{i=1}^m$
    by ERM for the LCB objective.
    \STATE Set
    $\hat B_t(x_t)\leftarrow
    \hat b_t(x_t)+
    \hat\tau_{\hat\lambda_t,\hat b_t}
    +\epsilon_{0, T}$.

\ENDFOR

\STATE \textbf{return}
Baseline system $\{\hat B_{T+1},\hat B_T,\dots,\hat B_1\}$.

\end{algorithmic}
\end{algorithm}

\subsubsection{One-pass LCB training}

Algorithm \ref{alg:one pass lcb training} shows a one-pass approach to learning the system of baselines $\{B_{T + 1} \dots B_1\}$. Here, the mechanism for enforcing mutual compatibility at training time is to evolve the training set as a system of $m$ particles $x^1 \dots x^m$. Given particles on the $t$-th step, $x^1_t \dots x^m_t$, we first train on the joint set $(x^1_t, y^1_t) \dots (x^m_t, y^m_t)$ where $y^i_t \sim \ppre_{t - 1}(\cdot | x_t)$. This gives us the baseline function $B_t$, with which we can sample aligned particles $x^1_{t - 1} \dots x^m_{t - 1}$ where $x^i_{t - 1} \sim \hat{q}(\cdot | x^i_t)$. The total training complexity is then $O(mT)$ instead of $O(mT^2)$.

The wrinkle is that, at timestep $t$, the accepted particles $\{x^i_t\}$, though marginally distributed according to $\hat{q}_t$, may depend on each other via $B_{t + 1}$. Therefore, Proposition \ref{prop:lcb learning bounds}'s iid assumption is violated. If we choose $B_{t + 1}$ from a finite set, then this problem can be avoided by means of a covering argument. 

Algorithm \ref{alg:one pass lcb training} forces the baseline to live in a finite set through $\mathrm{Quantize}_\mathcal{T}(\tau)$, which rounds $\tau$ upwards in the set $\mathcal T$ with size $|\mathcal T|= C$. Then, so long as $\mathcal H$ is finite with $|\mathcal H| = K$, the total number of pairs $(b_{t + 1},  \tau_{t + 1})$ is $CK$. and therefore the number of possible baselines $b_{t + 1} + \tau_{t + 1} + \epsilon_{0, \mathrm{adaptive}}$  is also $CK$. 

Here,  
\[
\epsilon_{0,\mathrm{adaptive}}
=
O\left(
e^{4B\Lambda}
\left[
B\sqrt{\frac{\log K}{m}}
+
\sqrt{
\frac{
T\log(CK)
+
\log\left(\frac{4\Lambda mT}{\eta}\right)
}{m}}
\right]
+
\frac{B}{m}
\right).
\]
is a slightly enlarged version of the slack from Proposition \ref{prop:lcb learning bounds} to accommodate the covering argument.

\begin{theorem}
Assume $H$ is finite. After running Algorithm \ref{alg:one pass lcb training}, then with probability $1 - \eta$, the following holds simultaneously for all $t \in [T]$
\[J_t(\hat{\lambda}_t, \hat{b}_{t + 1}) - J_t^* \leq 2\epsilon_{0, \mathrm{adaptive}} \qquad \hat{\tau}_{\hat{\lambda}_t, \hat{b}_{t + 1}} \in \tau_{\hat{\lambda}_t, \hat{b}_{t + 1}} \pm \epsilon_{0, \mathrm{adaptive}},\]
where the distribution $\hat{q}_{t + 1}$ is induced by the learned baselines $B_{T + 1} \dots B_{t +2}$ for each $t$. 

Thus, on this event, $B_{T + 1} \dots B_1$ is a mutually compatible set of baselines satisfying the assumptions of the TV:MGF lemma (Lemma \ref{lem:tv mgf body})\footnote{Apart from worst-case coverage, which remains an assumption, and is addressed in Appendix \ref{apdx:coverage}}. 
\end{theorem}
\begin{proof}
Fix a time step $t$. Consider a fixed set of baselines $B^{t + 1} = \{B_{T + 1} \dots B_{t + 2}\}$. For this set of baselines, consider sampling $\{x^i_{t + 1}\}_{i \in [m]}$ according to Algorithm \ref{alg:baseline rejection sampling}. Since these baselines are fixed, these particles constitute i.i.d. draws from a distribution $\hat{q}_{t + 1}$ depending on $B^{t + 1}$. Consider $y_t^i \sim \ppre_t(\cdot | x^i_{t + 1})$ i.i.d. for $i = 1 \dots m$, forming the pairs $\{(x^i_{t + 1}, y^i_t)\}_{i = 1}^m$. This set satisfies the assumptions of Proposition \ref{prop:lcb learning bounds}, meaning that for this particular $\hat{q}_{t + 1}$, we have $J_t(\hat{\lambda}_t, \hat{b}_{t + 1}) - J_t^* \leq 2\epsilon_0$. 

But the baseline prefix $B^{t + 1}$ was arbitrary, and so the above holds for any such prefix. We may therefore cover (union bound over) the set of prefixes, which has size $(CK)^{T - t} \leq (CK)^T$. Therefore, even for a random setting of $B^{t + 1}$ with arbitrary dependence on the particles, we have the desired conclusion at $t \in [T]$. Union bounding once again over $T$ completes the proof. 

The Rademacher complexity of $H$ is bounded by $B\sqrt{\frac{\log K}{m}}$. The two union bounds over $(CK)^T$ and $T$ events then produce $\epsilon_{0, \mathrm{adaptive}}$ from Proposition \ref{prop:lcb learning bounds}'s $\epsilon_0$. 

Because $\hat{\tau}_{\hat{\lambda}_t, \hat{b}_t}$ is rounded up, each baseline $B_{t}$ remains a valid joint baseline and the full set $B_{T + 1} \dots B_1$ is mutually compatible. 
\end{proof}

Though Algorithm \ref{alg:one pass lcb training} guides our experiments, we explain how we simplify the training process still further in practice in Appendix \ref{apdx:experiments}.

\begin{algorithm}
\caption{One-pass LCB Training}\label{alg:one pass lcb training}
\begin{algorithmic}[1]

\STATE \textbf{Input:} particle count $m$, confidence level $\delta$, terminal time $T$, finite hypothesis class $\mathcal H$. 
\STATE Draw $y_T^1,\dots,y_T^m \sim \ppre_T(\cdot)$.

\STATE Fit scalar $(\hat\lambda_{T+1},\hat\tau_{T+1})$ by ERM on $\{\hat v_T(y_T^i)\}_{i=1}^m$
      (i.e.\ minimize the chosen LCB objective at level $\delta$), and set $B_{T + 1} := \mathrm{Quantize}_\mathcal{T}(\hat{\tau}_{T + 1}) + \epsilon_{0, \mathrm{adaptive}}$.
\FOR{$i=1$ \textbf{to} $m$}
        \REPEAT
            \STATE Sample candidate $x' \sim \ppre_{T}(\cdot)$.
            \STATE Accept with prob.\ $\min\{1,\exp(\hat v_{T}(x')-B_{T + 1}\}$.
        \UNTIL{accepted}
        \STATE Set $x_{T}^i \leftarrow x'$.
    \ENDFOR

\FOR{$t = T$ \textbf{down to} $1$}

    \FOR{$i=1$ \textbf{to} $m$}
        \STATE Sample $y_t^i \sim \ppre_{t-1}(\cdot \mid x_t^i)$.
    \ENDFOR

    \STATE Fit $(\hat\lambda_t,\hat b_t) \in (2, \Lambda) \times \mathcal H$ on $\{(x_t^i,y_t^i)\}_{i=1}^m$ by ERM for the LCB objective
    \STATE Set $B_t(x_t)\leftarrow \hat b_t(x_t)+ \mathrm{Quantize}_\mathcal{T}( \hat\tau_{\hat\lambda_t,\hat b_t}) +\epsilon_{0, \mathrm{adaptive}}$.

    \FOR{$i=1$ \textbf{to} $m$}
        \REPEAT
            \STATE Sample candidate $x' \sim \ppre_{t-1}(\cdot\mid x_t^i)$.
            \STATE Accept with prob.\ $\min\{1,\exp(\hat v_{t-1}(x')-B_t(x_t^i))\}$.
        \UNTIL{accepted}
        \STATE Set $x_{t-1}^i \leftarrow x'$.
    \ENDFOR
\ENDFOR

\STATE \textbf{return} Baseline system $\{B_{T+1},B_T,\dots,B_1\}$ and aligned samples $\{x_t^i\}_{t=0}^T$.
\end{algorithmic}
\end{algorithm}

\subsection{Practical variants of optimization}\label{apdx:opt practice}

In practice, $b_t$ may be parametrized by a neural network architecture, $b_\phi(x, t)$. We may consider optimizing $J_t$ (a) only in $\lambda$, given fixed $b_{t + 1}$ (b) only $b_{t + 1}$, given fixed $\lambda$, or (c) both $b_{t + 1}$ and $\lambda$. 

In the case of (a), it is most natural to fix $b_{t + 1} \gets \hat{\E}[\hat{v}_t(x_t) | x_{t + 1}]$, which can be trained using an MSE-regression objective with samples from the pretrained model. Another strong option is to use $b_{t + 1} \gets \hat{v}_t$, which eliminates most of the training burden of the LCB. This may be important if training $b_\phi$ would necessitate a large number of parallel particles $m$ in Algorithm \ref{alg:one pass lcb training}.

$\lambda$ can be found using ternary search. In the case of (b), $\lambda$ is essentially viewed as a user parameter, and $b_{t + 1}$ is optimized on the basis of a fixed Chernoff exponent. Since $\lambda$ is easy to optimize (Proposition \ref{prop:lambda optimizable symmetric apdx}), the main theoretical reason to do this would be in order to avoid the $\Lambda$ terms of Proposition \ref{prop:lcb learning bounds}.

Where possible, we anticipate that alternating between gradient steps on $b_\phi(x, t)$ (with multiple epochs per particle step) and search steps for $\lambda$ will be the strongest optimization method, and that practical hacks could enhance training stability (warm start $b_\phi$ on data from the pretrained model, use a schedule cap $\lambda$'s jump size, etc). 

\section{Estimating $\hat{v}_t$}\label{apdx:value estimation}

As mentioned, most methods for estimating the soft-value function in the alignment context make loose approximations like $\hat{v}_t(x_t) \approx \E[r(x_0) | x_t]$ \citep{li2024derivativefreeguidancecontinuousdiscrete} or $\hat{v}_t(x_t) \approx r(\E[x_0 | x_t])$ \citep{yoon2025psisampler}. While such approximations may be practical, they introduce a source of constant error, and are therefore not suitable for our goal of deriving vanishing end-to-end guarantees on TV error.

Given $n$ trajectories $\mathcal D_n := \{(x^i_T, ..., x^i_0, r(x^i_0)\}_{i = 1}^n$, we form $\hat{v}_t$ for each $t$ individually by letting 

\[ \textstyle \tilde{r}_t \in \arg\min_{h \in H}\sum_{i = 1}^n \left(e^{r(x^i_0)} - h(x^i_t)\right)^2\]
\begin{equation}\label{eq:v estimate def}
    \hat{v}_t(x_t) := \log \tilde r_t(x_t)
\end{equation}

Assume (wlog, by clipping) that functions in $H$ are bounded within $e^{-B}$ and $e^B$
\subsection{Error of $\hat{v}_t$}

We now give a fairly crude analysis of the error of $\hat{v}_t$. Rates given here could improve with localized Rademacher complexity arguments \citep{wainwright}. 

We will make the following realizability assumption, {simply for notational simplicity}.

\begin{assumption} $\tilde r_t$ is solving a realizable problem, i.e.
    $h^*(x) :=  e^{v_t(x)} =  \E_{\ppre}[e^{r(x_0)} | x_t = x] \in H$
\end{assumption}

Define $\mathcal{R}_n := \E_{\mathcal{D}_n}[\sup_h \frac{1}{n}\sum_i \sigma_i h(x^i_t)]$ to be the (population) Rademacher complexity of $H$ wrt $\ppre_t$, where $\sigma_i$ are Rademacher RVs.

\begin{lemma}
    \[
\|\tilde r_t(x_t) - e^{v_t(x_t)}\|_{L^2(\ppre_t)}^2
\le
8e^{B}\,\mathcal R_n(H)
+
4e^{2B}\sqrt{\frac{\log(1/\delta)}{2n}}.
\]
\end{lemma}
\begin{proof}
For any $h \in H$, define $R(h):= \E_{x_t}[(h(x_t) - e^{r(x_0)})^2]$, and $\hat{R}(h) = \hat \E[(h(x_t) - e^{r(x_0)})^2]$. Then, for any $h$, $R(h) - R(h^*) = \E[(h(x_t) - h^*(x_t))^2] = \|h - h^*\|_{L^2(\ppre_t)}^2$.

Furthermore, for $\tilde r_t$:
\[R(\tilde r_t) - R(h^*) = R(\tilde r_t) - \hat{R}(\tilde r_t) + \hat{R}(\tilde r_t) - \hat{R}(h^*) + \hat{R}(h^*) - R(h^*) \leq 2\sup_{h \in H}|R(h) - \hat{R}(h)|\]

Introducing a symmetrized dataset $\mathcal D' = \{x^{i'}_T \dots r(x^{i'}_0)\}$, expected value of the RHS (over the data) is bounded as
\[
\E_{\mathcal D}\left[ \sup_h \left|
\frac{1}{n}\sum_i \left(h(x_t^i) - e^{r(x_0^i)}\right)^2
- \mathbb{E}_{x_t, x_0}\!\left[\left(h(x_t) - e^{r(x_0)}\right)^2\right]
\right|\right]
\]
\[\leq
\E_{\mathcal D, \mathcal D'}\left[ \sup_h \left|
\frac{1}{n}\sum_i \left(h(x_t^i) - e^{r(x_0^i)}\right)^2
- \left(h(x^{i'}_t) - e^{r(x^{i'}_0)}\right)^2
\right| \right]
\]
\[\leq
2\E_{\mathcal D}\left[ \sup_h \left|
\frac{1}{n}\sum_i \sigma_i\left(h(x_t^i) - e^{r(x_0^i)}\right)^2
\right| \right]
\]

$h$ and $e^r$ are bounded within $[e^{-B}, e^B]$, so that their difference is bounded within $[-e^{B}, 2e^B]$. Thus  $(z \mapsto (z - e^{r(x^i_0)})^2$ is Lipschitz with parameter $4e^{B}$. By Talagrand's contraction lemma (Lemma 5.7 \citep{Mohri2018FML}), we have the bound
\[\leq
8e^{B}\E_{\mathcal D}\left[ \sup_h \left|
\frac{1}{n}\sum_i \sigma_i h(x_t^i)
\right| \right] = 8e^{B}\mathcal R_n(H)
\]

where $\mathcal{R}_n(H)$ is Rademacher complexity. 

Moreover, since $(h(x_t)-e^{r(x_0)})^2\in[0,4e^{2B}]$, changing a single sample changes
$\sup_{h\in H}|\hat R(h)-R(h)|$ by at most $4e^{2B}/n$. By McDiarmid's inequality, with probability at least $1-\delta$,
\[
\sup_{h\in H}|\hat R(h)-R(h)|
\le
8e^{B}\,\mathcal R_n(H)
+
4e^{2B}\sqrt{\frac{\log(1/\delta)}{2n}}.
\]
\end{proof}

Because $\log z$ is Lipschitz continuous with parameter $e^B$ over $[e^{-B}, e^B]$, we have
\[\E_{z \sim p^{pre}_t(\cdot)}(\hat{v}_t(z) - v_t(z))^2 = \E(\log{\tilde r_t(z)} - \log \E[\exp r | x_t = z])^2\]
\[\leq \E[(e^B(\tilde r_t(z) - E[e^{r(x_0)}))^2 | x_t = z]|] = e^{2B}\|\tilde r_t - e^{v_t(x_t)}\|_{L^2(p^{pre}_t)}^2\]

\section{Experiment details}\label{apdx:experiments}

Code and data is available online at \url{github.com/SunilMadhow/LearnableChernoffBaselines} and \url{drive.google.com/drive/folders/1X-lbH0B5hxp_Jx3U2IyRBY6vV0xqumPs?usp=drive_link}.

\subsection{General implementation notes}

Theoretically, for any $\lambda, b$, we needed to add $\epsilon_0$ in $b(\cdot) + \hat{\tau}_{\lambda, b} + \epsilon_0$ in order to ensure that the joint baseline property is met. This is in order to ensure that $\hat{\tau}_{\lambda, b} \geq \tau_{\lambda, b}$ whp, where $\tau_{\lambda, b}$
is the threshold set by the Chernoff bound (with exponent $\lambda$) in order to make the tail $\delta$. In our experiments, we set $\epsilon_0 \gets 0$, with the understanding that uniform convergence bounds are often loose in practice. Thus, for whichever $\lambda, b$ are chosen by the optimizer in our experiments, we deploy the baseline $b(\cdot) + \hat{\tau}_{\lambda, b}$. Although our theory requires $\lambda \geq 2$, we enforce only that $\lambda \geq 1$. 

\begin{figure}[!ht]
    \centering
    \includegraphics[width=0.7\linewidth]{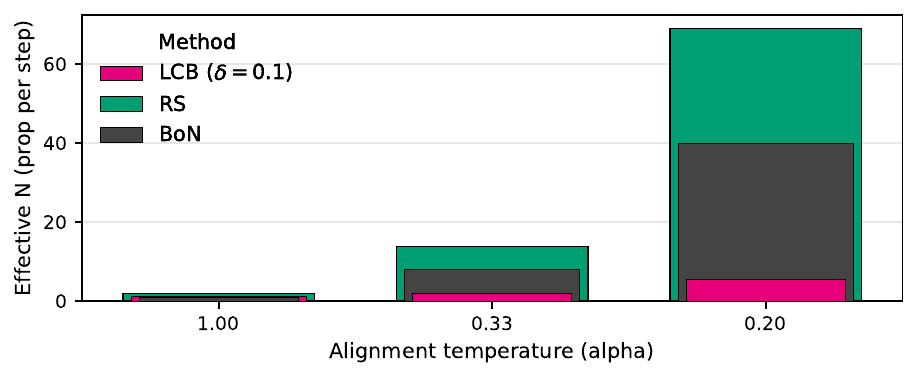}
    \caption{Gaussian mixture: In green, we show the number of proposals required by Rejection Sampling (RS) at temperature $\alpha = 0.2$. In pink, the number required by LCBs ($\delta = 0.1$). According to Figure \ref{fig:apdx mog samples}, we choose a number $N$ so that Bo$N$ produces visually similar results to the other two methods at each temperature. The gray bar for BoN corresponds to this $N$.}
    \label{fix:apdx mog num proposals}
\end{figure}

\subsection{Setting I: Gaussian Mixture}\label{apdx:mog experiments}

\textbf{Methodology}

In this experiment, the soft-value function is formed by solving the regression problem of Appendix \ref{apdx:value estimation}. In particular, we parameterize a 3-layer neural net with layer widths $128, 500, 500$ respectively. We sample $70,000$ training trajectories and train for $20$ epochs on the $\ell_2$ regression task $\exp(\frac{r(x_0)}{\alpha}) \sim x_t$, where $Y \sim X$ denotes regressing Y onto $X$. Letting $\tilde r(x_t)$ be the outcome of this regression problem, we form $\hat{v}_t(x_t) := \log \tilde r_t(x_t)$. 

We parameterize $T$ baseline functions $b_{\phi_t}(x)$. We found it important to maintain a list of $T$ neural networks $\{b_{\phi_t}: \R^2 \rightarrow \R\}$ rather than a single neural net $b_\psi: \R^2 \times [T] \rightarrow \R$, because the iterative procedure of Algorithm \ref{alg:one pass lcb training} tends to lead to forgetting of larger values of $t$. This could likely also be mitigated by simply maintaining the sets $X_T \cdots X_0$ and then replaying the transitions to the neural net. 

Our training procedure for the LCB is based on 2 passes of Algorithm \ref{alg:one pass lcb training}. On the first pass, we fix $\lambda = 1$, and train only the baseline networks. In particular, given $m = 7000$ particles, $X_{t + 1} \in \R^{2\times m}$, we repeat the following $200$ times: (1) Generate $Y_t \sim \ppre_t (\cdot | X_{t + 1})$ (2) Minimize the LCB objective for $b_t$ using full batch gradient descent in $\phi_t$ (3) transition the particles to $X_t \sim \hat{q}_t(\cdot | X_{t + 1})$.

On the second pass, we repeat the same procedure, but we freeze $b_{\phi_t}$ and  replace step (2) with ternary search over $\lambda \in [1, \Lambda]$, with $\Lambda = 12$.

In the last panel of Figure \ref{fig:1d hists}, we chain LCB with $\alpha = 0.2$ and $\delta = 0.3$ with a Bo3 pass, yielding an effective $N$ of $10$. 


\textbf{Additional results}

 In Figure \ref{fig:apdx mog samples}, we show a scatter plot of samples for RS, LCB and BoN. We find that $\delta = 0.1$ suffices to make LCB's samples visually indistinguishable from RS. 

Based on the results for RS/LCB, we pick $N$ such that Bo$N$ produces visually similar results to RS/LCB, and we report $T\times N$ as the proposal complexity. Though we pick $N = 40$ to match the peak heights, we find that we need $N = 45$ in order for Bo$N$ to place the same amount of mass in the reward region as LCB ($\delta = 0.1$). 

In Figure \ref{fix:apdx mog num proposals}, we show the scaling of number of proposals at temperature $\alpha=0.2$ for rejection sampling (Algorithm \ref{alg:exact_rejection_sampling}) and LCB based sampling (Algorithm \ref{alg:one pass lcb training} and Algorithm \ref{alg:baseline rejection sampling}). Specifically, if $L$ is the total number of proposals required to generate $7000$ samples, we report ``effective $N$'', as $L/(T\times7000)$, which is the number of queries to $\ppre$ per sample. Bo$N$ by definition requires $L = N\times T$ proposals in order to generate 1 sample, and so has ``effective $N$'' equal to $(L\times 7000)/(T\times 7000)  = N$.

In Figure \ref{fig:apdx mog delta curves}, we show for each temperature how ``effective $N$'' and positive reward mass scale as we vary $\delta$.

\begin{figure}
    \centering
    \includegraphics[width=0.8\linewidth]{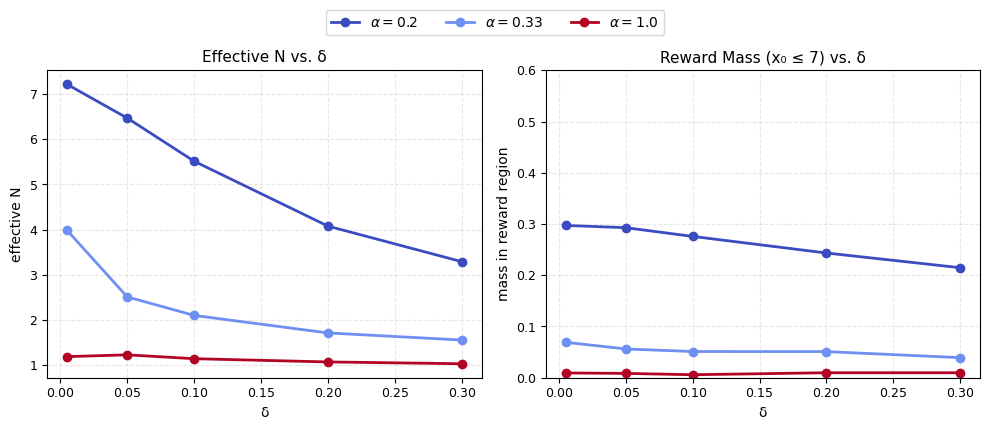}
    \caption{In terms of the reward mass, there are diminishing returns for taking $\delta < 0.1$, while ``effective $N$'' continues to increase for smaller $\delta$.}
    \label{fig:apdx mog delta curves}
\end{figure}

In Figure \ref{fig:values_vs_baselines_filmstrip}, we visualize the value function and baseline learned for $\alpha = 0.2$.  

\begin{figure}
\centering
\setlength{\tabcolsep}{0pt}
\renewcommand{\arraystretch}{0}

\begin{tabular}{@{}cccccccccc@{}}
\includegraphics[width=0.10\textwidth]{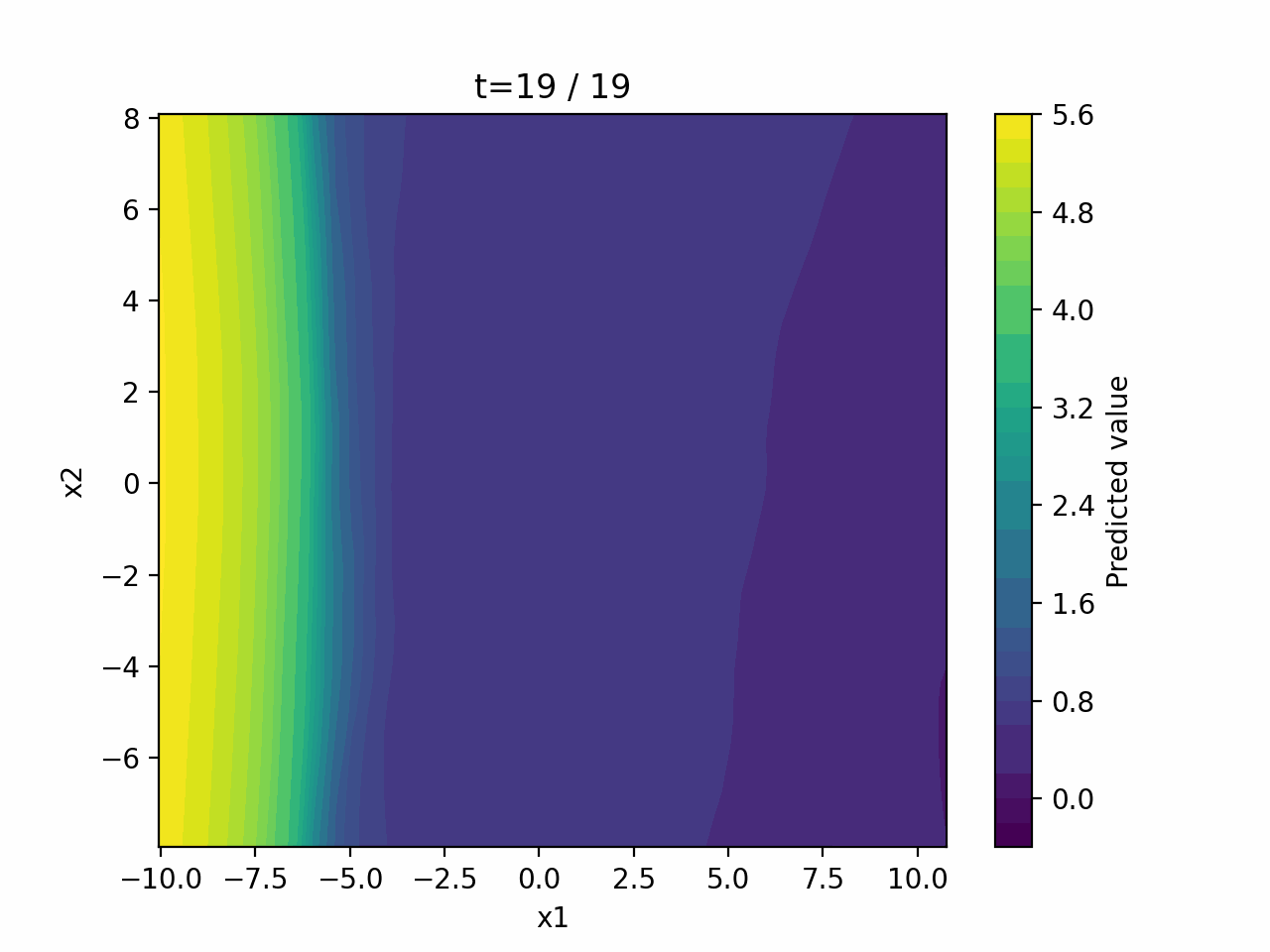} &
\includegraphics[width=0.10\textwidth]{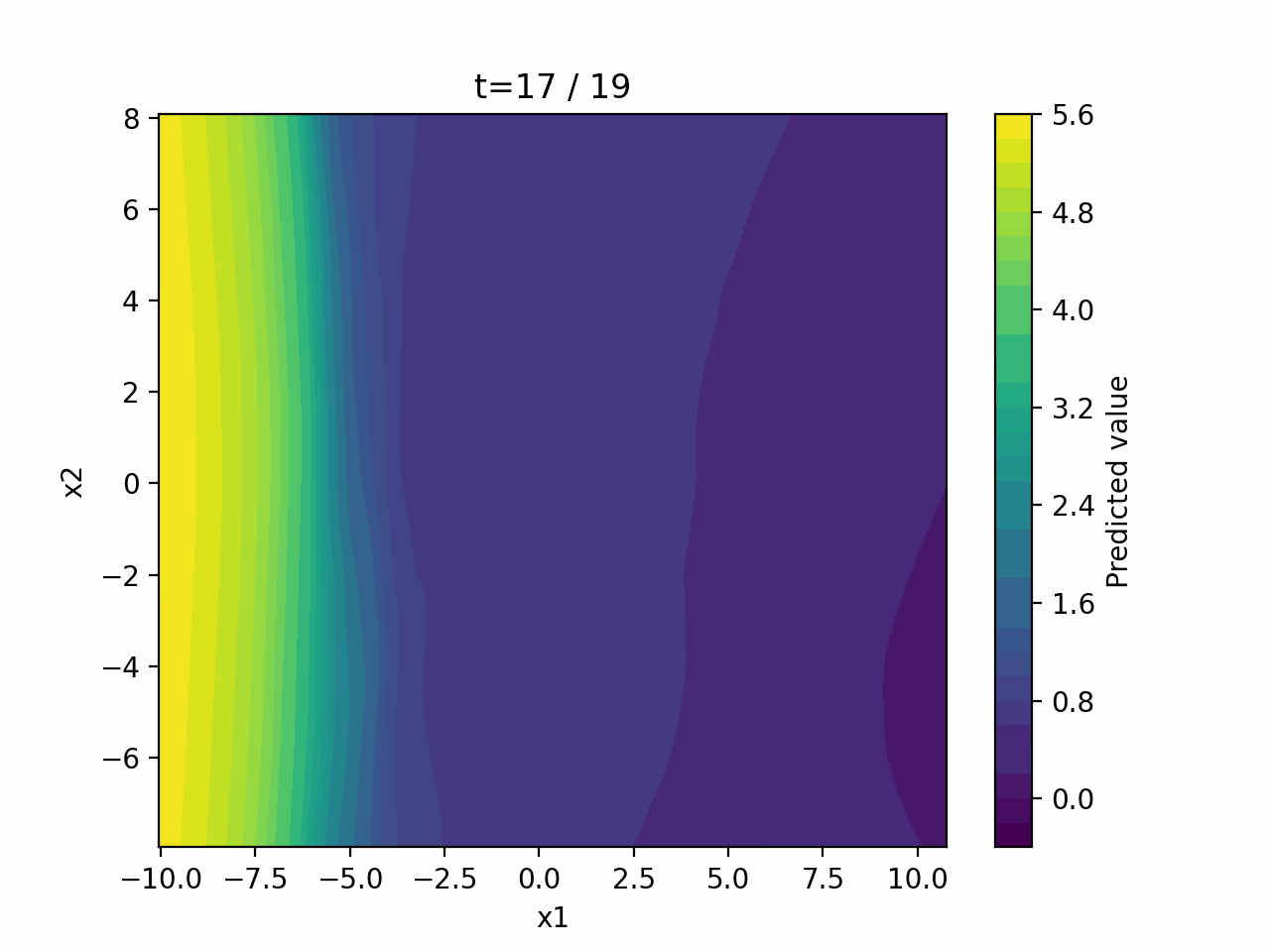} &
\includegraphics[width=0.10\textwidth]{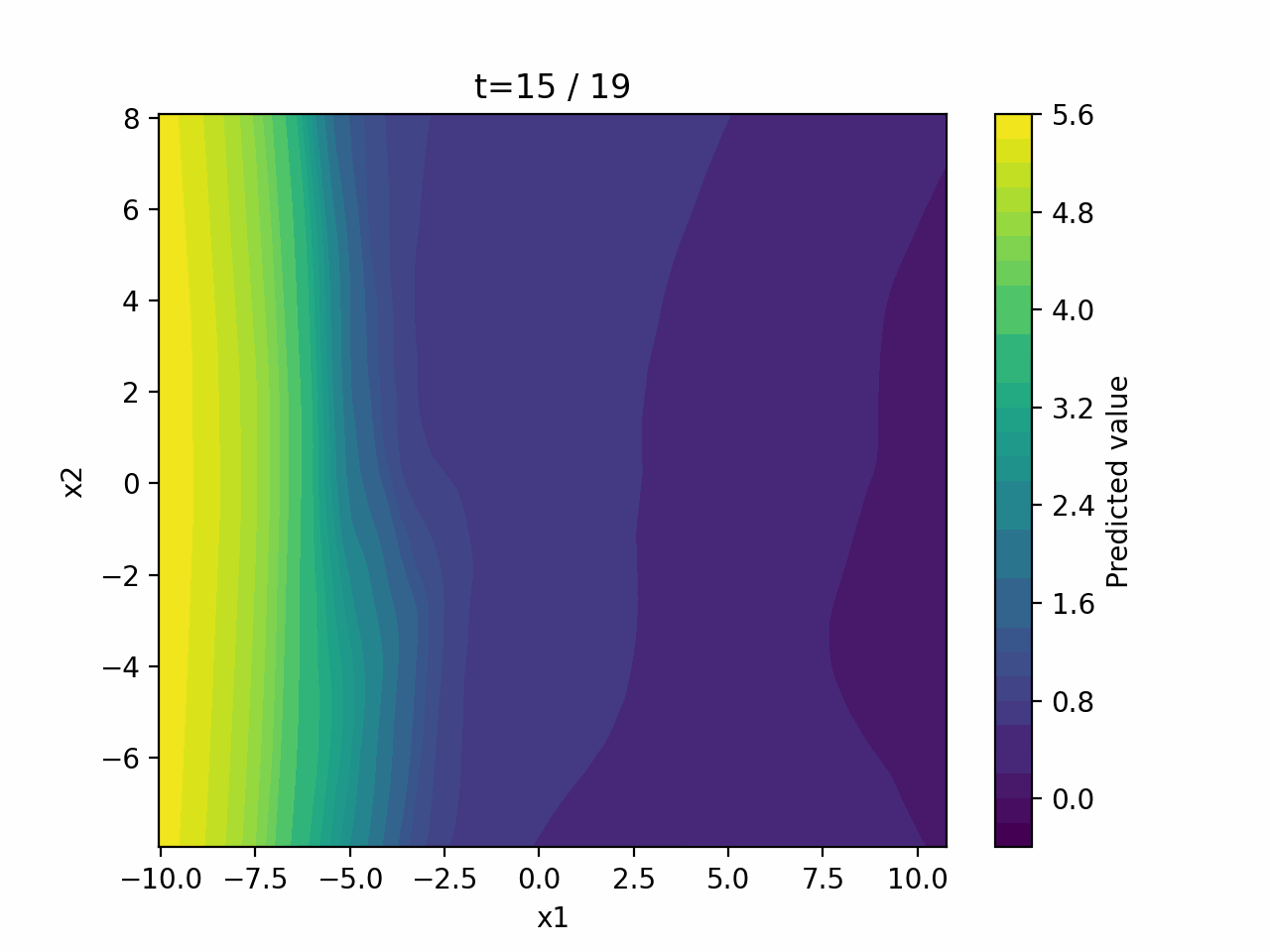} &
\includegraphics[width=0.10\textwidth]{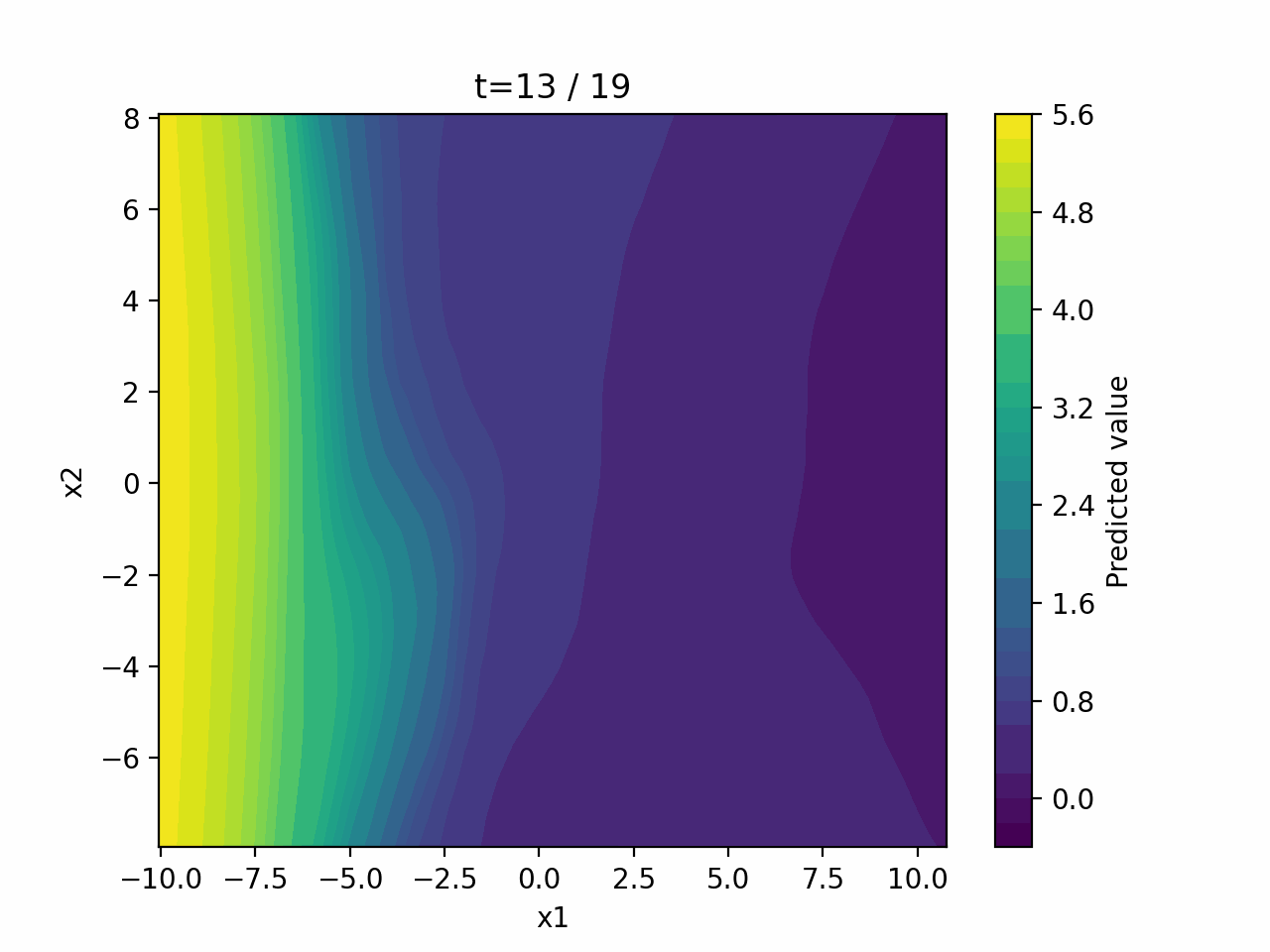} &
\includegraphics[width=0.10\textwidth]{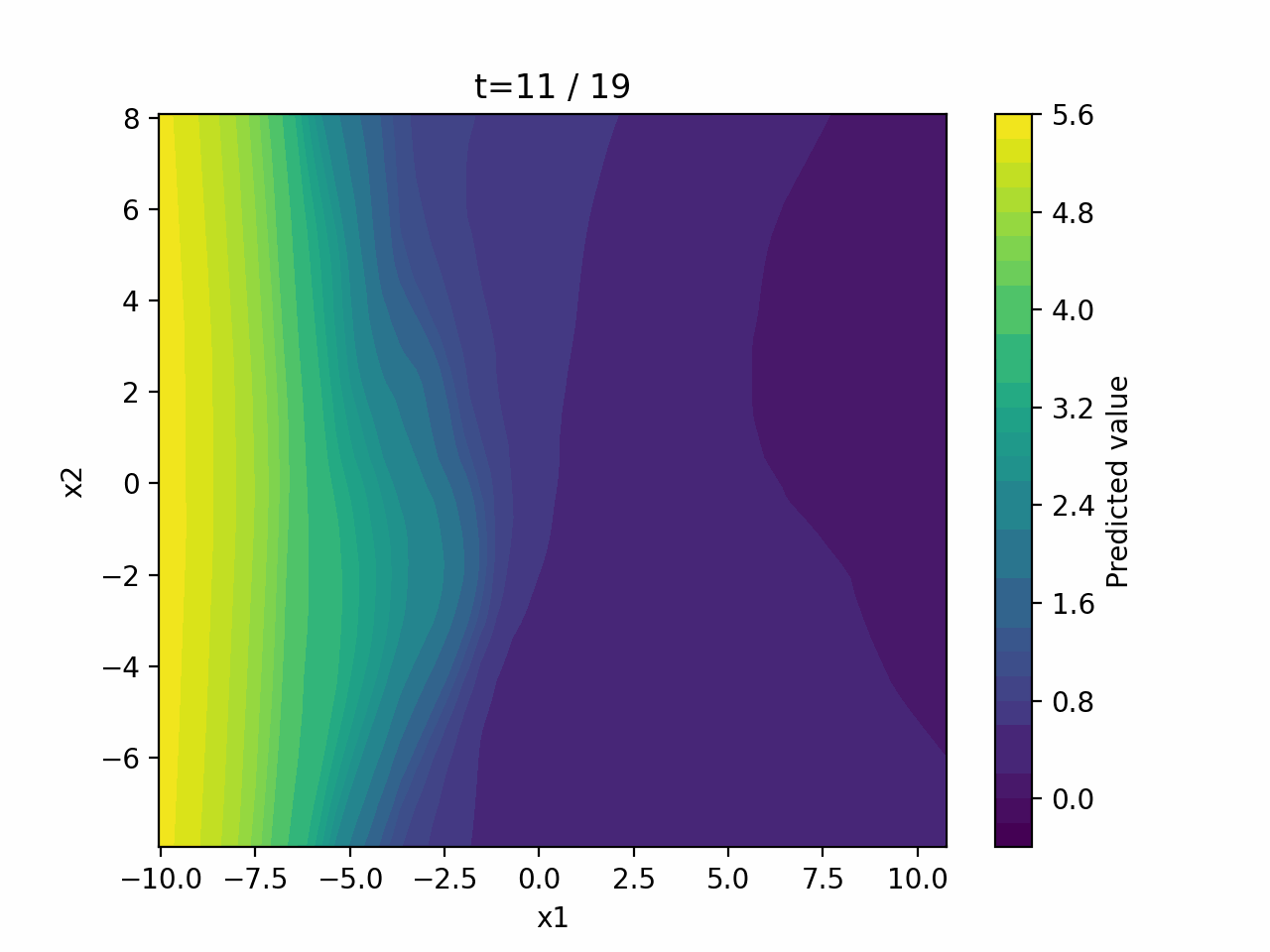} &
\includegraphics[width=0.10\textwidth]{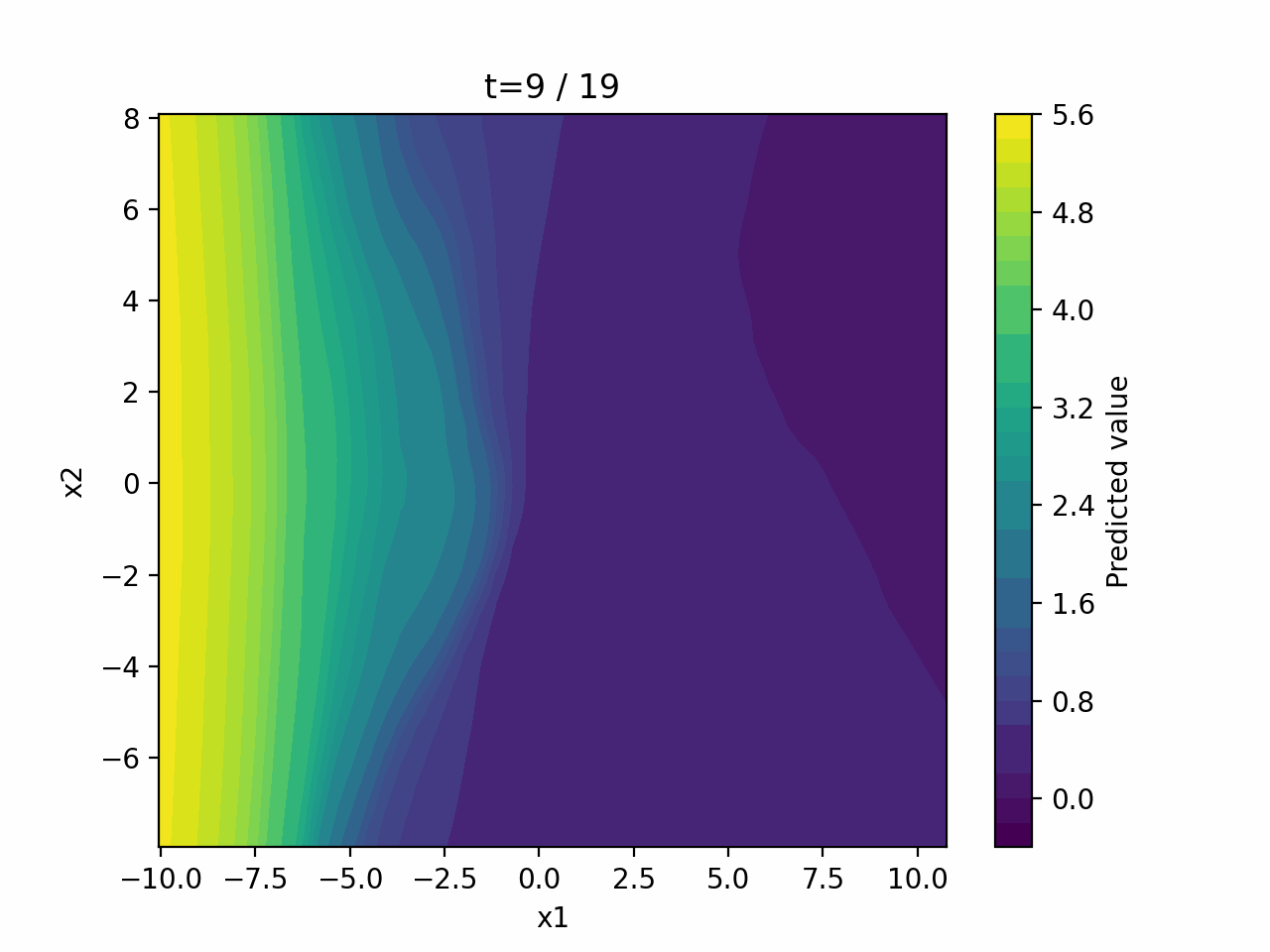} &
\includegraphics[width=0.10\textwidth]{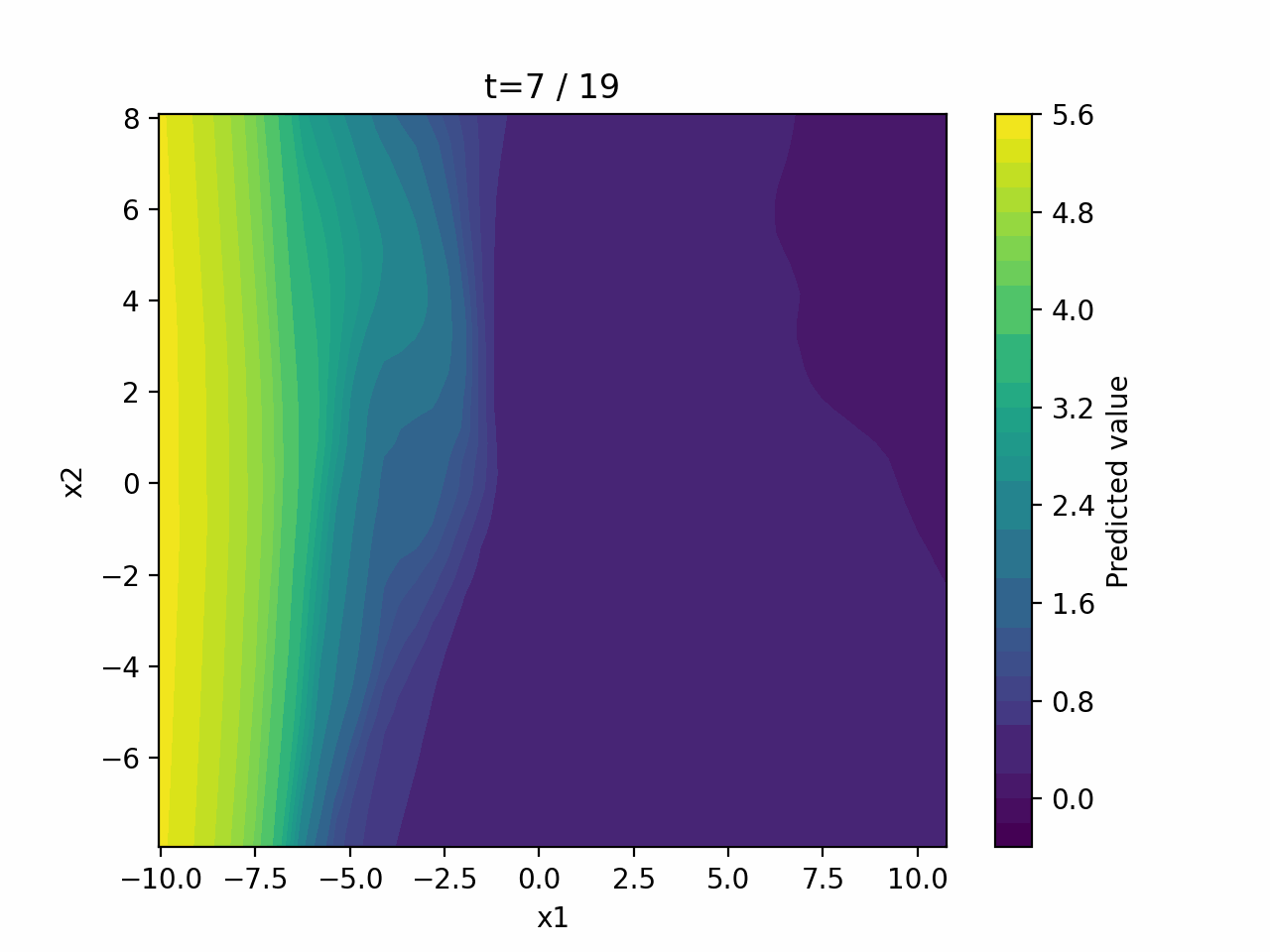} &
\includegraphics[width=0.10\textwidth]{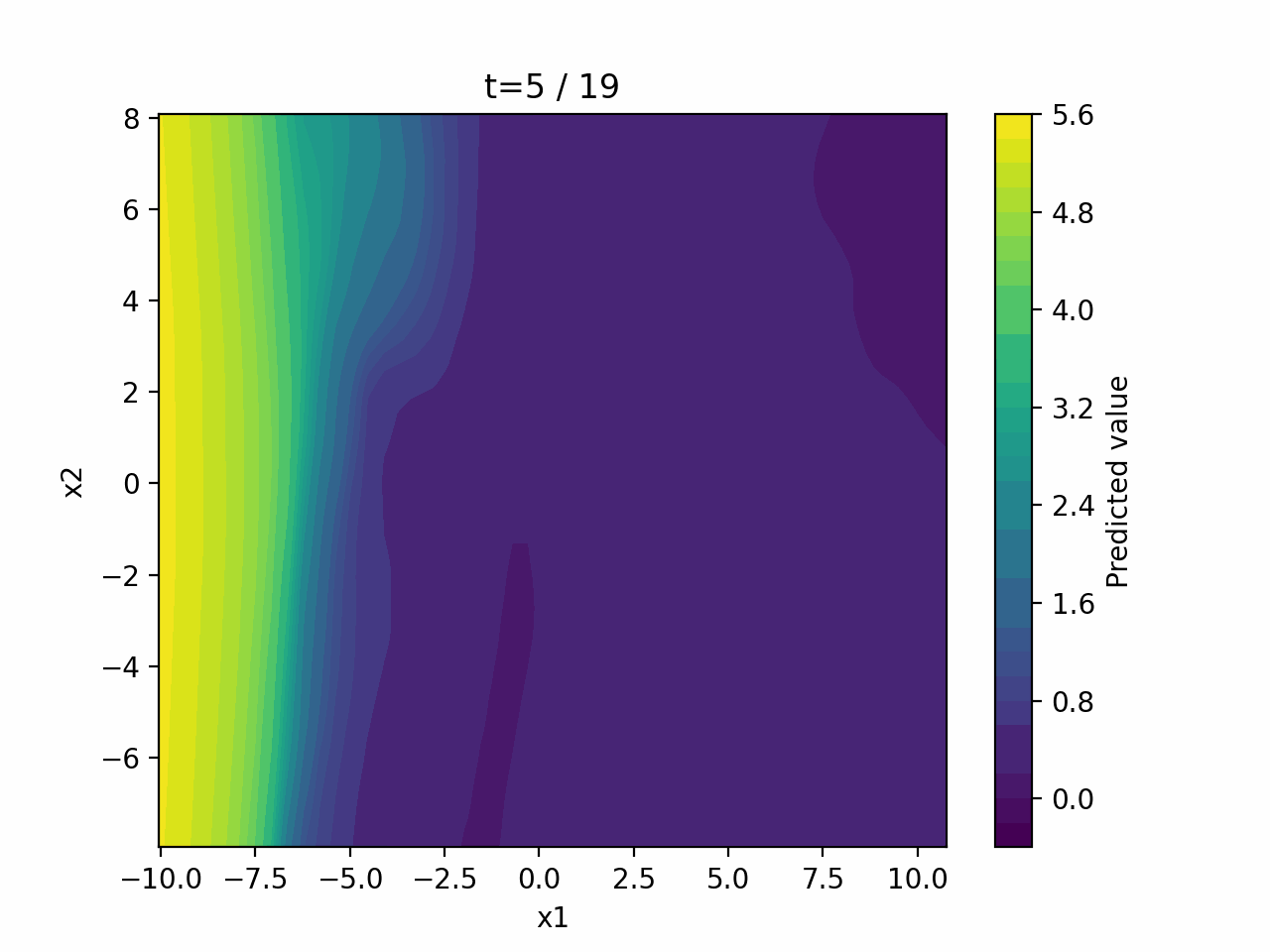} &
\includegraphics[width=0.10\textwidth]{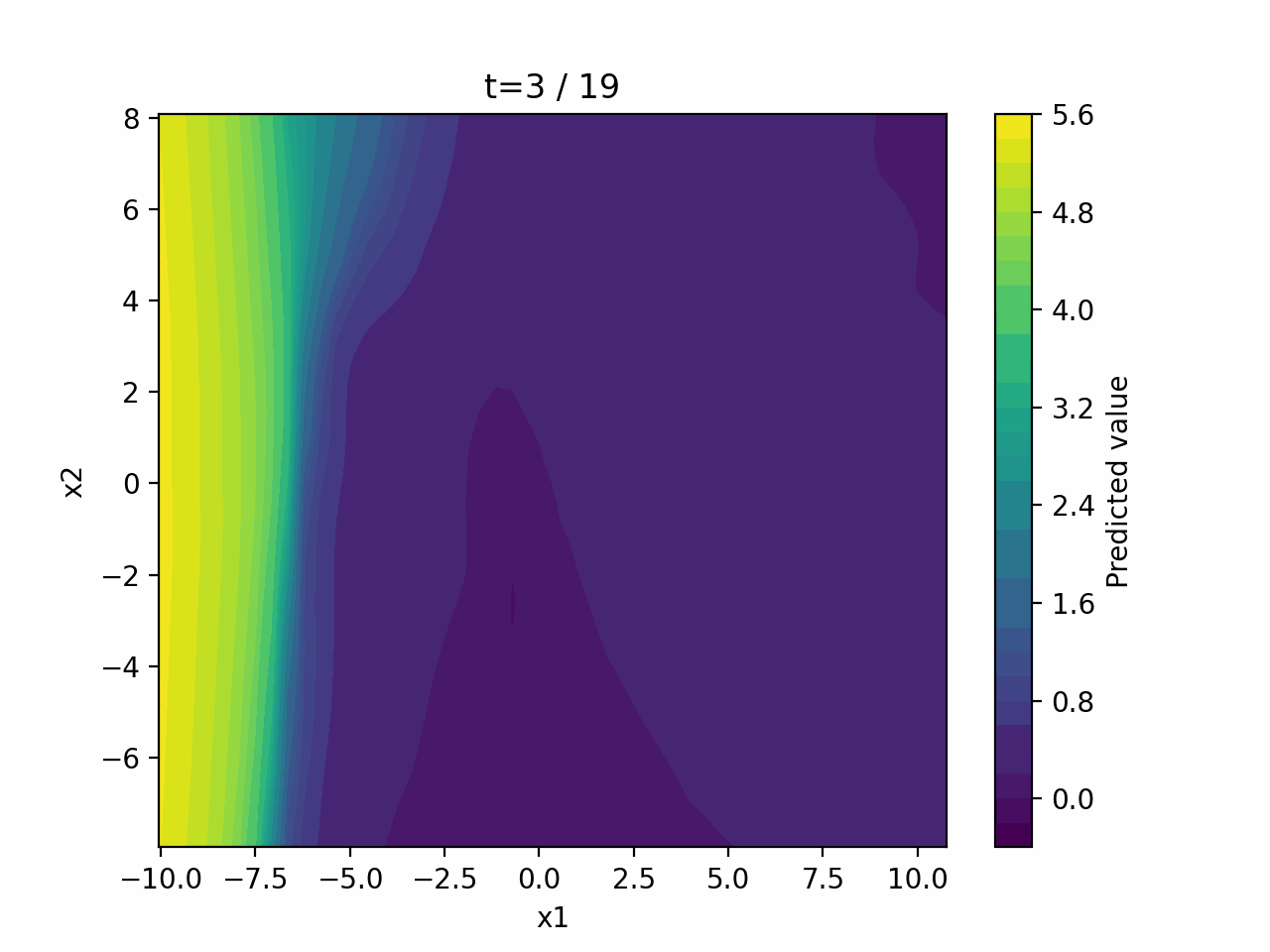} &
\includegraphics[width=0.10\textwidth]{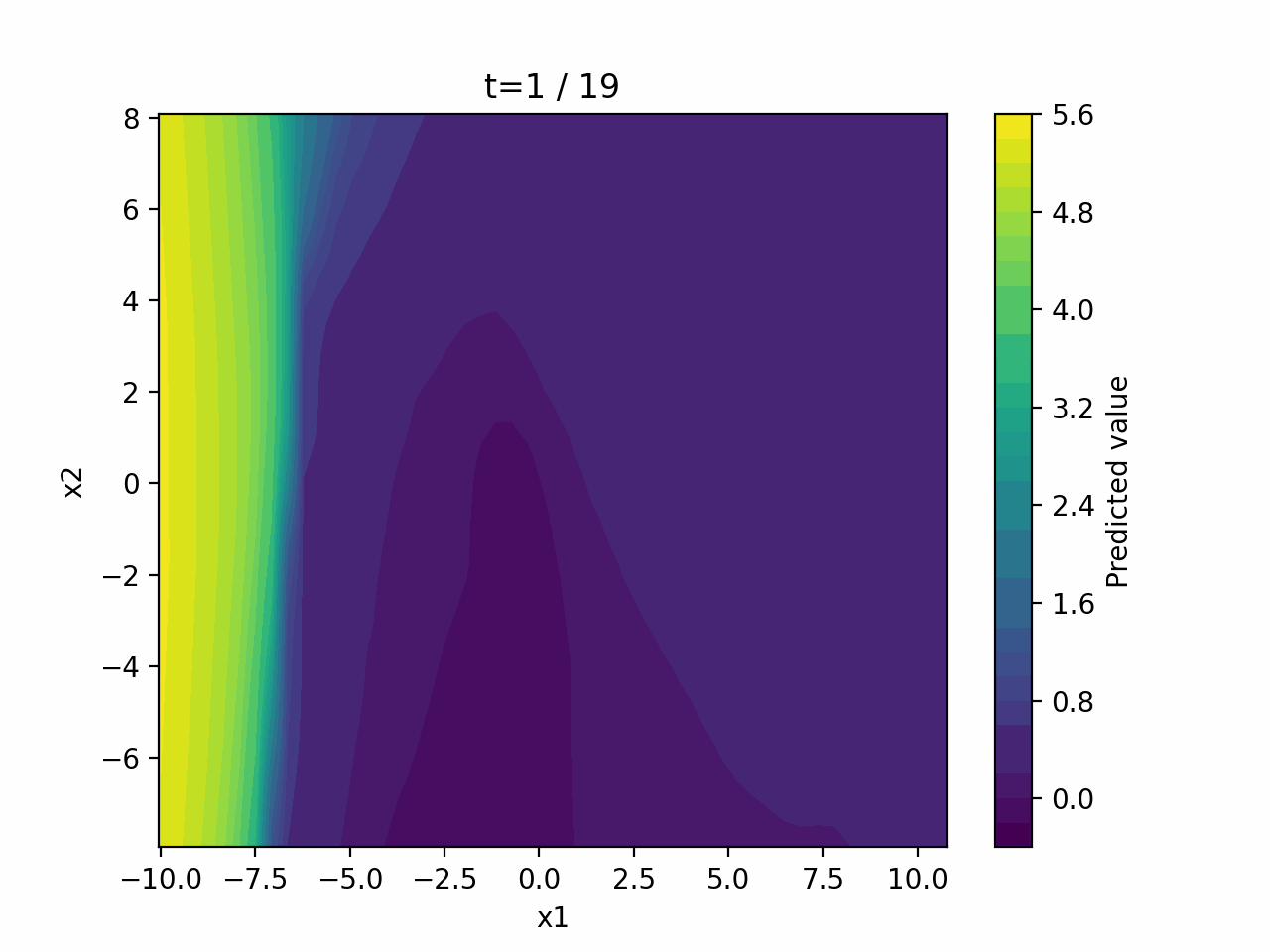} \\
\includegraphics[width=0.10\textwidth]{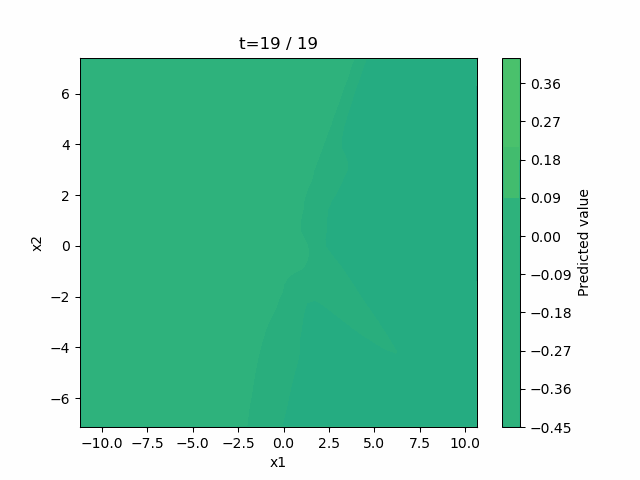} &
\includegraphics[width=0.10\textwidth]{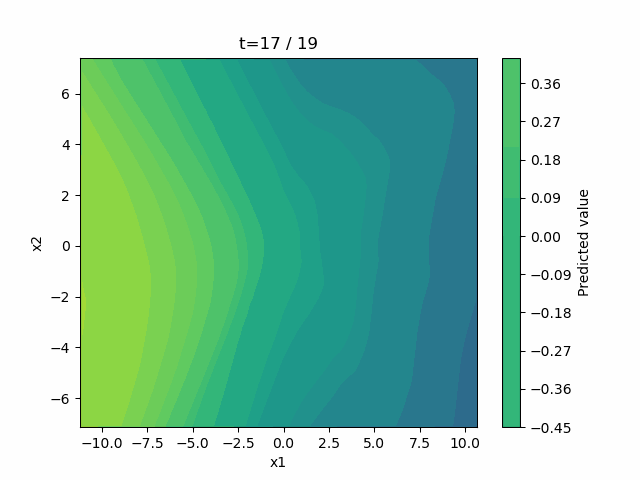} &
\includegraphics[width=0.10\textwidth]{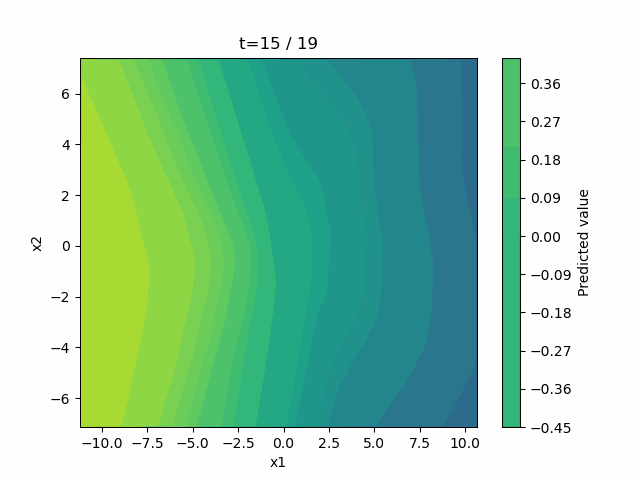} &
\includegraphics[width=0.10\textwidth]{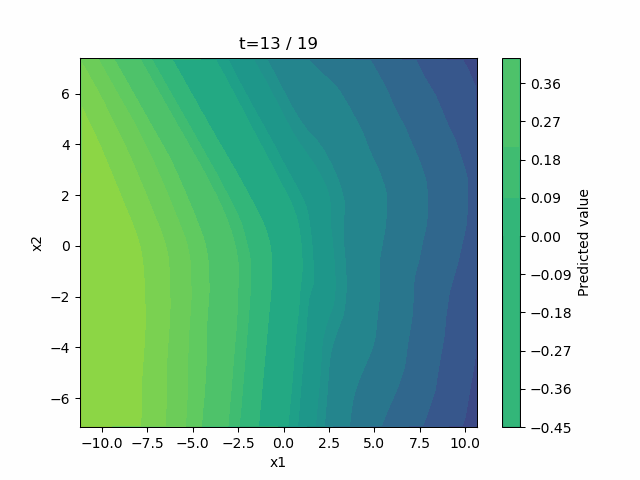} &
\includegraphics[width=0.10\textwidth]{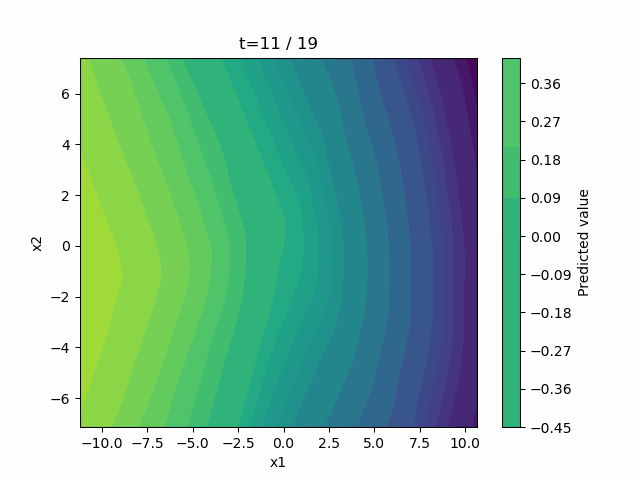} &
\includegraphics[width=0.10\textwidth]{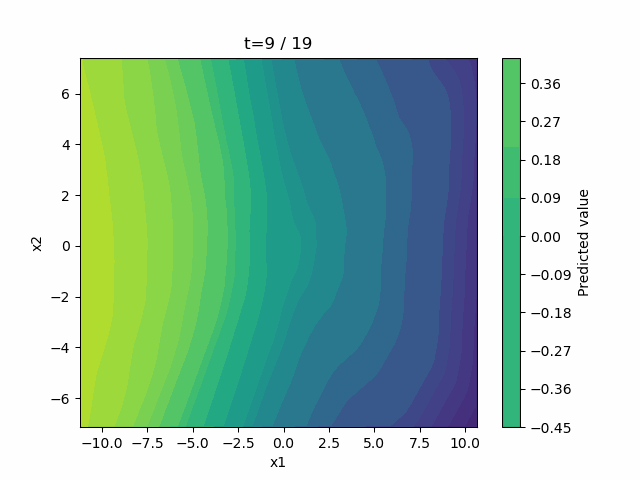} &
\includegraphics[width=0.10\textwidth]{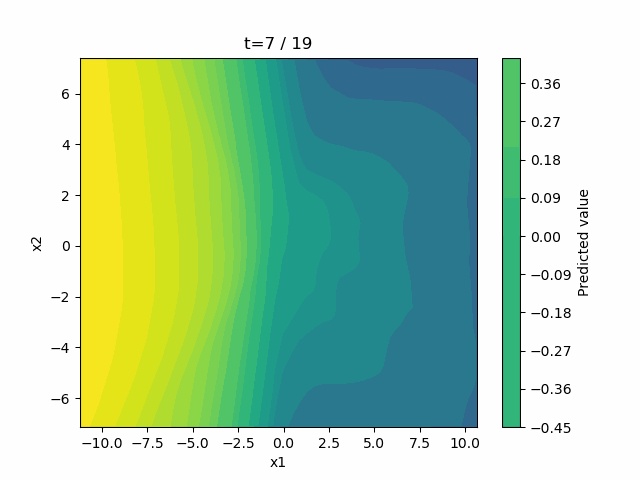} &
\includegraphics[width=0.10\textwidth]{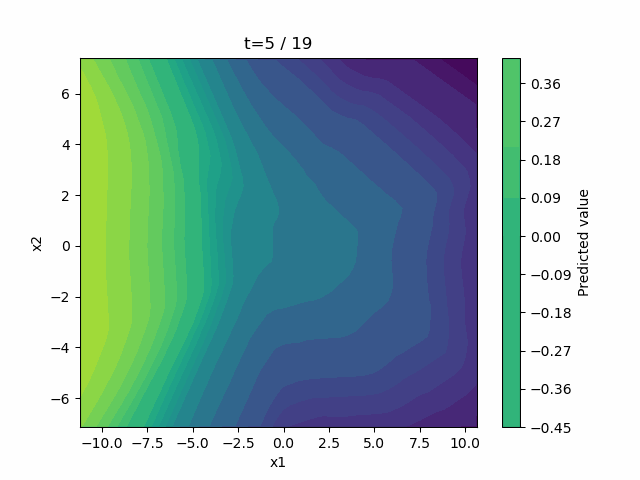} &
\includegraphics[width=0.10\textwidth]{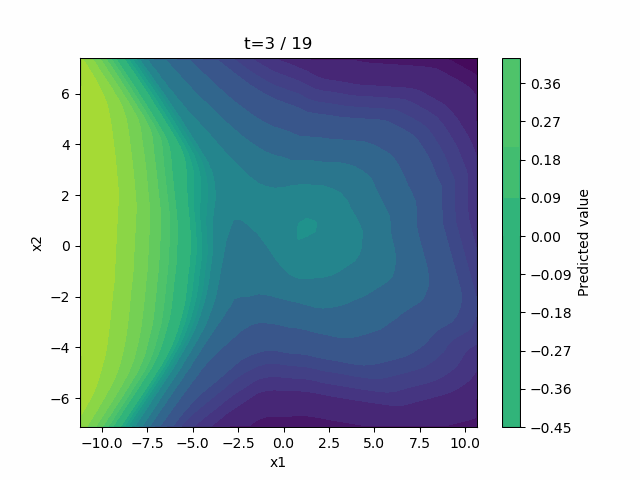} &
\includegraphics[width=0.10\textwidth]{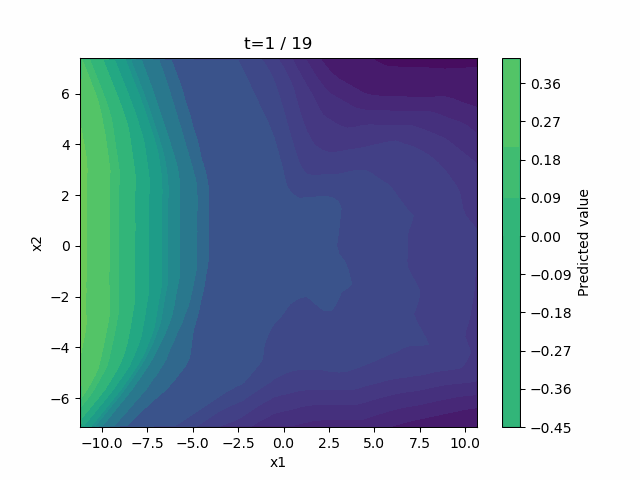}
\end{tabular}

{\scriptsize Top: $\hat{v}_t$ \qquad Bottom: $b_{\phi_t}$}

\caption{Gaussian mixture: Evaluations of estimated value and baselines over time, with every other frame shown. Noisiest timestep on the left, least noisy on the right.}
\label{fig:values_vs_baselines_filmstrip}
\end{figure}

\begin{figure}[t]
    \centering
    \begin{minipage}{0.32\textwidth}
        \centering
        \includegraphics[width=\textwidth]{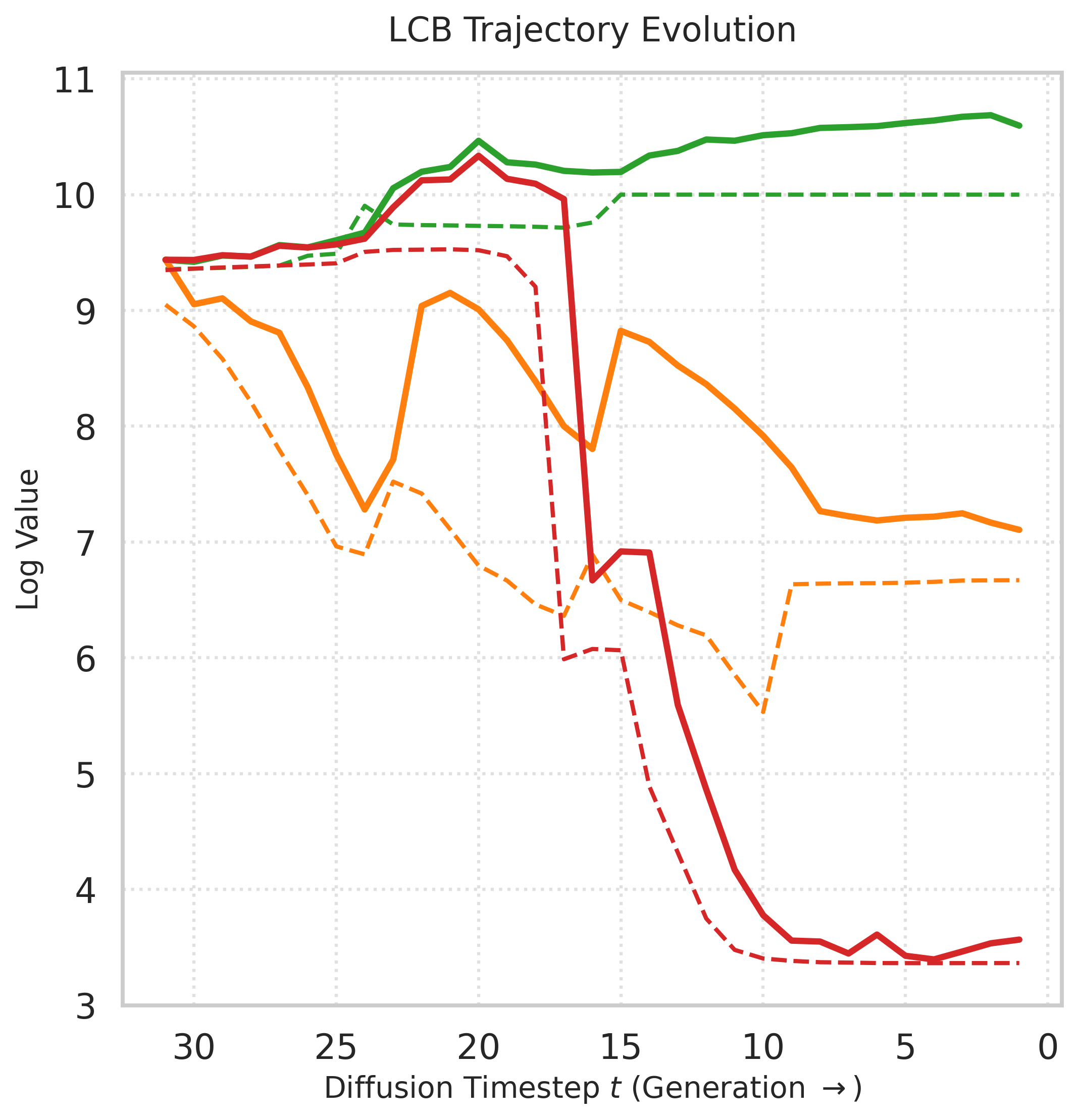}
        \subcaption{Temperature $\alpha = 0.1$}
    \end{minipage}\hfill
    \begin{minipage}{0.32\textwidth}
        \centering
        \includegraphics[width=\textwidth]{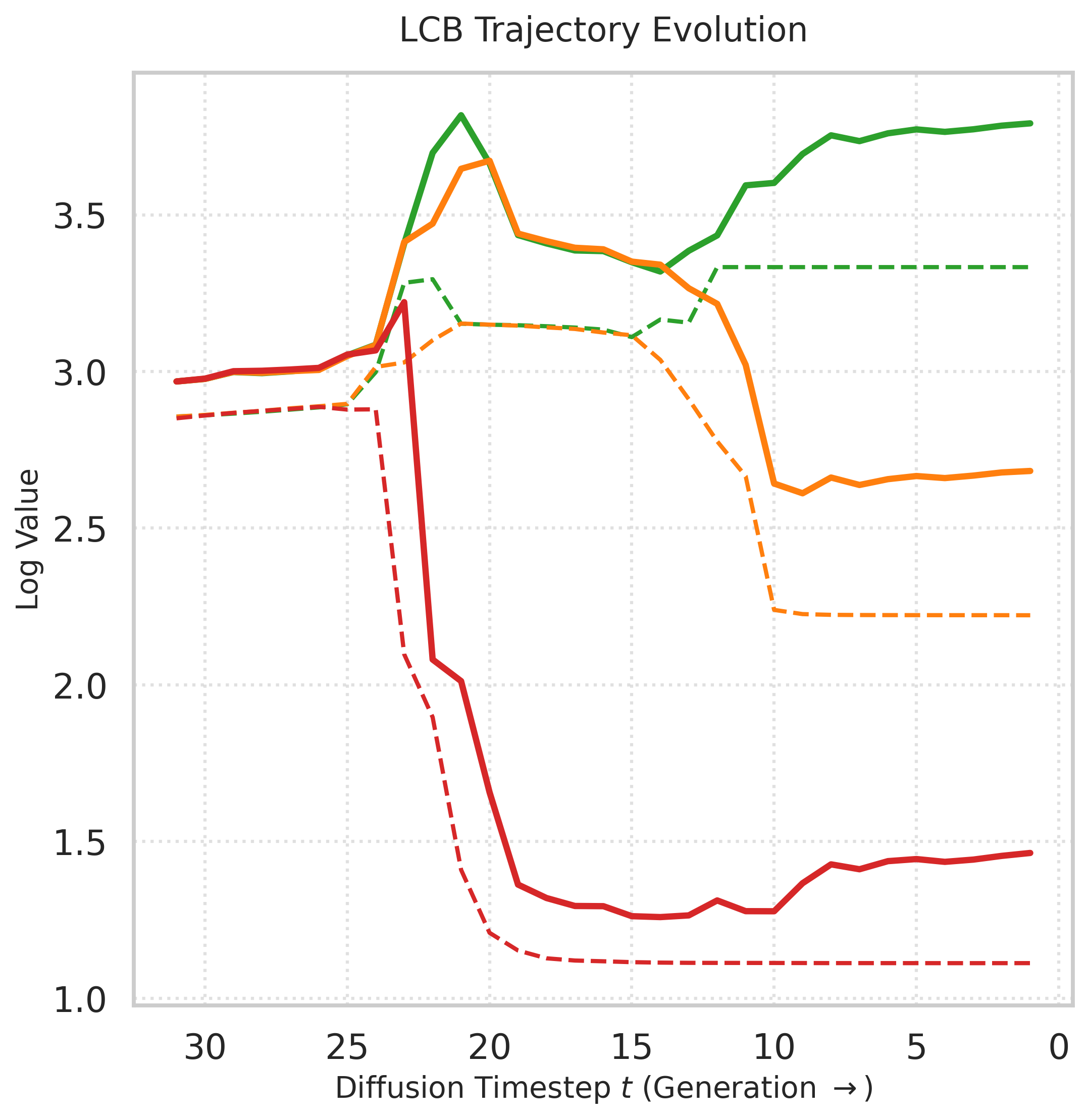}
        \subcaption{Temperature $\alpha = 0.3$}
    \end{minipage}\hfill
    \begin{minipage}{0.32\textwidth}
        \centering
        \includegraphics[width=\textwidth]{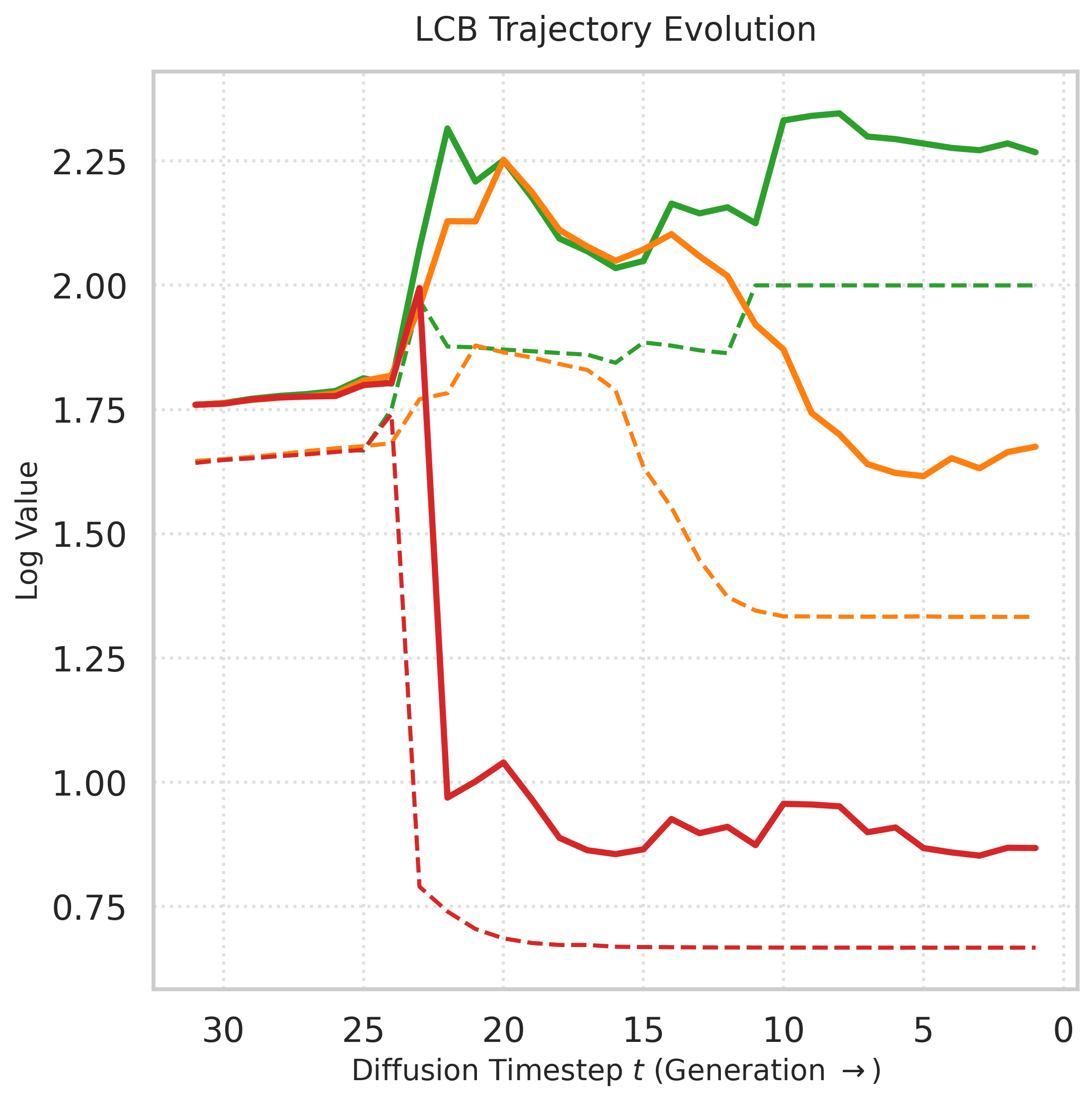}
        \subcaption{Temperature $\alpha = 0.5$}
    \end{minipage}
    \caption{%
    LLaDA: We unroll guided trajectories from the LCB-based sampler. We select a trajectory representing each reward $1/3$, $2/3$, or $1$, which are colored red, orange and green respectively. For these trajectories, we plot the evolution of the value function, $(t, \hat{v}_t(x_t))$ as a dotted line, and the LCB $\left(t, b_\phi(x_{t + 1}, t + 1) + \tau_{t + 1}\right)$ as a solid line. We see that the solid line is a fairly tight upper bound of the dotted line. Since the probability of acceptance is related to the difference between the dotted line and the solid line, and correctness is related to the robustness of the upper bound, this is desirable baseline behavior. We also see that these trajectories can cross, which corresponds to the fact that the value of a trajectory can change dramatically at any point in the masked diffusion setting. 
    }
    \label{fig:trajectories llada}
\end{figure}


\begin{figure}
    \centering
    \includegraphics[width=0.8\linewidth]{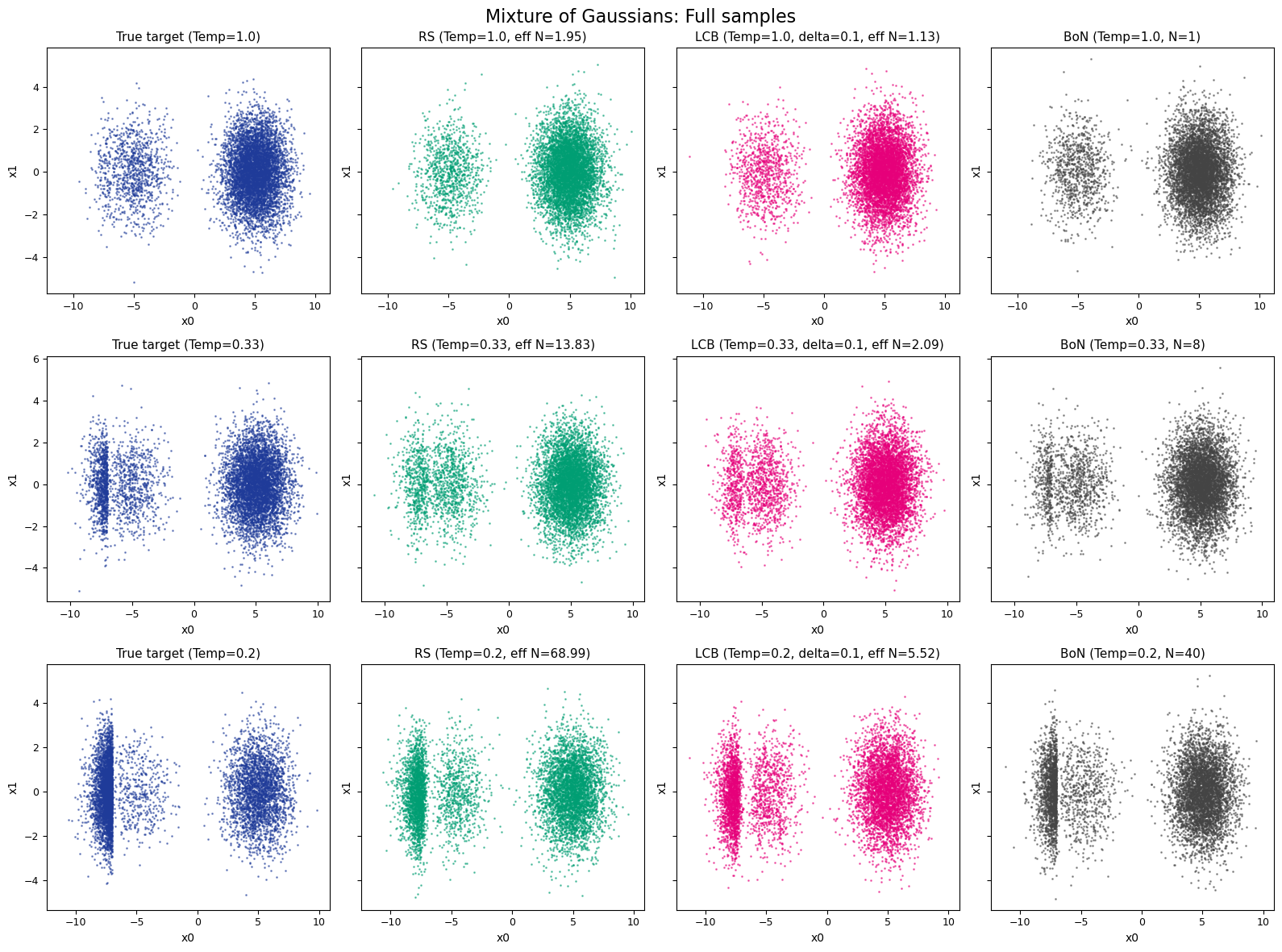}
    \includegraphics[width=0.8\linewidth]{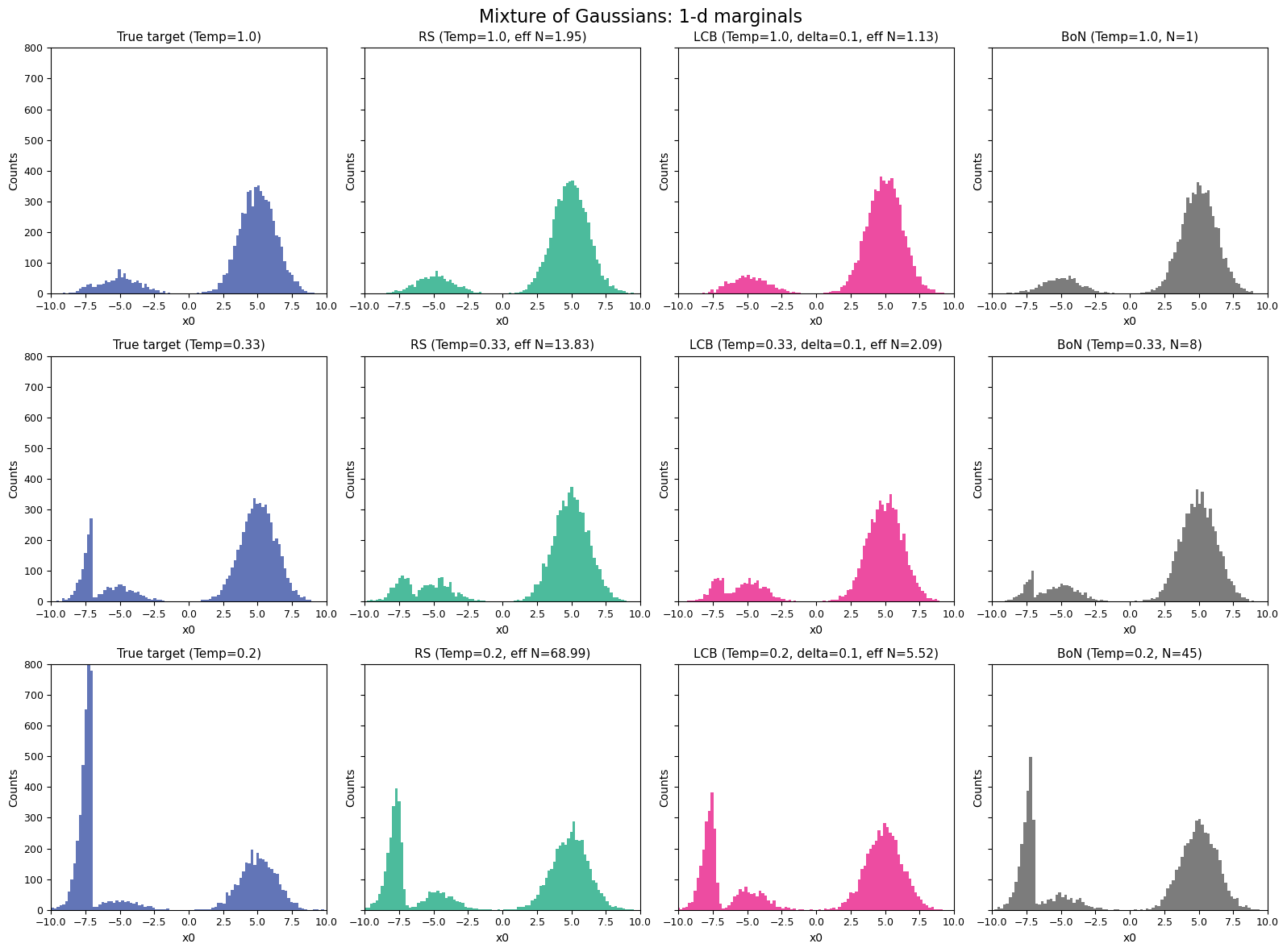}
    \caption{Gaussian mixture: 2-d scatter plots and $x$-axis histograms for the Gaussian Mixture using LCBs, Rejection Sampling, and Bo$N$ (with varying $N$ to simulate decreasing temperature). Vertical axis corresponds to temperature, horizontal axis to method. }\label{fig:apdx mog samples}
\end{figure}

\subsection{Setting II: LLaDA with Verifiable Rewards}\label{apdx:llada experiments}

\textbf{Methodology}

The pretrained model is LLaDA \citep{nie2025large}, a masked language model with 8 billion parameters. In masked language diffusion, our space $\mathcal{X}$ is the space of sequences of $L$ tokens, including a mask token $M \in \mathcal{X}$. In order to output $L$ tokens, LLaDA begins with a completely masked sentence $p_T = \delta_{M^L}$. LLaDA uses bidirectional attention to predict each token (in a block), and retains a random or structured subset of the predicted tokens (which defines $\ppre_t(\cdot | x_t)$). 

The pretrained model, $\ppre$, is defined by giving LLaDA the prompt

\begin{center}
    \begin{emph}``Write a creative story of exactly 3 sentences. Stylistically, it is preferred that the first word of each sentence start with 'A', 'B', and 'C' respectively.''\end{emph}
\end{center}

with $L = 32$ and sampling its output at high temperature ($1.12$).

We learn soft value functions by finetuning RoBERTa (\texttt{distilroberta-base}, 82 million parameters) with a simple prediction head and learned time embedding. We learn a classifier $p_\theta(c, x, t) \approx \Pr[r(x_0) = c | x_t = x]$, for $c \in \{0, 1/3, 2/3, 1\}$ using multinomial regression, and at any temperature $\alpha$ form the soft-value estimate
\[\begin{aligned} v_\theta(x, t, \alpha) = \log \sum_{c \in \{0, 1/3, 2/3, 1\}}p_\theta(c, x, t)e^{c/\alpha}\end{aligned}\]
where $p_\theta$ is trained on 5000 trajectories sampled from LLaDA. 

We now elaborate on the optimization procedure for the baseline function, which is also parameterized as RoBERTa with a time embedding. 

In order to mitigate training instability and minimize engineering complexity, the baseline network is pretrained to learn $b_{\phi}(x_{t + 1}, t ) \approx \E[\hat{v}_t(x_t) | x_{t + 1}]$ using data $x_{t + 1}, x_t \sim \ppre(x_{t + 1}, x_t)$ and an MSE objective. We run Algorithm \ref{alg:one pass lcb training} with $m = 350$ particles only for setting $\lambda \in [1, 6]$. This allows us to mitigate the catastrophic forgetting phenomenon that was encountered in the multivariate Gaussian setting, without resorting to replay buffers or other complicated training techniques. The pretraining of $b_{\phi}(x, t)$ uses 2500 trajectories (separately drawn from the data used to train the value function). Training both $\phi$ and $\lambda$ for a single temperature takes about $15$ minutes on an A100 (High Ram) Google Colab notebook. 

When performing inference, we clip $B_{t + 1}(x_{t + 1}) = b_{\phi}(x_{t + 1}, t + 1) + \tau_{t + 1}$ above at $1/\alpha$, which is a uniform upper bound on the soft value function.

\begin{table}[!t]
\centering
\small 
\setlength{\tabcolsep}{3.5pt} 
\renewcommand{\arraystretch}{0.9} 
\begin{tabular}{@{}lccc@{}} 
\toprule
\textbf{Method} & \textbf{Rew.} & \textbf{Perf. Rate (\%)} & \textbf{Prop.} \\
\midrule
RS 0.5 & $0.822 \pm 0.023$ & $64.40 \pm 4.20$ & $49.9 \pm 2.0$ \\
RS 0.3 & $0.863 \pm 0.020$ & $70.60 \pm 3.99$ & $63.3 \pm 4.3$ \\
RS 0.1 & $0.923 \pm 0.016$ & $83.20 \pm 3.28$ & $403.2 \pm 115.9$ \\
\midrule
LCB 0.5 & $0.827 \pm 0.023$ & $65.00 \pm 4.18$ & $37.2 \pm 0.3$ \\
LCB 0.3 & $0.885 \pm 0.021$ & $74.2 \pm 3.84$ & $38.7 \pm 0.5$ \\
LCB 0.1 & $0.917 \pm 0.017$ & $81.80 \pm 3.38$ & $51.0 \pm 1.8$ \\
\midrule
Best of 2 & $0.893 \pm 0.017$ & $74.60 \pm 3.82$ & $64.0 \pm 0.0$ \\
Unguided & $0.766 \pm 0.017$ & $52.00 \pm 3.10$ & $32.0 \pm 0.0$ \\
\bottomrule
\end{tabular}
\caption{Table from Figure  2 repeated: comparison of policies at temps $0.5, 0.3, 0.1$. Average reward ($95\%$ Gaussian CI), average rate of ``perfect samples'' (those with reward of 1) ($95 \%$ Bernoulli CI), and average number of proposals ($95\%$ Gaussian CI).}\label{tab:method_comparison}
\end{table}

\begin{figure}
    \centering
    \includegraphics[width=0.6\linewidth]{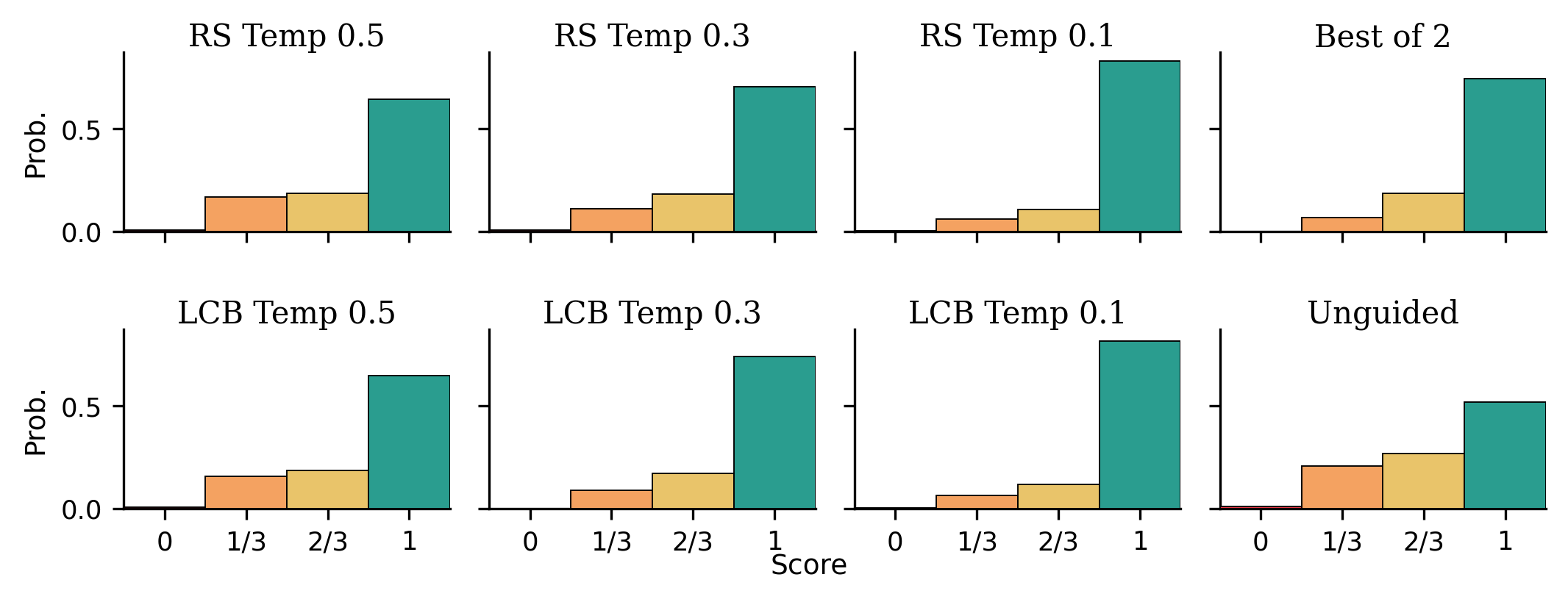}
    \caption{Reward histograms induced by each policy.}
    \label{fig:setting II reward hist}
\end{figure}

\textbf{Discussion}

Due to the nature of textual data, we cannot exactly quantify the distances between the sampling distributions of each policy. However, we now study their reward statistics. The next section includes various text quality metrics for each policy and finds no significant differences. 

Comparing rejection sampling (RS) and LCB sampling (LCB) for each temperature $\alpha \in \{0.5, 0.3, 0.1\}$ in Figure \ref{fig:setting II reward hist} and Table \ref{tab:method_comparison}, we see that all functionals of the reward observed are within statistical distances of each other, which suggests that the LCB is correctly targeting the distribution $\hat{p}(x_0 \dots x_T)$. 
Moreover, LCB reduces the proposal complexity of RS by $25-87\%$, depending on the temperature. At higher temperatures like 0.5, LCB uses $16.3\%$ more queries to LLaDA than using no guidance whatsoever, while still substantially boosting reward metrics.  

The reward statistics of Best-of-2 (Bo2) are sandwiched in between LCB 0.3 and LCB 0.1, but Bo2 requires substantially more queries to LLaDA. LCB 0.1 achieves a slightly higher average reward and substantially higher rate of perfect outputs while using $20\%$ fewer proposals (64 versus 51). LCB 0.3 achieves slightly lower rewards than Bo2, but uses $40\%$ (64 versus 38.7) fewer queries to the base model. 

On the left hand side of Figure \ref{fig:success_efficiency}, the CI for ``effective $N$'' is obtained by dividing the Gaussian CI for average proposals by $32$.

\textbf{Additional Results for LCB and BoN}

In Figure \ref{fig:trajectories llada}, we visualize the upper bounds provided by the LCBs against the value over time for a representative set of guided trajectories, in order to show that the baseline is a tight and correct upper bound for the value. 

In Figure \ref{fig:text stats}, we show that a number of textual quality metrics remain within statistical error across all policies considered, which suggests that boosting the reward leaves many aspects of text generation intact. 

\begin{figure}[!ht]
    \centering
    \includegraphics[width=\linewidth]{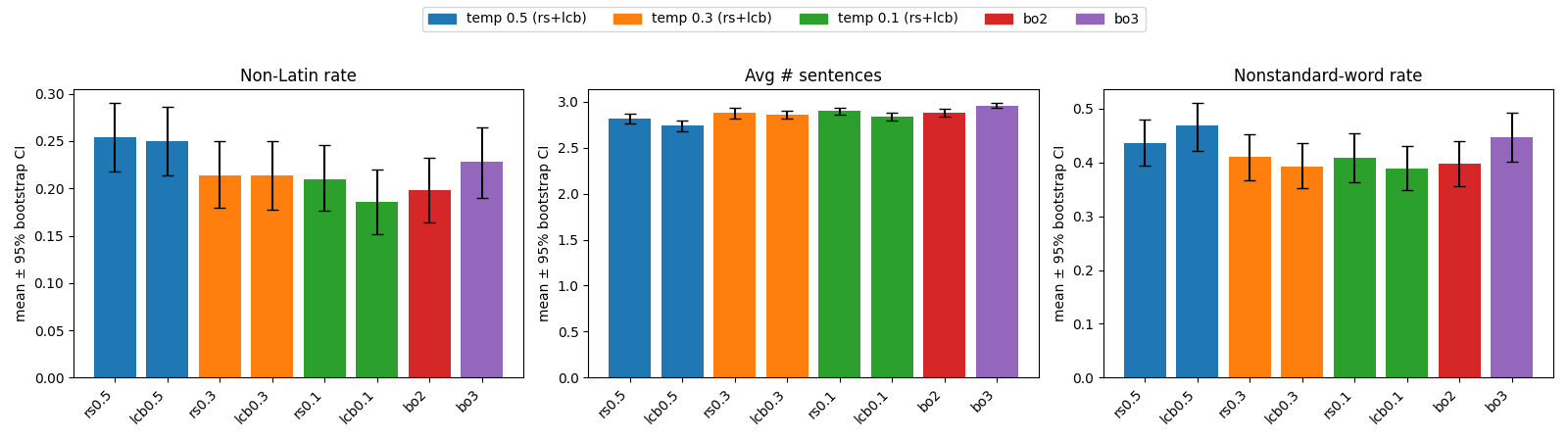}
    \caption{We ensure that guidance does not substantially increase the rate of non-standard words (words not in the dictionary) or non-Latin words (of which there are a good deal, due to our high sampling temperature). Average number of sentences generally trends weakly upward as temperature decreases, which is to be expected, but does not vary statistically between RS and LCB at any temperature.}
    \label{fig:text stats}
\end{figure}

\subsubsection{Comparison with SMC}

In this appendix, we provide additional experiments to offer a detailed comparison between LCB and SMC methods for Setting II. 

As a representative SMC algorithm, we consider Algorithm 4 from \cite{uehara2025tutorial}. This algorithm evolves a set of $K$ particles $x^1_t \dots x^K_t$ in parallel, proposing $x_t^i \sim \ppre_t(\cdot\mid x_{t+1}^i)$ and updating cumulative weights as $w_t^i = w_{t+1}^i\exp(\hat v_t(x_t^i)-\hat v_{t+1}(x_{t+1}^i))$ for $i\in[K]$ (with weights reset to uniform after resampling). For steps $t$ where ``effective sample size'', $\frac{(\sum_{i = 1}^{K} w^i_t)^2}{\sum_{i = 1}^{K} (w^i_t)^2}$ falls below a threshold, ESS (taken to be a user hyperparameter), the particles are resampled, so that $ x_t^i=x_t^{\sigma_i}$ for independent $\sigma_i\sim\mathrm{Cat}(\{w_t^j/\sum_{\ell=1}^K w_t^\ell\}_{j=1}^K)$. Then we proceed by evolving the particles as $x^i_{t - 1} \sim \ppre_{t-1}(\cdot\mid x_t^i)$ for $i \in [K]$. 

\textbf{We emphasize that, like LCBs, SMC is designed in order to sample from $\hat{p}(x_0)$.}

We run SMC in Setting II with the same value function used by LCBs. In this section, we only consider $\alpha = 0.1$, to compare the strong alignment setting. We perform extensive sweeps over $K$, and also sweep over ESS. For each $K$ and ESS, we run SMC $M = 5$ times. 

Notice that, in the setting where SMC is queried asynchronously or with prompts, only one of the $K$ particles can ultimately be used. We also show that even in the unprompted setting, the resampling step has the effect that the suite of $K$ final particles loses diversity, as measured by unique 4-grams in the outputs. \textbf{Thus, even in the unprompted setting, the $K$ trajectories unrolled only reliably give us access to a single, high-quality response (or, at best, a small number thereof).} Figure \ref{fig:apdx smc diversity-side-by-side} shows that in order to boost reward, SMC must reduce diversity.

\begin{figure}[t]
    \centering

    \begin{subfigure}{0.48\textwidth}
        \centering
        \includegraphics[width=\linewidth]{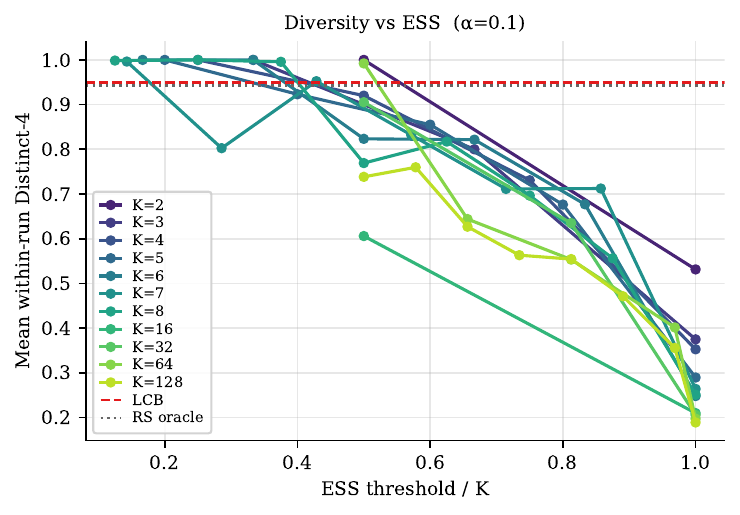}
        \caption{Diversity vs ESS}
        \label{fig:diversity-ess}
    \end{subfigure}
    \hfill
    \begin{subfigure}{0.48\textwidth}
        \centering
        \includegraphics[width=\linewidth]{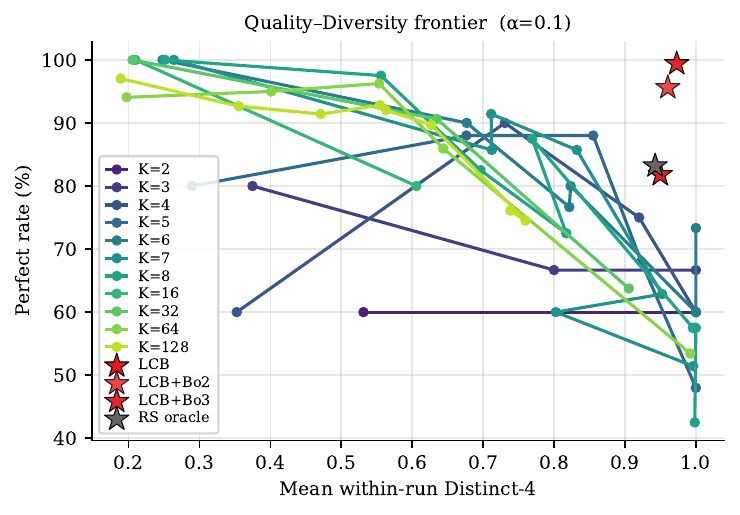}
        \caption{Quality--Diversity frontier}
        \label{fig:diversity-perfrate}
    \end{subfigure}

    \caption{High rewards in SMC require raising ESS, which destroys diversity. Even in the unconditioned setting, this hinders our ability to use all $K$ particles in downstream tasks. In the above plots, we try many ESS for each $K$. For fixed $K$, the line is created by varying ESS. }
    \label{fig:apdx smc diversity-side-by-side}
\end{figure}

Due to these considerations, we treat $K$ as the effective $N$ for SMC methods. Thus, from each system of $K$ particles, we uniformly sample a single output to treat as SMC's output.   Under this convention, we produce compute scaling plots of mean reward versus effective N. The cost of resampling at the ESS threshold is low, and so we consider it negligible. See Figure \ref{fig:apdx smc inference scaling}.

\begin{figure}
    \centering
    \includegraphics[width=0.7\linewidth]{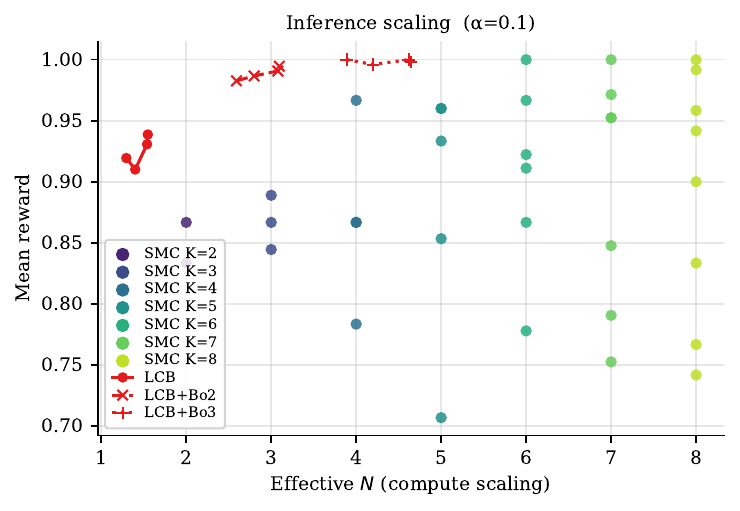}
    \caption{LCB+Bo$M$ saturates the reward faster than SMC. SMC: for each $K$, we try every ESS of size at most $K$. Rewards for SMC are estimated by an average over all $M\times K$ outputs for each $K$ and ESS.}
    \label{fig:apdx smc inference scaling}
\end{figure}

As mentioned SMC -- like LCBs -- is designed to sample from $\hat{p}(x_0)$. That is, SMC is also implicitly designed to avoid reward hacking using KL penalization. This means that, alongside evaluation of how efficiently SMC improves reward while scaling inference compute, we also must evaluate how effectively SMC fulfills its primary objective of targeting $\hat{p}(x_0)$ compared to LCBs.

Though calculating TV distance is intractable, we can use the induced reward distribution as a helpful surrogate. This is possible in Setting II, because rewards are discrete. For each $K$ and ESS, we form the reward PMF from all $MK$ samples. Once this pmf is formed, we may calculate the TV distance between it and the one induced by rejection sampling (which produces perfect samples from $\hat{p}(x_0)$). This gives us the quantities $TV_{K, ESS}$ for many $K$ and $ESS$. \textbf{For each $K$, we report the best TV distance over all ESS, in order to give a best-case view of SMC.} For each $K$ we form the pmf as in Figure \ref{fig:setting II reward hist}, using only $MK$ samples. and simply calculate TV distance.  In Figure \ref{fig:smc tv distance}, we show that LCB is more accurate than $SMC(K, ESS)$ for every setting of $K$ and ESS. 

\begin{figure}
    \centering
    \includegraphics[width=0.7\linewidth]{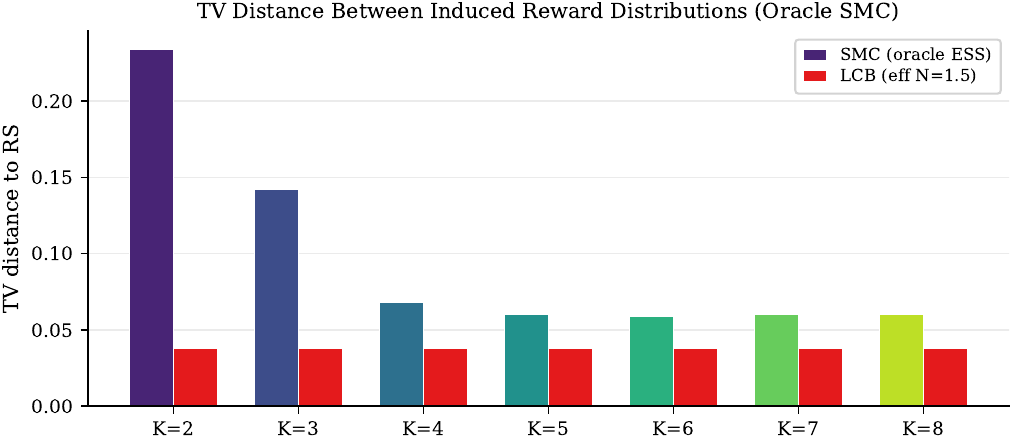}
    \caption{LCB produces more accurate samples from $\hat{p}(x_0)$ at a fraction of the cost. For each $K$, we take the oracle TV-minimizing ESS. We match number of samples between SMC and LCB at each $K$ to ensure fair TV estimation.}
    \label{fig:smc tv distance}
\end{figure}

\subsection{Setting III: LLaDA with Cross-prompt Generalization}\label{apdx:setting III prompted llada}

\textbf{Methodology}

We create 512 prompts of the form \emph{``Complete the sentence in a polite but firm tone: ``Thank you for your $\langle s\rangle$, but...''''}, where $s$ is from a set of ``stems'' that require very different completions (e.g. “bid,” “confidential tip,” etc.). LLaDA's output is then scored by Skywork-Reward-V2-Qwen3-1.7B (SkyR), which defines $r$.

Because each stem $s$ leads to a different reward scale under SkyR, we train $\hat{v}_t(x_t, \mathrm{prompt})$ on normalized rewards. Specifically, letting $\mu_0(\mathrm{prompt})$ and $\sigma_0(\mathrm{prompt})$ be the mean and standard deviation of $\ppre(r(x_0) | \mathrm{prompt})$, we define:
\[Z(x_0 | \mathrm{prompt}) := \frac{r(x_0) - \mu_0(\mathrm{prompt})}{\sigma_0(\mathrm{prompt)}}.\]

We then finetune the same distill-RoBERTa-82m model as in Setting II, prepending the prompt to the input. We train $\hat{v}_t(x_t, \mathrm{prompt})$ with square loss to predict the Z-score of the denoised output. 

\[\E_{x_t, x_0, t \sim unif([T])}[\left(\hat{v}_t(x_t, \mathrm{prompt}) -Z(x_0 | \mathrm{prompt}))^2\right)].\]

Notice that we ignore the softmax transform, and model conditional expectation directly. 

We train on $410$ prompts, with $50$ completions for each prompt. All evaluations are performed on a set of $102$ held-out prompts. 

While training the LCB, we let $b_{t + 1}(x_{t + 1}) := v_{t + 1}(x_{t + 1})$. This means learning the LCB reduces to learning $\lambda_{t + 1}, \tau_{t + 1}$, which is a simple $1$d unimodal optimization problem. \textbf{Therefore, the LCB carries essentially no training cost apart from learning $\hat{v}_t$.} We use $\delta = 0.5$, and did not conduct any parameter sweep over $\delta$.

\textbf{Additional results}

\begin{table}[h]
\centering
\begin{tabular}{lccc}
\toprule
 & Mean Z & Mean Reward & Eff N \\
\midrule
LCB ($\alpha = 0.45$) & $0.5003 \pm 0.02$ & $3.8458 \pm 0.16$ & $1.7530$ \\
LCB ($\alpha = 0.5$)  & $0.4681 \pm 0.03$ & $3.7609 \pm 0.16$ & $1.6530$ \\
LCB ($\alpha = 1.0$)  & $0.3585 \pm 0.02$ & $3.4747 \pm 0.16$ & $1.2968$ \\
LCB ($\alpha = 2.0$)  & $0.2281 \pm 0.03$ & $3.1345 \pm 0.17$ & $1.1435$ \\
LCB ($\alpha = 4.0$)  & $0.1364 \pm 0.03$ & $2.8942 \pm 0.17$ & $1.0732$ \\
Unguided & $0.0000 \pm 0.00$ & $2.5156 \pm 0.18$ & $1.0000$ \\
Bo2 & $0.5453 \pm 0.01$ & $3.9569 \pm 0.16$ & $2.0000$ \\
\bottomrule
\end{tabular}
\caption{Evaluation of LCBs on unseen prompts. For each $\alpha$, and for each of 102 held-out prompts, we compute average Z-score, reward, and effective $N$ over 100 completions, leading to an overall set of $10,200$ inferences. We average these quantities over prompts to get scalar summaries for each method. Gaussian CI $95\%$ applied to per-prompt averages.}
\label{tab:lcb_results}
\end{table}

\begin{figure}
    \centering
    \includegraphics[width=\linewidth]{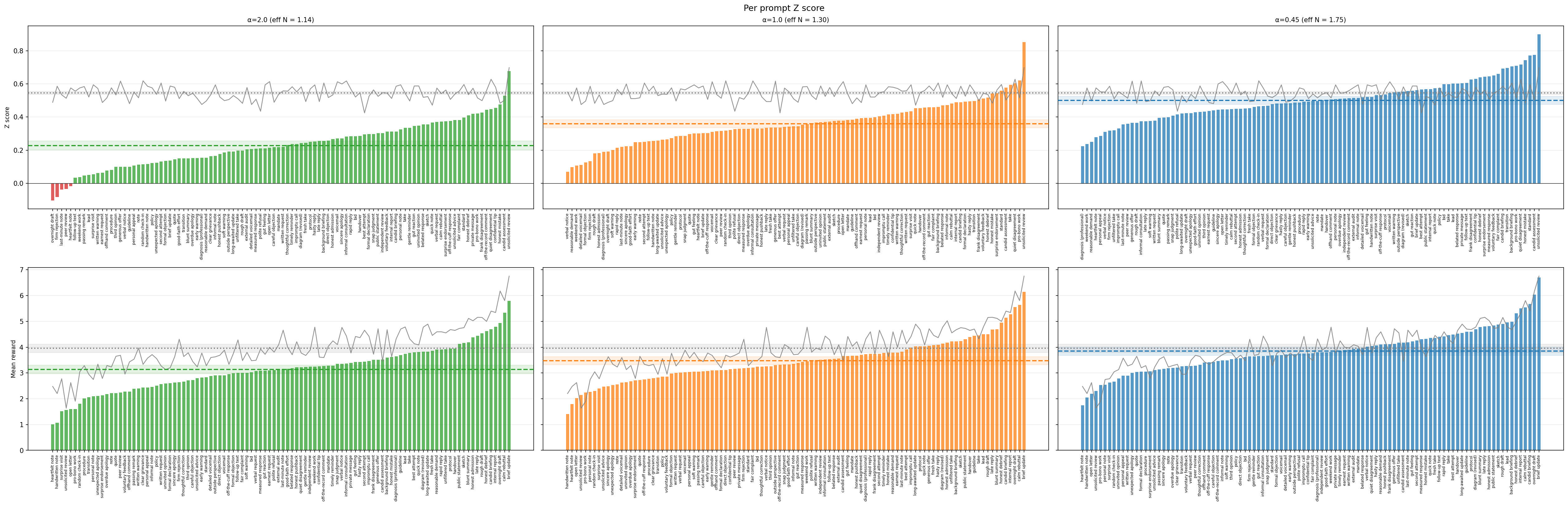}
    \caption{We arrange prompts on the $x$ axis and average Z-scores (top) or mean rewards (bottom) on the y axis. LCBs with $\alpha = 2, 1, 0.45$ are given in colors, with Bo2 in gray. $\alpha = 0.45$ has comparable reward to Bo2, with effective $N$ of 1.75.}
    \label{fig:per prompt Z score}
\end{figure}

\begin{figure}
    \centering
    \includegraphics[width=0.5\linewidth]{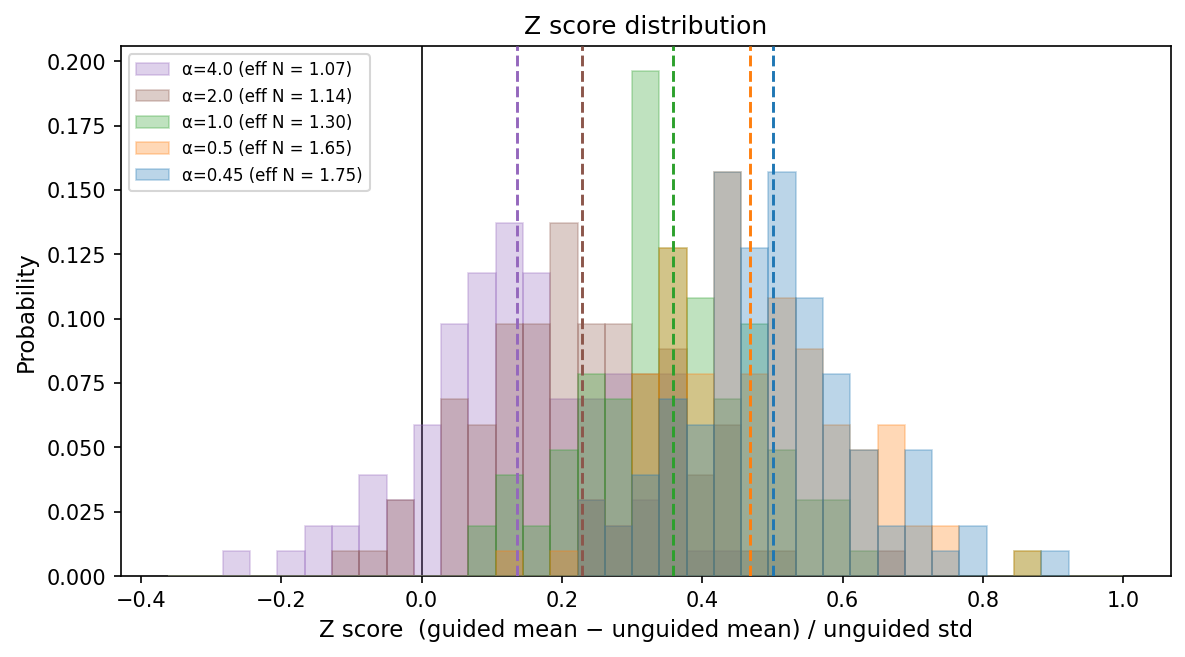}
    \caption{Histograms of $Z$ scores induced by a uniform distribution over prompts. This pmf induces the survival function in Figure \ref{fig:setting-iii-survival-scaling}.}
    \label{fig:placeholder}
\end{figure}

\section{Relaxed Coverage Assumption}\label{apdx:coverage}

In Section \ref{sec:finegrained tv theory} theory, we introduce a worst-case coverage assumption, which is used in Theorem \ref{thm:lcb tv bound}, Corollary \ref{cor:subgaussian value corollary}, and Theorem \ref{thm:mog thm}. In this Appendix, we weaken it to a condition that can be empirically verified.

As explained in the body of the paper (Section \ref{sec:finegrained tv theory}), we say a joint baseline $B_{t + 1}$ has \textbf{worst-case coverage at level $c \in (0,1 )$} (wrt $\hat{v}_t(x_t)$) if there exists $c \in (0, 1)$ such that $\forall x_{t + 1}$, $B_{t + 1}$ satisfies
\begin{equation}\label{eq:worst-case coverage assumption}
\Pr_{x_t \sim \ppre_t(\cdot | x_{t + 1})}[\hat v_t(x_t) \geq B_{t + 1}(x_{t + 1})] \leq c
\end{equation}

Importantly, we do not require that $c \to 0$ in order to prove that sampling error vanishes. It suffices for $c$ to be a constant like $\frac{1}{2}$ or $\frac{2}{3}$, which is a natural circumstance if $B_{t+1}(x_{t+1}) \geq \E_{x_t \sim \ppre_t(\cdot | x_{t + 1})}[\hat v_t(x_t) | x_{t+1}]$  $x_{t+1}$-a.s and $\hat v_t(x_t) | x_{t+1}$ is roughly symmetric. Moreover, we can track the assumption's effect through our bounds even when $c = 1 - \epsilon$ for some tiny $\epsilon > 0$. 

Nevertheless, rather than assuming that the condition in \eqref{eq:worst-case coverage assumption} holds for $x_{t + 1}$, it is more realistic to assume that the condition holds \emph{with high probability} $\zeta \in (0, 1)$ with respect to the sampling distribution $\hat{q}_{t + 1}(x_{t + 1})$. In particular, we define the set of ``covered states'' at level $c \in (0, 1)$ to be 

\[\mathcal {C}_{t + 1}(c) = \{x_{t + 1} \ : \Pr_{x_t \sim \ppre_t(\cdot | x_{t + 1})}\left[\hat{v}_t(x_t) \geq B_{t + 1}(x_{t + 1})\right] \leq c\} \]

We then say that a baseline has \textbf{high-probability coverage at levels $(\zeta, c) \in (0, 1)\times (0, 1)$} if 

\[\Pr_{x_{t + 1} \sim \hat{q}_{t+ 1}}\left[x_{t + 1} \in \mathcal{C}_{t + 1}(c)\right] \geq 1 - \zeta\]

In other words, for a good set of $x_{t + 1}$ with measure $1 - \zeta$ under the sampling distribution, $B_{t + 1}(x_{t + 1})$ upper bounds $\hat{v}_t(x_t)$ with probability $c$ over next-state proposals $x_t$. When $c \in (0, 1)$ has been fixed, we will omit the dependence of $\mathcal C_{t + 1}$ on $c$. 

In the next section of the Appendix, we show that, assuming high-probability coverage at levels $(\zeta, c)$, is sufficient to reproduce the results of Theorem \ref{thm:lcb tv bound} (and all of its corollaries), at the cost of an additive $\zeta$ term.

\subsection{Theory extension}

We show how to weaken the worst-case coverage assumption in the TV:MGF lemma
(Lemmas~\ref{lem:tv mgf body},~\ref{lem:tv mgf apdx}) to a high-probability coverage assumption.

\begin{lemma}[TV:MGF with high-probability coverage]\label{lem:apdx weakened tv mgf lemma}
Assume all conditions of Lemma~\ref{lem:tv mgf body} except worst-case coverage.
Instead, assume $B_{t+1}$ has high-probability coverage at levels $(\zeta,c)\in(0,1)\times(0,1)$
under the sampling distribution $\hat q_{t+1}$, i.e.
$\Pr_{x_{t+1}\sim \hat q_{t+1}}[x_{t+1}\in \mathcal C_{t+1}(c)]\ge 1-\zeta$.

Then for $\hat q_t(\cdot\mid x_{t+1})$ induced by Algorithm~\ref{alg:baseline rejection sampling},
\begin{equation}\label{eq:relaxed tv:mgf expression}
\E_{x_{t + 1} \sim q_{t + 1}}[d_{TV}(\hat{q}_t(\cdot | x_{t + 1}), \hat{p}_t(\cdot | x_{t + 1}))]  \leq \frac{2\delta^{1 - 2/\lambda}}{(1 - c)^2}(M(\lambda)M(-\lambda))^{1/\lambda}  + \zeta
\end{equation}
\end{lemma}

\begin{proof}[Proof sketch]
Let $\mathcal C_{t+1}=\mathcal C_{t+1}(c)$ be the covered set.
Decompose
\begin{align*}
&\mathbb E_{x_{t+1}\sim \hat q_{t+1}}\!\left[d_{{TV}}\!\left(\hat q_t(\cdot\mid x_{t+1}),\hat p_t(\cdot\mid x_{t+1})\right)\right]\\
&=
\mathbb E\!\left[d_{{TV}}(\cdot)\, 1\{x_{t+1}\in \mathcal C_{t+1}\}\right]
+
\mathbb E\!\left[d_{{TV}}(\cdot)\, 1\{x_{t+1}\notin \mathcal C_{t+1}\}\right] \\
&\le
\mathbb E\!\left[d_{{TV}}(\cdot)\, 1\{x_{t+1}\in \mathcal C_{t+1}\}\right]
+
\Pr[x_{t+1}\notin \mathcal C_{t+1}] \\
&\le
\mathbb E\!\left[d_{{TV}}(\cdot)\, 1\{x_{t+1}\in \mathcal C_{t+1}\}\right]
+\zeta,
\end{align*}
using $d_{{TV}}(\cdot)\le 1$.

It remains to bound the first term. Condition on $x_{t+1}\in \mathcal C_{t+1}$ and replay the proof of
Lemma~\ref{lem:tv mgf apdx} (via the Exceedance Error Lemma~\ref{lem:excedence error} and its decomposition
$d_{\mathrm{TV}}(\cdot)\le A_{x_{t+1}}+B_{x_{t+1}}$).
The bound on $B_{x_{t+1}}$ is unchanged, since it does not use coverage.
For the $A_{x_{t+1}}$ term, the only place coverage is used in Lemma~\ref{lem:tv mgf apdx} is to lower bound
\[
\ppre_t(E\mid x_{t+1}) \quad \text{where } \quad
E=\{\hat v_t(x_t)\le B_{t+1}(x_{t+1})\}.
\]
On the covered set $\mathcal C_{t+1}(c)$ we have $\ppre_t(E\mid x_{t+1})\ge 1-c$, hence
$\ppre_t(E\mid x_{t+1})^{-2}\le (1-c)^{-2}$, and the remainder of the TV:MGF argument goes through verbatim,
yielding
\[
\mathbb E\!\left[d_{{TV}}(\cdot)\,\mathbf 1\{x_{t+1}\in \mathcal C_{t+1}\}\right]
\le
\frac{2\,\delta^{1-2/\lambda}}{(1-c)^2}\big(M(\lambda)M(-\lambda)\big)^{1/\lambda}.
\]

Substituting into the first expression establishes \eqref{eq:relaxed tv:mgf expression}

\end{proof}

With Lemma \ref{lem:apdx weakened tv mgf lemma} proven, we can replace worst-case coverage in Theorem \ref{thm:general baseline tv bound}, Corollary \ref{cor:subgaussian value corollary} and Theorem \ref{thm:mog thm} by following the proofs of \ref{apdx:lcb main thm proofs} and adding $\zeta$ when TV:MGF is applied. For completeness, we restate the main theorem with high-probability coverage. 

\begin{theorem}[Theorem \ref{thm:lcb tv bound}]
Let $\hat \lambda, \hat b$ be the ERMs of $\hat{J}$. Let $\hat B^{(\hat \lambda,\hat b)}_{t + 1} := \hat{b}_{t + 1}(x_{t + 1}) + \hat{\tau}_{\hat {\lambda}, \hat{b}_{t + 1}} + \epsilon_0$ be the associated LCB, and assume that it is also satisfies high-probability coverage at levels $(\zeta, c)$. Then, on the good training event from Proposition \ref{prop:lcb learning bounds} 
\[\textstyle \E_{x_{t + 1} \sim \hat{q}_{t + 1}}[d_{TV}(\hat{q}_{t}(\cdot | x_{t + 1}), \hat{p}(\cdot | x_{t + 1}))] \leq \frac{2\delta}{(1 - c)^2} e^{J^* + 2\epsilon_0} + \zeta\]
\end{theorem}

\subsection{Experimental Validation}

We now show that the \textbf{high-probability coverage} assumption is
reasonable for the LCBs learned for LLaDA
(Section \ref{sec:llada experiments}). In Setting II of
Section \ref{sec:llada experiments} we take $\alpha = 0.3$ and train
an LCB at $\delta = 0.1$ and generate $M$ trajectories from
$\hat{q}(x_T \dots x_0)$ by performing LCB-guided inference
(Algorithm \ref{alg:baseline rejection sampling}). This defines a
dataset of $M \times T$ \emph{source states}:
$\{x^i_{t+1}\}_{i \in [M],\,t \in \{0,\dots,T-1\}}$.
For each source state, $x^i_{t+1}$, we draw $N$ proposal states,
$\{y^{i,j}_t\}_{j \in [N]}$ with
$y^{i,j}_t \sim \ppre_t(\cdot \mid x^i_{t+1})$.
For each source state, $x^i_{t+1}$, we then check the rate at which
it successfully upper bounds proposal states' value:
\[
D_{i,t+1}
:= \frac{1}{N}\sum_{j=1}^N
1\{\hat{v}_t(y^{i,j}_t) \leq B_{t+1}(x^i_{t+1})\}.
\]

Experimentally, we set $M = 50, N = 20$. In
Figure \ref{fig:coverage heatmap}, we show the heatmap for the
matrix with entries $D_{i,t+1}$.

We now investigate approximate coverage levels $(c_s,\zeta_s)$,
indexed by source timestep $s \in \{1,\dots,T\}$.
We (somewhat arbitrarily) fix $c_s = 1/2$ and heuristically
estimate $\zeta_s$ as the fraction of source states whose empirical
coverage is below $1-c_s$:

\[
\widehat{\zeta}_s
:= \frac{1}{M}\sum_{i=1}^M
\mathbf{1}\{D_{i,s} < 1-c_s\}
= \frac{1}{M}\sum_{i=1}^M
\mathbf{1}\{D_{i,s} < 1/2\}.
\]
Thus, a source state passes the empirical coverage criterion
when at least $10$ of its $20$ proposals are covered.

In Table \ref{table:zetas}, we report the nonzero empirical
estimates $\widehat{\zeta}_s$ and their sum.
The high-probability coverage bounds, together with
Lemma \ref{lem:joint TV decomp body}, yield an additive
$\sum_s \zeta_s$ term in the bound on final sampling error.
Our empirical estimate of this term is
$\sum_s \widehat{\zeta}_s = 0.10$.

\begin{figure}
    \centering
    \includegraphics[width=0.65\linewidth]
        {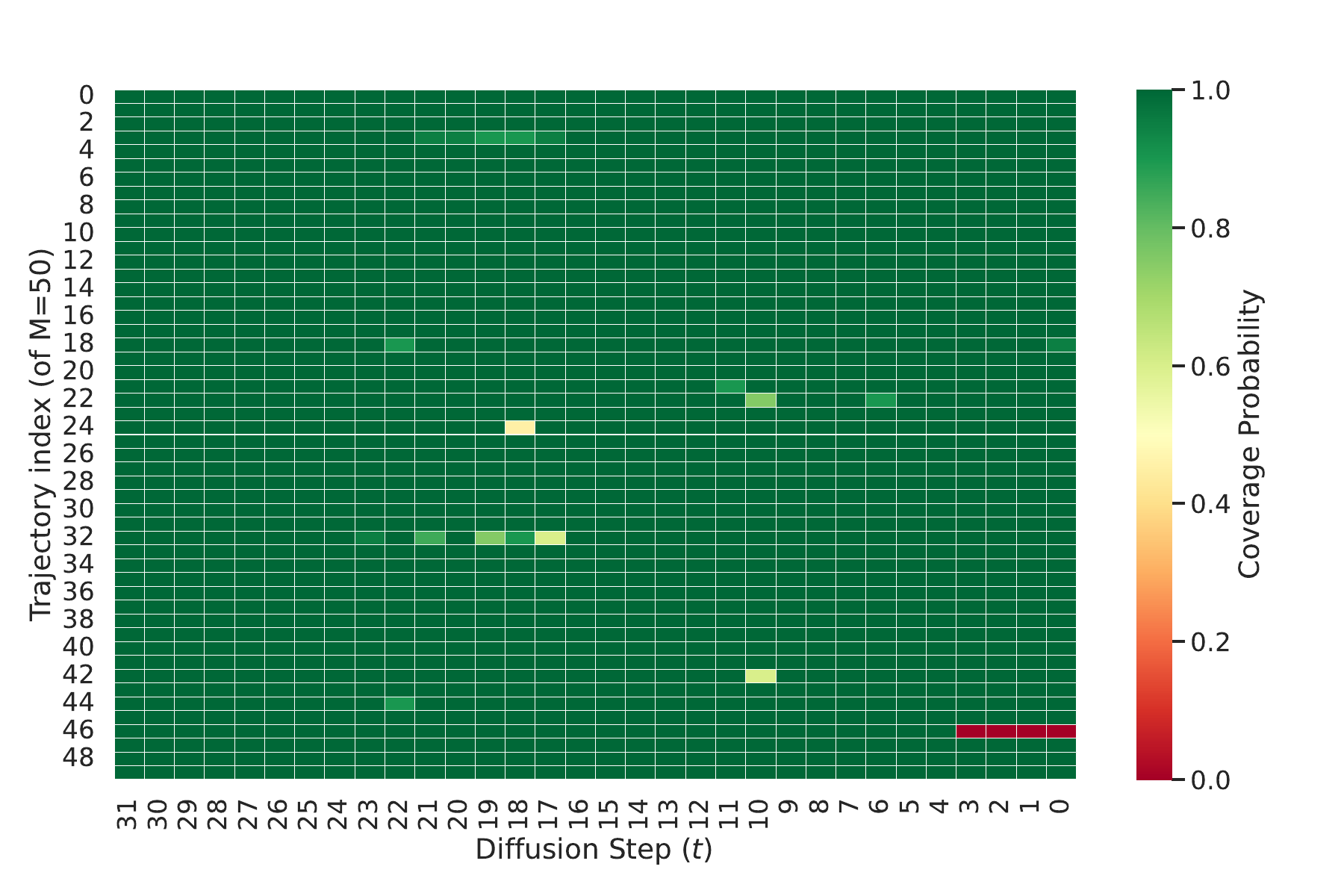}
    \caption{A large fraction of states satisfy coverage for all
    their sampled proposals, for LCBs trained
    in Setting II.}
    \label{fig:coverage heatmap}
\end{figure}

\begin{table}[h]
\centering
\begin{tabular}{cccccc}
\toprule
$\widehat{\zeta}_{19}$ &
$\widehat{\zeta}_{4}$ &
$\widehat{\zeta}_{3}$ &
$\widehat{\zeta}_{2}$ &
$\widehat{\zeta}_{1}$ &
$\sum_s \widehat{\zeta}_s$ \\
\midrule
$0.02$ & $0.02$ & $0.02$ & $0.02$ & $0.02$ & $0.10$ \\
\bottomrule
\end{tabular}
\caption{
Nonzero empirical estimates $\widehat{\zeta}_s$ and their sum,
indexed by source timestep $s$, with $c_s = 1/2$.
}
\label{table:zetas}
\end{table}

The key conclusion is that LCBs tend to satisfy high-probability
coverage with favorable values of $\zeta$ and $c$.

\newpage

\newpage
\section*{NeurIPS Paper Checklist}


\begin{enumerate}

\item {\bf Claims}
    \item[] Question: Do the main claims made in the abstract and introduction accurately reflect the paper's contributions and scope?
    \item[] Answer: \answerYes{} 
    \item[] Justification: Abstract claims TV  analysis and experimental proof of fidelity to $\hat{p}$, both of which are shown. 

\item {\bf Limitations}
    \item[] Question: Does the paper discuss the limitations of the work performed by the authors?
    \item[] Answer: \answerYes{} 
    \item[] Justification: The requirement of a value signal and upfront training cost are mentioned in the paper. 

\item {\bf Theory assumptions and proofs}
    \item[] Question: For each theoretical result, does the paper provide the full set of assumptions and a complete (and correct) proof?
    \item[] Answer: \answerYes{} 
    \item[] Justification: Theorems are quoted formally and in detail in Sections 3 and 4. Proofs are given in appendices. 

    \item {\bf Experimental result reproducibility}
    \item[] Question: Does the paper fully disclose all the information needed to reproduce the main experimental results of the paper to the extent that it affects the main claims and/or conclusions of the paper (regardless of whether the code and data are provided or not)?
    \item[] Answer: \answerYes{} 
    \item[] Justification: Appendix \ref{apdx:experiments} gives complete descriptions of experimental setup for each of the 3 settings tested. 

\item {\bf Open access to data and code}
    \item[] Question: Does the paper provide open access to the data and code, with sufficient instructions to faithfully reproduce the main experimental results, as described in supplemental material?
    \item[] Answer: \answerYes{} 
    \item[] Justification: Code will be provided upon acceptance. File sizes are too large to include in Openreview. 

\item {\bf Experimental setting/details}
    \item[] Question: Does the paper specify all the training and test details (e.g., data splits, hyperparameters, how they were chosen, type of optimizer) necessary to understand the results?
    \item[] Answer: \answerYes{} 
    \item[] Justification: see item 4. 

\item {\bf Experiment statistical significance}
    \item[] Question: Does the paper report error bars suitably and correctly defined or other appropriate information about the statistical significance of the experiments?
    \item[] Answer: \answerYes{} 
    \item[] Justification: Confidence intervals are given in tables, and the method for arriving at them are described. 

\item {\bf Experiments compute resources}
    \item[] Question: For each experiment, does the paper provide sufficient information on the computer resources (type of compute workers, memory, time of execution) needed to reproduce the experiments?
    \item[] Answer: \answerYes{} 
    \item[] Justification: This is discussed in Appendix \ref{apdx:experiments}, where time to run on an A100 colab notebook is described. 
    
\item {\bf Code of ethics}
    \item[] Question: Does the research conducted in the paper conform, in every respect, with the NeurIPS Code of Ethics \url{https://neurips.cc/public/EthicsGuidelines}?
    \item[] Answer: \answerYes{} 
    \item[] Justification: we have reviewed the Code of Ethics and confirmed that this research conforms. 

\item {\bf Broader impacts}
    \item[] Question: Does the paper discuss both potential positive societal impacts and negative societal impacts of the work performed?
    \item[] Answer: \answerYes{} 
    \item[] Justification: We mention that the ability to steer commercial models towards obscure data modes using only black box query access is dual use, and can be used in pro-social or anti-social ways. Mitigating these effects requires careful training recipes for pretrained models. 
    
\item {\bf Safeguards}
    \item[] Question: Does the paper describe safeguards that have been put in place for responsible release of data or models that have a high risk for misuse (e.g., pre-trained language models, image generators, or scraped datasets)?
    \item[] Answer: \answerNA{} 
    \item[] Justification: We do not perform experiments on sensitive data, and our trained models are not vulnerable to misuse, even though the general techniques could be misused. 

\item {\bf Licenses for existing assets}
    \item[] Question: Are the creators or original owners of assets (e.g., code, data, models), used in the paper, properly credited and are the license and terms of use explicitly mentioned and properly respected?
    \item[] Answer: \answerYes{} 
    \item[] Justification: Pretrained diffusion models and reward models (LLaDA and Skywork) are cited. 

\item {\bf New assets}
    \item[] Question: Are new assets introduced in the paper well documented and is the documentation provided alongside the assets?
    \item[] Answer: \answerNA{} 
    \item[] Justification: New assets are not introduced

\item {\bf Crowdsourcing and research with human subjects}
    \item[] Question: For crowdsourcing experiments and research with human subjects, does the paper include the full text of instructions given to participants and screenshots, if applicable, as well as details about compensation (if any)? 
    \item[] Answer: \answerNA{} 
    \item[] Justification: Paper does not involve crowdsourcing

\item {\bf Institutional review board (IRB) approvals or equivalent for research with human subjects}
    \item[] Question: Does the paper describe potential risks incurred by study participants, whether such risks were disclosed to the subjects, and whether Institutional Review Board (IRB) approvals (or an equivalent approval/review based on the requirements of your country or institution) were obtained?
    \item[] Answer: \answerNA{} 
    \item[] Justification: Paper does not include human subjects.

\item {\bf Declaration of LLM usage}
    \item[] Question: Does the paper describe the usage of LLMs if it is an important, original, or non-standard component of the core methods in this research? Note that if the LLM is used only for writing, editing, or formatting purposes and does \emph{not} impact the core methodology, scientific rigor, or originality of the research, declaration is not required.
    \item[] Answer: \answerNA{} 
    \item[] Justification: LLMs were used for verification of proofs and for assistance in coding for experiments. The research direction, solution, and proofs of theorems were produced by human authors. As such  the usage of LLMs does not need to be disclosed in the body of this paper. 

\end{enumerate}

\end{document}